\documentclass[journal]{IEEEtran}

\usepackage{tabularx}
\usepackage{ctable}
\usepackage{booktabs}
\usepackage{multirow}
\usepackage{amsmath,amsthm,amssymb,amsfonts,dsfont,stmaryrd}
\usepackage[ruled,linesnumbered]{algorithm2e}

\usepackage{subfigure}
\usepackage{cite}
\newtheorem{proposition}{Proposition}
\newtheorem{lemma}{Lemma}

\newcommand{\hdr}{HDR}
\newcommand{\Schoberl}{Sch\"oberl}
\newcommand{\N}{\mathcal{N}}
\newcommand{\R}{\mathbb{R}}
\DeclareMathOperator*{\argmin}{arg\,min}
\DeclareMathOperator*{\argmax}{arg\,max}

\definecolor{myRed}{RGB}{229,28,35}
\definecolor{myViolet}{RGB}{74,20,140}
\definecolor{myGreen}{RGB}{0,137,123}
\definecolor{myYellow}{RGB}{255,193,7}

\newcommand{\psnr}{PSNR}

\newcommand{\hpnlb}{HBE}

\newcommand\sR{\mathbb{R}}

\newcommand\sN{\mathbb{N}}

\renewcommand \L {\boldsymbol{\Lambda}}
\renewcommand \k {\kappa}
\renewcommand{\S}{\boldsymbol{\Sigma}}
\newcommand{\mean}{\boldsymbol{\mu}}
\renewcommand{\a}{g}
\newcommand{\f}{\mathbf{C}}
\renewcommand{\c}{\mathbf{c}}

\newcommand{\fa}{o}
\newcommand{\gp}{a}
\newcommand{\C}{C}
\newcommand{\Cp}{\mathrm{C}}

\newcommand{\U}{\boldsymbol{D}}

\newcommand{\E}{\mathbb{E}}

\newcommand{\ns}{\mathbf{N}}

\newcommand{\Yp}{\mathrm{Y}}

\newcommand{\y}{\mathbf{Z}}

\newcommand{\z}{\mathbf{z}}               

\newcommand{\Z}{\mathbf{Z}}  
\newcommand{\Zp}{\mathrm{Z}}  

\newcommand{\prnu}{\textsc{prnu}}

\newcommand{\Np}{n}

\newcommand{\meanR}{\mean_R}

\newcommand{\sn}{\sigma^2_R}

\ifCLASSINFOpdf
   \graphicspath{{images/}{images/nl_ple/}{images/nl_ple/barbara/}{images/nl_ple/boat/}{images/nl_ple/eglise1/}{images/nl_ple/lampReal}{images/nl_ple/palomas/}{images/nl_ple/telecom/}{images/nl_ple/traffic/}}
\else
 \graphicspath{{images/}{images/nl_ple/}{images/nl_ple/barbara/}{images/nl_ple/boat/}{images/nl_ple/eglise1/}{images/nl_ple/lampReal}{images/nl_ple/palomas/}{images/nl_ple/telecom/}{images/nl_ple/traffic/}}
\fi 

\begin{document}

\title{A Bayesian Hyperprior Approach for Joint Image Denoising and Interpolation, \\
with an Application to HDR Imaging}

\author{Cecilia~Aguerrebere, Andr\'es~Almansa, Julie~Delon, Yann~Gousseau and Pablo~Mus\'e\thanks{C. Aguerrebere is with the Department of Electrical and Computer Engineering, Duke University, Durham NC 27708, US (e-mail: cecilia.aguerrebere@duke.edu)}\thanks{Y. Gousseau is with the LTCI, T\'el\'ecom ParisTech, Universit\'e Paris-Saclay, 75013 Paris, France (e-mail: gousseau@telecom-paristech.fr).}\thanks{A. Almansa and J. Delon are with MAP5 (CNRS UMR 8145), Universit\'e Paris Descartes, 75270 Paris Cedex 06 (e-mail: andres.almansa,julie.delon@parisdescartes.fr)}\thanks{P. Mus\'e is with the Department of Electrical Engineering, Universidad de la Rep\'ublica, 11300 Montevideo, Uruguay (e-mail: pmuse@fing.edu.uy)}}

\markboth{Journal of \LaTeX\ Class Files,~Vol., No., March~2017}%
{Shell \MakeLowercase{\textit{et al.}}: Bare Demo of IEEEtran.cls for Journals}

\maketitle

\begin{abstract}
Recently, impressive denoising results have been achieved by Bayesian approaches which assume Gaussian models for the image patches.
This improvement in performance can be attributed to the use of per-patch models.
Unfortunately such an approach is particularly unstable for most inverse problems beyond denoising.
In this work, we propose the use of a hyperprior to model image patches, in order to stabilize the estimation procedure.
There are two main advantages to the proposed restoration scheme:
Firstly it is adapted to diagonal degradation matrices, and in particular to missing data problems (e.g. inpainting of missing pixels or zooming).
Secondly it can deal with signal dependent noise models, particularly suited to digital cameras. As such, the scheme is especially adapted to computational photography. In order to illustrate this point, we provide
an application to high dynamic range imaging from a single image
taken with a modified sensor, which shows the effectiveness of the proposed scheme. 
\end{abstract}


\IEEEpeerreviewmaketitle

\section{Introduction}

\IEEEPARstart{D}{igital} images are subject to a wide variety of degradations, which in most cases can be modeled as 
\begin{equation}
\y = \U \f + \ns,
\label{eq:modelNewI}
\end{equation}
where $\y$ is the observation, $\U$ is the degradation operator, $\f$ is the underlying ground-truth image and $\ns$ is additive noise. Different settings of the degradation matrix $\U$ model different problems such as zooming, deblurring or missing pixels. Different versions of the noise term $\ns$ include  the classical additive Gaussian noise with constant variance or more complicated and realistic models such as signal dependent noise.

Due to the inherent ill-posedness of such inverse problems, standard approaches impose some prior on the image, in either variational or Bayesian approaches. Popular image models have been proposed through the total variation~\cite{rudin1992nonlinear}, wavelet decompositions~\cite{portilla2003image} or the sparsity of image patches~\cite{elad2006image}. 
Buades et al.~\cite{buades05} introduced the use of patches and the self-similarity hypothesis to the denoising problem leading to a new era of patch-based image restoration techniques.  
A major step forward in fully exploiting the  potential of patches was achieved by several state-of-the-art restoration methods with the introduction of patch prior models, in a Bayesian framework. Some methods are devoted to the denoising problem~\cite{lyu09,chatterjee12,lebrun13,wang13}, while others propose a more general framework for the solution of image inverse problems~\cite{zoran11,yu12}, including inpainting, deblurring and zooming. The work by Lebrun et al.~\cite{lebrun12,lebrun13} presents a thorough analysis of several recent restoration methods, revealing their common roots and their relationship with the Bayesian approach. 

Among the state-of-the-art restoration methods, two noticeable approaches are the patch-based Bayesian approach by Yu et al.~\cite{yu12}, namely the piece-wise linear estimators (PLE), and the non-local Bayes (NLB) algorithm by Lebrun et al.~\cite{lebrun13}. PLE is a general framework for the solution of image inverse problems under Model~\eqref{eq:modelNewI}, while NLB is a denoising method ($\U = Id$). Both methods use a Gaussian patch prior learnt from image patches through iterative procedures. In the case of PLE, patches are modeled according to a Gaussian Mixture Model (GMM), with a relatively small number of classes (19 in all their experiments), whose parameters are learnt from all image patches\footnote{Actually, the authors report the use of $128 \times 128$ image sub-regions in their experiments, so we may consider PLE as a semi-local approach.}. In the case of NLB, each patch is associated with a single Gaussian model, whose parameters are computed from similar patches chosen from a local neighbourhood, i.e., a search window centered at the patch. We refer hereafter to this kind of per-patch modelling as \textit{local}. 

Zoran and Weiss~\cite{zoran11} (EPLL) follow a similar approach, but instead of iteratively updating the GMM from image patches, they use a larger number of classes that are fixed and learnt from a large database of natural image patches. Wang and Morel~\cite{wang13} claim that, in the case of denoising, it is better to have fewer models that are updated with the image patches (as in PLE) than having a large number of fixed models (as in EPLL). 

All of the previous restoration approaches share a common Bayesian framework based on Gaussian patch priors. Relying on local priors~\cite{lebrun13,wang13} has proven more accurate for the task of image denoising than relying on a mixture of a limited number of Gaussian models~\cite{zoran11,yu12}. In particular, NLB outperforms PLE for this task~\cite{wang13c}, mostly due to its local model estimation. On the other hand, PLE yields state-of-the-art results in other applications such as interpolation of missing pixels (especially with high masking rates), deblurring and zooming. 

As a consequence we are interested in taking advantage of a local patch modelling for more general inverse problems than denoising. The main difficulty lies in the estimation of the models, especially when the image degradations involve a high rate of missing pixels, in which case the estimation is seriously ill-posed. 

In this work we propose to model image patches according to a Gaussian prior, whose parameters, the mean $\mean$ and the covariance matrix $\S$, will be estimated locally from similar patches. In order to tackle this problem, we include prior knowledge on the model parameters making use of a hyperprior, i.e. a probability distribution on the parameters of the prior. In Bayesian statistics, $\mean$ and $\S$ are known as hyperparameters, while the prior on them is called a hyperprior. Such a framework is often called hierarchical Bayesian modelling~\cite{gelman2014bayesian}. Its application to inverse problems in imaging is not new. In particular, in the field of image restoration, this methodology was proposed by Molina et al.~\cite{Molina1994,Molina1999}, and was more recently applied to image unmixing problems~\cite{dobigeon2008semi} and to image deconvolution and the estimation of the point spread function of a camera~\cite{Orieux2010}. However, to our knowledge, this is the first time that such a hierarchical Bayesian methodology is used to reduce ill-posedness in patch-based image restoration. In this context, the use of a hyperprior 
compensates for the patches missing information. 

There are two main contributions of this work:
First, as described above, we propose a robust framework enabling the use of Gaussian local priors on image patches for solving a useful family of restoration problems by drawing on a hierarchical Bayesian approach. 
The second advantage of the proposed framework is its ability to deal with
signal dependent noise, therefore making it adapted to realistic digital photography applications. 

Experiments on both synthetic and real data show that the approach is well suited to various problems involving a diagonal degradation operator. First, we show state-of-the-art results in image restoration problems such as denoising, zooming and interpolation of missing pixels. Then we consider the generation of high dynamic range (HDR) images from a single snapshot using spatially varying pixel exposures~\cite{nayar00} and demonstrate that our approach significantly outperforms existing methods to deal with this inverse problem. It is worth mentioning that modified sensors enabling such approaches have been recently made available by Sony but are not yet fully exploited by available smartphones and digital cameras.

The article is organized as follows. Section~\ref{sec:newMethod} introduces the proposed approach while Section~\ref{sec:implDetails} presents the main implementation aspects. Supportive experiments are presented in Section~\ref{ssec:expsNewMethod}. Section~\ref{sec:HDR} is devoted to the application of the proposed framework to the HDR imaging problem. Last, conclusions are summarized in Section~\ref{sec:conclusions}.

\section{Hyperprior Bayesian Estimator}
\label{sec:newMethod}
\begin{figure}
\centering
\includegraphics[width=\linewidth]{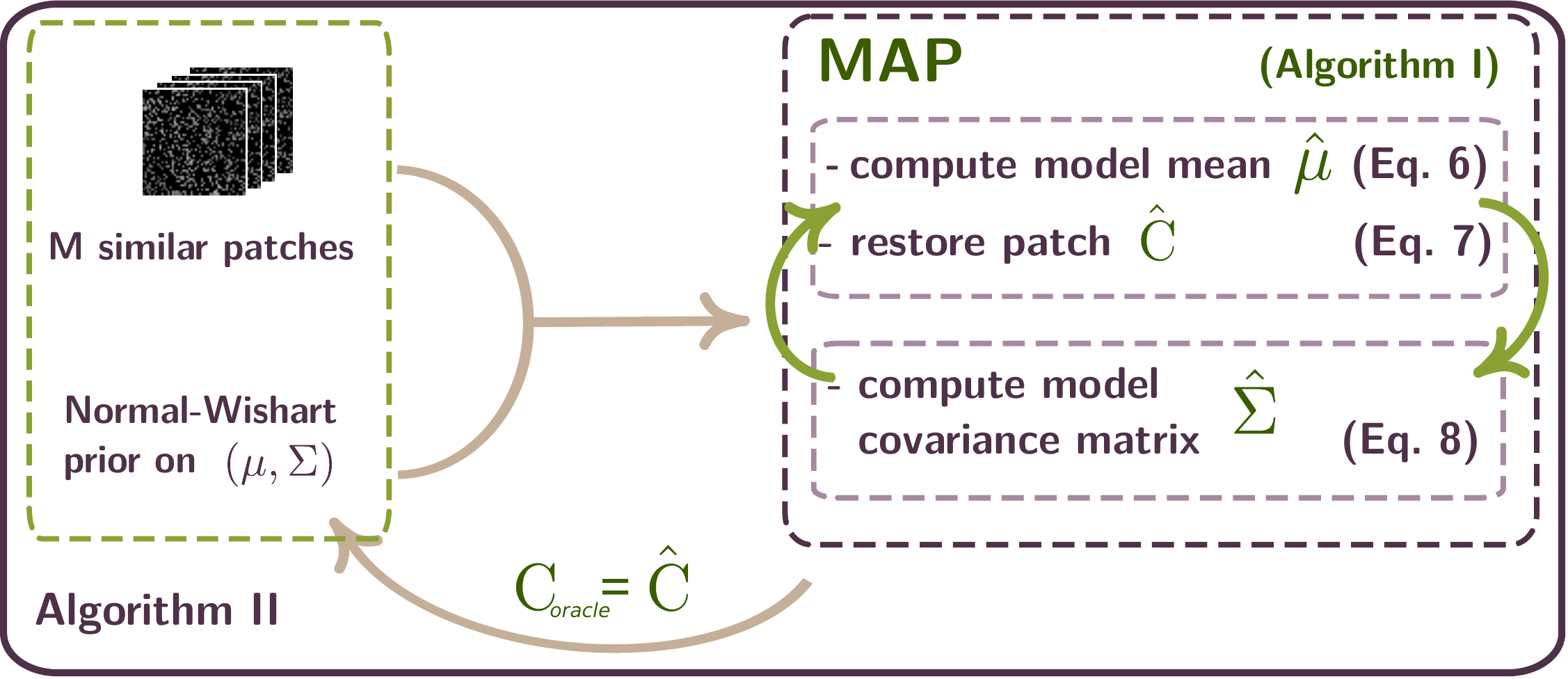} 
\caption{Diagram of the proposed iterative approach.}
\label{fig:methodDiagram}
\end{figure}
The proposed restoration method, called Hyperprior Bayesian Estimator (\hpnlb{}), assumes a Gaussian prior for image patches, with parameters $\mean$ and $\S$. A \textbf{joint maximum a posteriori} formulation is used to estimate both the image patches and the parameters $\mean$ and $\S$, thanks to a Bayesian hyperprior model on these parameters, stabilizing the local estimation of the Gaussian statistics. As a consequence, we can exploit the accuracy of local model estimation for general restoration problems, in particular with missing values (e.g. for interpolation or zooming). Figure~\ref{fig:methodDiagram} illustrates the proposed approach which is described in detail in the following.

\subsection{Patch degradation model} 
\label{ssec:degradation}
The observed image $\z$ is decomposed into $I$ overlapping patches $\{\z_i\}_{i=1,\dots,I}$ of size $\sqrt{\Np}\times\sqrt{\Np}$. Each patch $\z_i \in \sR^{\Np \times 1}$ is considered to be a realization of the random variable $\y_i$ given by
\begin{equation}
\y_i = \U_i \f_i + \ns_i,
\label{eq:modelNew}
\end{equation}
where $\U_i \in \sR^{\Np \times \Np}$ is a degradation operator, $\f_i \in \sR^{\Np  \times 1}$ is the original patch we seek to estimate and $\ns_i \in \sR^{\Np \times 1}$ is an additive noise term, modeled by a Gaussian distribution $\ns_i\sim \N(0,\S_{N_i})$.
Therefore, the distribution of $\y_i$ given $\f_i$ can be written as
\begin{flalign}
  &p(\y_i \;|\; \f_i) \sim \N(\U_i\f_i,\S_{N_i})& \nonumber\\
  &\propto|\S_{N_i}^{-1}|^{\frac 1 2} \exp\left(-\frac 1 2 (\y_i - \U_i
    \f_i)^T \S_{N_i}^{-1} (\y_i - \U_i \f_i)\right).&
\label{eq:modelPatch}
\end{flalign}
In this noise model, the matrix $\S_{N_i}$ is only assumed to be diagonal (the noise is uncorrelated). It can represent a constant variance, spatially variable variances or even variances dependent on the pixel value (to approximate Poisson noise). 

This degradation model is deliberately generic. We will see in Section~\ref{ssec:expsNewMethod} that keeping a broad noise model is essential to properly tackle the problem of HDR imaging from a single image. The model also includes the special case of multiplicative noise. 

\subsection{Joint Maximum A Posteriori }
\label{ssec:jointmap}
We assume a Gaussian prior for each patch, with unknown mean $\mean$ and covariance matrix $\S$, $p(\f_i \;|\; \mean,\S) \sim \N(\mean,\S).$ 
To simplify calculations we work with the precision matrix $\L = \S^{-1}$. As it is usual when considering hyperpriors, we assume that the parameters $\mean$ and $\S$ follow a conjugate distribution. In our case, that boils down to assuming a Normal-Wishart\footnote{The Normal-Wishart distribution is the conjugate prior of a multivariate normal distribution with unknown mean and covariance matrix. $\mathcal{W}$ denotes the Wishart distribution~\cite{nla.cat-vn751926}.} prior for the couple ($\mean,\L$),
\begin{flalign}
\label{eq:pmuL}
p(\mean,\L) &= \N(\mean | \mean_0,(\k\L)^{-1})\mathcal{W}(\L|(\nu \S_0)^{-1},\nu)\\
	  &\propto |\L|^{1/2} \exp{\left( -\frac{\k}{2}(\mean - \mean_0) \L (\mean - \mean_0)^T \right)} \nonumber \\
          &\phantom{\propto } |\L|^{(\nu-n-1)/2} \exp{ \left( -\frac{1}{2} \text{tr}(\nu \S_0 \L) \right) }, \nonumber 
\end{flalign}
with parameters $\mean_0$, $\S_0$, scale parameter $\k > 0$ and $\nu > n - 1$ degrees of freedom.

Now, assume that we observe a group $\{\y_i\}_{i=1,\dots,M}$ of similar patches and that we want to recover the restored patches $\{\f_i\}_{i=1,\dots,M}$. If these unknown $\{\f_i\}$ are independent\footnote{We rely on the classical independence assumption made in the patch-based literature, even if it is wrong in case patches overlap.} and follow the same  Gaussian model, we can compute the joint maximum a posteriori
\begin{flalign}
&\argmax\limits_{\{\f_i\},\mean,\L} \;\;p( \{\f_i\},\mean,\L \;|\;\{\y_i\}) = \\ 
\;\;\;\;=&\;\;p(\{\y_i\} \;|\; \{\f_i\},\mean,\L)\;  \; p(\{\f_i\} \;|\; \mean,\L)\; \; p(\mean,\L)& \nonumber\\
\;\;\;\;=&\;\;p(\{\y_i\} \;|\; \{\f_i\})\;  \; p(\{\f_i\} \;|\; \mean,\L)\; \; p(\mean,\L).& \nonumber
\end{flalign}
In this product, the first term is given by the noise model~\eqref{eq:modelPatch}, the second one is the Gaussian prior on the set of patches $\{\f_i\}$ and the third one is the hyperprior~\eqref{eq:pmuL}. In the last equality we omit the explicit dependence on $\mean$ and $\L$ in $p(\{\y_i\} \;|\; \{\f_i\},\mean,\L)$, since these parameters are completely determined by the set $\{\f_i\}$.   
\subsection{Optimality conditions}
Computing the joint maximum a posteriori amounts to minimizing
{\small
\begin{eqnarray*}
\label{eq:pmuLAp}
f(\{\f_i\},\mean,\L) &:=& -\log p( \{\f_i\},\mean,\L \;|\;\{\y_i\}) \\
&=&\frac 1 2 (\y_i - \U_i
    \f_i)^T \S_{N_i}^{-1} (\y_i - \U_i \f_i)\\
&-&  {\frac{\nu -n +M}{2}}\log |\L|\\
&+&\frac 1 2 \sum_{i=1}^M (\f_i - \mean)^T \L (\f_i - \mean)  \\
&+&\frac {\kappa}{2}(\mean-\mean_0)^T\L(\mean-\mean_0)+\frac {1}{2}\mathrm{trace}[\nu\S_0 \L ],
\end{eqnarray*}
} over the set $\R^{nM} \times \R^n \times S_n^{++}(\R)$, with  $S_n^{++}(\R)$ the set of real symmetric positive definite matrices of size $n$.

The function $f$ is biconvex respectively in the variables $(\{\f_i\},\mean)$ and $\mathbf \L$. To minimize this energy for a given set of hyper-parameters $(\mean_0,\L_0)$, we will use an alternating convex minimization scheme. At each iteration, $f$ is first minimized with respect to $(\{\f_i\},\mean)$ with $\L$ fixed, then viceversa.

Differentiating $f$ with respect to each variable, we get explicit optimality equations for the minimization scheme. The proofs of the following propositions are straightforward and available in the supplementary material. 
\begin{proposition}
Assume that $\L$ is fixed and that the covariance $\S_{N_i}$ does not depend on the $\{\f_i\}$. The function  $(\{{\f_i}\},\mean) \mapsto f(\{{\f_i}\},\mean,\L)$ is convex on $\R^{n(M+1)}$ and attains its minimum at $(\{\hat{\f_i}\},\hat{\mean})$, given by
\begin{align}
 \widehat{\mean} &=&  \left(\kappa \mathrm{Id} + \sum_{i=1}^M  \mathbf{A}_i\U_i \right)^{-1} \left(\sum_{i=1}^M \mathbf{A}_i \y_i +\kappa \mean_0 \right). 
\label{eq:muhat}
\end{align}
\begin{align}
\widehat{\f_i}  &=& \mathbf{A}_i (\y_i - \U_i \hat{\mean}) + \hat{\mean}, \;\;\;\;\;\;\forall i \in \{1,\dots M\}
\label{eq:Ci1}
\end{align}
with $\mathbf{A}_i = \L^{-1} \U_i^T( \U_i \L^{-1} \U_i^T + \S_{N_i})^{-1}$.
\end{proposition}

\begin{proposition}
Assume that  the variables $(\{{\f_i}\},\mean)$ are fixed. The function $\L \rightarrow f(\{{\f_i}\},\mean,\L)$ is convex on $S_n^{++}(\R)$ and attains its minimum at  $\hat{\L}$ such that  
{\small \begin{equation}
\hat{\L}^{-1} = \frac{\nu \S_0 + \kappa(\mean - \mean_0) (\mean - \mean_0)^T + \sum_{i=1}^M (\f_i - \mean) (\f_i - \mean)^T}{\nu +M-n}.
\label{eq:Lhat}
\end{equation}}
\end{proposition}

The expression of $\widehat{\f}_i$ in~\eqref{eq:Ci1} is obtained under the hypothesis that the noise covariance matrix $\mathbf{\Sigma}_{N_i}$ does not depend on $\f_i$. Under the somewhat weaker hypothesis that the noise $N_i$ and the signal $\f_i$ are uncorrelated, this estimator is also the affine estimator $\tilde{\f}_i$ that minimizes the Bayes risk $\E[(\tilde{\f}_i - \f_i)^2]$ (c.f. supplementary material). 
The uncorrelatedness of $N_i$ and $C_i$ is a reasonable hypothesis in practice. This includes various noise models, such as $N_i = f(C_i) \varepsilon_i$ with $\varepsilon_i$ independent of $C_i$, which approximates CMOS and CCD raw data noise~\cite{aguerrebere12}.

From~\eqref{eq:muhat}, we find that the MAP estimator of $\mean$ is a weighted average of two terms: the mean estimated from the similar restored patches and the prior $\mean_0$. The parameter $\k$ controls the confidence level we have on the prior $\mean_0$. With the same idea, we observe that the MAP estimator for $\L$ is a combination of the prior $\L_0$ on $\L$, the covariance imposed by ${\mean}$ and the covariance matrix estimated from the patches ${\f}_i$.

\subsection{Alternating convex minimization of $f$}
The previous propositions imply that we can derive an explicit alternating convex minimization scheme for $f$, presented in Algorithm~\ref{algo:alternate}. Starting with a given value of $\L$, at each step, $\mean^l$ and $\f^l$ are computed according to Equations~\eqref{eq:muhat} and ~\eqref{eq:Ci1}, then $\L^l$ is updated according to~\eqref{eq:Lhat}. 
\begin{algorithm}
\KwIn{$\Z$, $\U$, $\mean_0, \S_0, \kappa, \nu$}
\KwOut{ $\{\hat{\f_i}\}$,$\hat{\mean}$, $\hat{\L}$, } 
\textbf{Initialization:} Set $\L^0=\S_0^{-1}$\\
\For{l = 1 to $\text{maxIts}$}{
Compute $({\f^{l}},{\mean}^{l}) = \argmin_{(\f,\mean)} f(\f,\mean,\L^{l-1})$ by equations~\eqref{eq:muhat} and~\eqref{eq:Ci1}
 
Compute ${\L}^{l} = \argmin_{\L} f({\f}^l,\mean^l,\L)$ by Eq.~\eqref{eq:Lhat}}
$\{\hat{\f_i} = \f_i^{\text{maxIts}}\}$, $\hat{\mean} = {\mean}^{\text{maxIts}}$, $\hat{\L} = \L^{\text{maxIts}}$
\caption{Alternating convex minimization for $f$}
\label{algo:alternate}
\end{algorithm}
We show in Appendix~\ref{ap:mapEst} the following convergence result for the previous algorithm. The proof adapts the arguments in~\cite{Gorski2007} to our particular case.
\begin{proposition}
 The sequence $f( \{\f_i^l\},\mean^l,\L^l )$ converges monotonically when $l\rightarrow +\infty$.  The sequence $\{\{\f_i^l\},\mean^l,\L^l\}$ generated by the alternate minimization scheme has at least one accumulation point. The set of its accumulation points forms a connected and compact set of partial optima and stationary points of $f$, and they all have the same function value. 
\end{proposition}
In practice, we observe in our experiments that the algorithm always converges after a few iterations. 

\subsection{Full restoration  algorithm}
\label{ssec:pasosAlgo}
The full restoration algorithm used in our experiments is summarized in Algorithm~\ref{algo:newMethod} and  illustrated by Figure~\ref{fig:methodDiagram}. It alternates between two stages: the minimization of $f$ using Algorithm~\ref{algo:alternate}, and the estimation of the hyper-parameters $\mean_0, \S_0$. In order to estimate these parameters, we rely on an \textit{oracle image} computed by aggregation of all the patches estimated on the first stage (details are provided in Section~\ref{ssec:initialization}).
\begin{algorithm}
\KwIn{$\Z$, $\U$, $\mean_0, \S_0, \kappa, \nu$ (see details in Section~\ref{ssec:paramsAnalisis})}
\KwOut{$\tilde{\f}$}
Decompose $\Z$ and $\U$ into overlapping patches. \\
\textbf{Initialization:} Compute first oracle image $\f_\mathrm{oracle}$ (see details in Section~\ref{ssec:initialization})\\
\For{it = 1 to $\text{maxIts}_2$}{
\For{all patches not yet restored}{
Find patches similar ($L^2$ distance) to the current $\z_i$ in $\f_\mathrm{oracle}$ (see details in Section~\ref{sec:implDetailsSearch}).\\
Compute $\mean_0$ and $\S_0$ from $\f_\mathrm{oracle}$ (see details in Section~\ref{ssec:paramsAnalisis}).\\
Compute $(\{\hat{\f}_i\},\hat{\mean},\hat{\S})$ following Algorithm~\ref{algo:alternate}.\\
}
Perform aggregation to restore the image. \\
Set $\f_\mathrm{oracle} = \tilde{\f}$.\\
}
\caption{HBE algorithm.}
\label{algo:newMethod}
\end{algorithm}

\section{Implementation details}
\label{sec:implDetails}
\subsection{Search for similar patches} 
\label{sec:implDetailsSearch}
The similar patches are all patches within a search window centered at the current patch,
whose $L_2$ distance to the central patch is less than a given threshold. This threshold is given by a tolerance parameter $\varepsilon$ times the distance to the nearest neighbour (the most similar one). In all our experiments, the search window was set to size $25 \times 25$ (with a patch size of $8\times 8$) and $\varepsilon=1.5$. The patch comparison is performed in an oracle image (i.e. the result of the previous iteration), so all pixels are known. However, it may be useful to assign different confidence levels to the known pixels and to those originally missing and then restored. For all the experimental results presented in Section~\ref{ssec:expsNewMethod}, the distance between patches $\c_p$ and $\c_q$ in the oracle image $\f_\mathrm{oracle}$ is computed as
\begin{equation}
d(p,q) = \frac{\sum_{j=1}^N (\c^j_p - \c^j_q)^2 \omega^j_{p,q}}{\sum_{j=1}^N \omega^j_{p,q}},
\end{equation}  
where $j$ indexes the pixels in the patch, $\omega^j_{p,q}=1$ if $\U^j_p = \U^j_q = 1$ (known pixel) and $\omega^j_{p,q}=0.01$ otherwise (originally missing then restored pixel)~\cite{arias12}. With this formulation, known pixels are assigned a much higher priority than unknown ones. Variations on these weights could be explored.

\subsection{Optional speed-up by Collaborative Filtering}
\label{ssec:colabFilt}
The proposed method computes one Gaussian model per image patch according to Equations~\eqref{eq:muhat} and~\eqref{eq:Lhat}. In order to reduce the computational cost, we can rely on the collaborative filtering idea previously introduced for patch-based denoising techniques~\cite{lebrun13,dabov07}. Based on the hypothesis that similar patches share the same model, we assign the same model to all patches in the set of similar patches (as defined in Section~\ref{sec:implDetailsSearch}). 

The number of similar patches jointly restored depends on the image and the tolerance parameter $\varepsilon$, but it is often much smaller than what would result from the patch clustering performed by methods that use global GMMs such as PLE or EPLL. Performance degradation is observed in practice when using a very large tolerance parameter ($\varepsilon>3$), showing that mixing more patches than needed is detrimental. The collaborative filtering strategy helps accelerating the algorithm up to a certain point, but a trade-off with performance needs to be considered. 
\subsection{Parameter choices}
\label{ssec:paramsAnalisis}
The four parameters of the Normal-Wishart distribution: $\k$, $\nu$, the prior mean $\mean_0$ and the prior covariance matrix $\S_0$, must be set in order to compute $\mean$ and $\S$. 
\paragraph{Choice of $\k$ and $\nu$}
The computation of $\mean$ according to~\eqref{eq:muhat} combines the mean $\sum_{i=1}^M \mathbf{A}_i \y_i $ estimated from the similar patches and the prior mean $\mean_0$. The parameter $\k$ is related to the degree of confidence we have on the prior $\mean_0$. Hence, its value should be a trade-off between the confidence we have in the prior accuracy vs. the one we have in the information provided by the similar patches.
The latter improves when both $M$ (\emph{i.e.} the number of similar patches) and $P=\operatorname{trace}(\U_i)$ (\emph{i.e.} the number of known pixels in the current patch) increase. These intuitive insights suggest the following rule to set $\kappa$:
\begin{equation}
\kappa = M\alpha,  \quad \alpha = \left\{
\begin{array}{rl}
\alpha_L & \text{if $P$ and $M$} > \text{threshold} \\
\alpha_H & \text{otherwise}.
\end{array} \right.
\end{equation}
A similar reasoning leads to the same rule for $\nu$,
\begin{equation}
\nu = M\alpha + n
\end{equation}
where the addition of $n$ ensures the condition $\nu > n - 1$ required by the Normal-Wishart prior to be verified.

This rule is used to obtain the experimental results presented in Section~\ref{ssec:expsNewMethod}, and proved to be a consistently good choice despite its simplicity. However, setting these parameters in a more general setting is not a trivial task and should be the subject of further study. In particular we could explore a more continuous dependence of $\alpha$ on $P$, $M$, and possibly a third term $Q=\sum_{i=1}^n S_{ii}$ where $S = \sum_{j=1}^M \L^{-1} \U_j \L^{*}_j \U_j$. This term estimates to what an extent similar patches cover the missing pixels in the current patch. 

\paragraph{Setting of $\mean_0$ and $\S_0$} Assuming an oracle image $\f_\mathrm{oracle}$ is available (see details in Section~\ref{ssec:pasosAlgo}), $\mean_0$ and $\S_0$ can be computed using the classical MLE estimators from a set of similar patches $(\tilde{\c}_1,\dots,\tilde{\c}_M)$ taken from $\f_\mathrm{oracle}$
\begin{equation}
\mean_0 = \frac{1}{M} \sum_{j=1}^{M} \tilde{\c}_j, \quad \S_0 = \frac{1}{M-1} \sum_{j=1}^{M} (\tilde{\c}_j - \mean_0)(\tilde{\c}_j - \mean_0)^T. 
\label{eq:musec}
\end{equation}
This is the same approach followed by Lebrun et al.~\cite{lebrun13} to estimate the patch model parameters in the case of denoising. 

\subsection{Initialization}
\label{ssec:initialization}
A good initialization is crucial since we aim at solving a non-convex problem through an iterative procedure. To initialize the proposed algorithm we follow the approach proposed by Yu et al.~\cite{yu12} (described in detail in Appendix A in the supplementary material). They propose to initialize the PLE algorithm by learning the $K$ GMM covariance matrices from synthetic images of edges with different orientations as well as from the DCT basis to represent isotropic patterns. As they state, in dictionary learning, the most prominent atoms represent local edges which are useful to represent and restore contours. Hence, this initialization helps to correctly restore corrupted patches even in quite extreme cases. The oracle of the first iteration of the proposed approach is the output of the first iteration of the PLE algorithm.

\subsection{Computational complexity}
With the original per-patch strategy, the complexity of the algorithm is given by step 3 in Algorithm 1: $[(4n^3 + n^3/3)M + n^3/3]$, so the total complexity is $[(4n^3 + n^3/3)M + n^3/3] \times maxIts \times maxIts_2 \times T$ (where $T =$ total number of patches to be restored and assuming the Cholesky factorization is used for matrix inversion). The collaborative filtering strategy reduces this value by a factor that depends on the number of groups of similar patches, which depends on the image contents and the distance tolerance parameter $\varepsilon$. The main difference with the PLE algorithm complexity ($(3n^3 + n^3/3) \times its_{PLE} \times T$) is a factor given by the number of groups defined by the collaborative filtering approach and the ratio between $its_{PLE}$ and $maxIts \times maxIts_2$. As mentioned by Yu et al.~\cite{yu12}, computational complexity can be further reduced in the case of binary masks by removing the zero rows and inverting a matrix of size $n^2/S \times n^2/S$ instead of $n^2 \times n^2$ where $S$ is the masking ratio. Moreover, the proposed algorithm can be run in parallel in different image subregions thus allowing for even further acceleration in multiple-core architectures. 
The complexity comparison to NLB needs to be made in the case where the degradation is additive noise with constant variance (translation invariant degradation), which is the task performed by NLB. In that case, the complexity of the proposed approach (without considering collaborative filtering nor parallelization, which are both done also in NLB), is $11n^3/3 \times maxIts \times maxIts_2 \times T$ whereas that of NLB is $2 \times (4n^3/3)$. 

\section{Image Restoration Experiments}
\label{ssec:expsNewMethod}
In this section we illustrate the ability of the proposed method to solve several image inverse problems. Both synthetic (i.e., where we have added the degradation artificially) and real data (i.e., issued from a real acquisition process) are used. The considered problems are: interpolation, combined interpolation and denoising, denoising, and zooming. The reported values of peak signal-to-noise ratio ($\psnr{} = 20\log_{10}(255/\sqrt{MSE})$) are averaged over 10 realizations for each experiment (variance is below 0.1 for interpolation and below 0.05 for combined interpolation and denoising and for denoising only). Similar results are obtained with the structural similarity index (SSIM) which is included in the supplementary material (Appendix B).
\begin{table*}
\footnotesize
\setlength{\tabcolsep}{2.2pt}
\centering
\begin{tabular}[h]{c c c c c c c c c c c c c c c c c c c c c}
\toprule
  & \multicolumn{12}{c}{\textbf{Interpolation - PSNR (dB)}} & \multicolumn{8}{|c}{\textbf{Interpolation \& Denoising - PSNR (dB)}} \\
  \cmidrule{1-21}
 & \multicolumn{4}{c}{50\%} & \multicolumn{4}{c}{70\%} & \multicolumn{4}{c|}{90\%} & \multicolumn{4}{c}{70\%} & \multicolumn{4}{c}{90\%}\\
\cmidrule{2-21}
   & \hpnlb{} & PLE & EPLL & \multicolumn{1}{c|}{E-PLE} & \hpnlb{} & PLE & EPLL & \multicolumn{1}{c|}{E-PLE} & \hpnlb{} & PLE & EPLL & \multicolumn{1}{c|}{E-PLE} & \hpnlb{} & PLE & EPLL & \multicolumn{1}{c|}{E-PLE} & \hpnlb{} & PLE & EPLL & E-PLE\\
\cmidrule{2-21}
barbara   & \textbf{39.11}	 & 	36.93	 & 	32.99	 & 	\multicolumn{1}{c|}{35.43}	 & 	\textbf{34.69}	 & 	32.50	 & 	27.96	 & 	\multicolumn{1}{c|}{28.77} & \textbf{24.86}	& 23.62 & 23.30 & 	\multicolumn{1}{c|}{23.26}	 & \textbf{33.34} & 31.99 & 27.63 & \multicolumn{1}{c|}{27.75} & \textbf{24.57} & 23.53 & 23.27 & 23.20 \\
boat      & \textbf{34.92}	 & 	34.32	 & 	34.21	 & 	\multicolumn{1}{c|}{33.59}	 & 	\textbf{31.37}	 & 	30.74	 & 	30.38	 & 	\multicolumn{1}{c|}{30.26}	& \textbf{25.96} & 25.35 & 24.72 & 	\multicolumn{1}{c|}{25.43}  & \textbf{30.61} & 30.41 & 30.15 & \multicolumn{1}{c|}{29.54} & \textbf{25.78} & 25.45 & 24.71 & 25.47  \\
traffic   & 30.17	 & 	30.12	 & 	\textbf{30.19}	 & 	\multicolumn{1}{c|}{28.86}	 & 	\textbf{27.27}	 & 	27.12	 & 	27.13	 & 	\multicolumn{1}{c|}{26.64} & \textbf{22.84}	& 22.34 & 21.85 & 	\multicolumn{1}{c|}{22.27}	& 26.99 & 26.98 & \textbf{27.05} & \multicolumn{1}{c|}{26.35} & \textbf{22.87}	& 22.43 & 22.21 & 22.35 \\
\toprule
  & \multicolumn{12}{c|}{\textbf{Denoising - PSNR (dB)}} & & \multicolumn{5}{c}{\textbf{Zooming - PSNR (dB)}} \\
  \cmidrule{1-21}
$\sigma^2$ & \multicolumn{3}{c}{10} & \multicolumn{3}{c}{30} & \multicolumn{3}{c}{50} & \multicolumn{3}{c|}{80} & & \multicolumn{5}{c}{$\times 2$}\\
\cmidrule{1-21}
& \hpnlb{} & NLB & \multicolumn{1}{c|}{EPLL} & \hpnlb{} & NLB & \multicolumn{1}{c|}{EPLL} & \hpnlb{} & NLB & \multicolumn{1}{c|}{EPLL} & \hpnlb{} & NLB & \multicolumn{1}{c|}{EPLL} & & \hpnlb{} & PLE & EPLL & E-PLE & Lanczos\\
\cmidrule{2-21} 
barbara  & \textbf{41.26}	 & 	41.20	 & 	\multicolumn{1}{c|}{40.56}	 & 	\textbf{38.40}	 & 	38.26	 & 	\multicolumn{1}{c|}{37.32}	 & 	\textbf{37.13}	 & 	36.94	 & 	\multicolumn{1}{c|}{35.84}	 & 	\textbf{35.96}	 & 	35.73	 & 	\multicolumn{1}{c|}{34.51}	& & \textbf{38.17} &  37.11 & 31.34 & 36.51 & 28.01 & &  \\
boat      & \textbf{40.05}	 & 	39.99	 & 	\multicolumn{1}{c|}{39.47}	 & 	36.71	 & 	\textbf{36.76}	 & 	\multicolumn{1}{c|}{36.34}	 & 	35.41	 & 	\textbf{35.46}	 & 	\multicolumn{1}{c|}{35.13}	 & 	34.30	 & 	\textbf{34.33}	 & 	\multicolumn{1}{c|}{34.12}	& & \textbf{32.35} & 31.96 & 31.95 & 32.08 & 29.60 & &  \\
traffic   & 40.73	 & 	\textbf{40.74}	 & 	\multicolumn{1}{c|}{40.55}	 & 	\textbf{37.03}	 & 	36.99	 & 	\multicolumn{1}{c|}{36.86}	 & 	\textbf{35.32}	 & 	35.26	 & 	\multicolumn{1}{c|}{35.20}	 & 	\textbf{33.78}	 & 	33.70	 & 	\multicolumn{1}{c|}{33.72}	& & 25.05 & 24.78  & \textbf{25.17} & 24.91 & 21.89 & &  \\
\toprule
\end{tabular}
\caption{Results of the interpolation, combined interpolation and denoising, denoising and zooming tests described in Section~\ref{ssec:synthExps}. Patch size of $8 \times 8$ for all methods in all tests. Parameter setting for interpolation, combined interpolation and denoising, and zooming, \hpnlb{}: $\alpha_H = 1$, $\alpha_L = 0.5$, PLE: $\sigma = 3$, $\varepsilon = 30$, $K = 19$~\cite{yu12}, EPLL: default parameters~\cite{zoran11_web}, E-PLE: parameters set as specified in~\cite{wang13b}. Parameter setting for denoising, \hpnlb{}: $\alpha_H = \alpha_L = 100$, NLB: code provided by the authors~\cite{lebrun13IPOL} automatically sets parameters from input $\sigma^2$, EPLL: default parameters for the denoising example~\cite{zoran11_web}}
\label{tab:psnrInterp}
\end{table*}

\subsection{Synthetic degradation}
\label{ssec:synthExps}
\paragraph{Interpolation} Random masks with 50\%, 70\% and 90\% missing pixels are applied to the tested ground-truth images. The interpolation performance of the proposed method is compared to that of PLE~\cite{yu12}, EPLL~\cite{zoran11} and E-PLE~\cite{wang13b} using a patch size of $8 \times 8$ for all methods. PLE parameters are set as indicated in~\cite{yu12} ($\sigma = 3$, $\varepsilon = 30$, $K = 19$). We used the EPLL code provided by the authors~\cite{zoran11_web} with default parameters and the E-PLE code available in~\cite{wang13b} with the parameters set as specified in the companion demo. The parameters for the proposed method are set to $\alpha_H = 1$, $\alpha_L = 0.5$ ($\alpha_H$ and $\alpha_L$ define the values for $\kappa$ and $\nu$, see Section~\ref{ssec:paramsAnalisis}). The \psnr{} results are shown in Table~\ref{tab:psnrInterp}. Figure~\ref{fig:syntheticExpsInterp} shows some extracts of the obtained results, the \psnr{} values for the extracts and the corresponding difference images with respect to the ground-truth. The proposed method gives sharper results than the other considered methods. This is specially noticeable on the reconstruction of the texture of the fabric of Barbara's trousers shown in the first row of Figure~\ref{fig:syntheticExpsInterp} or on the strips that appear through the car's window shown in the second row of the same figure. 
\begin{figure*}
\centering
\subfigure[Ground-truth]{\includegraphics[width=0.17\linewidth]{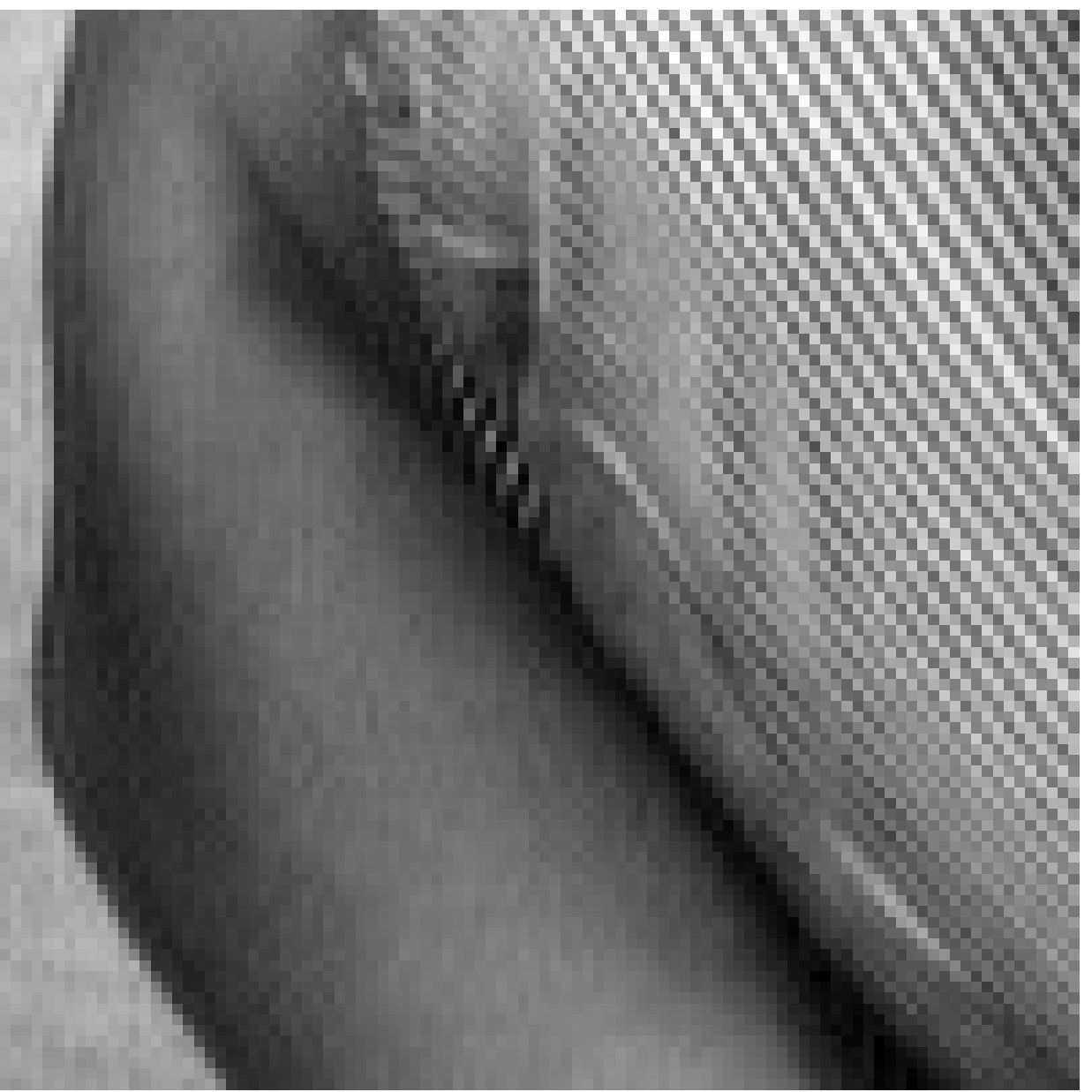}}\subfigure[\hpnlb{} (\textbf{30.01 dB})]{\includegraphics[width=0.17\linewidth]{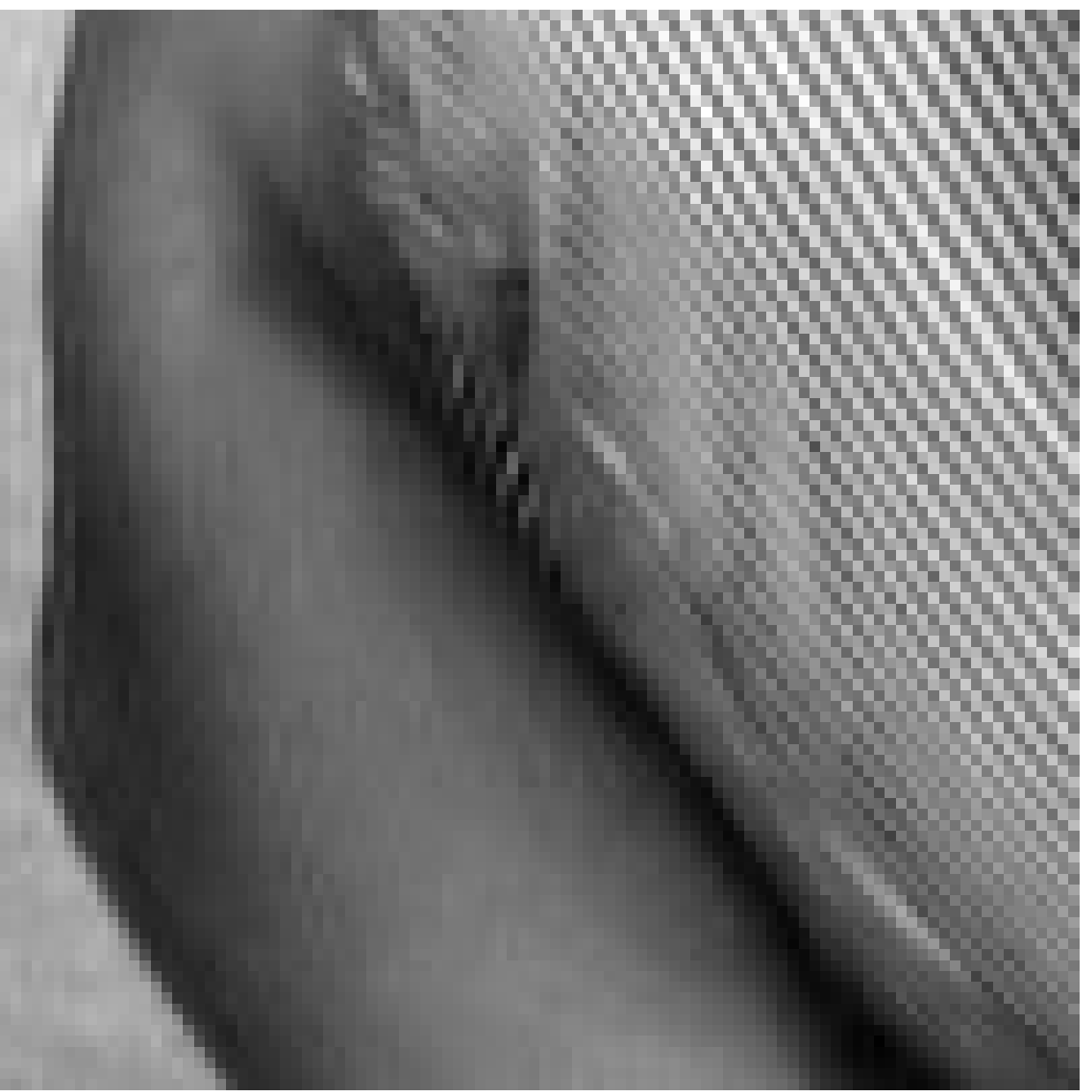}}\subfigure[PLE (26.78 dB)]{\includegraphics[width=0.17\linewidth]{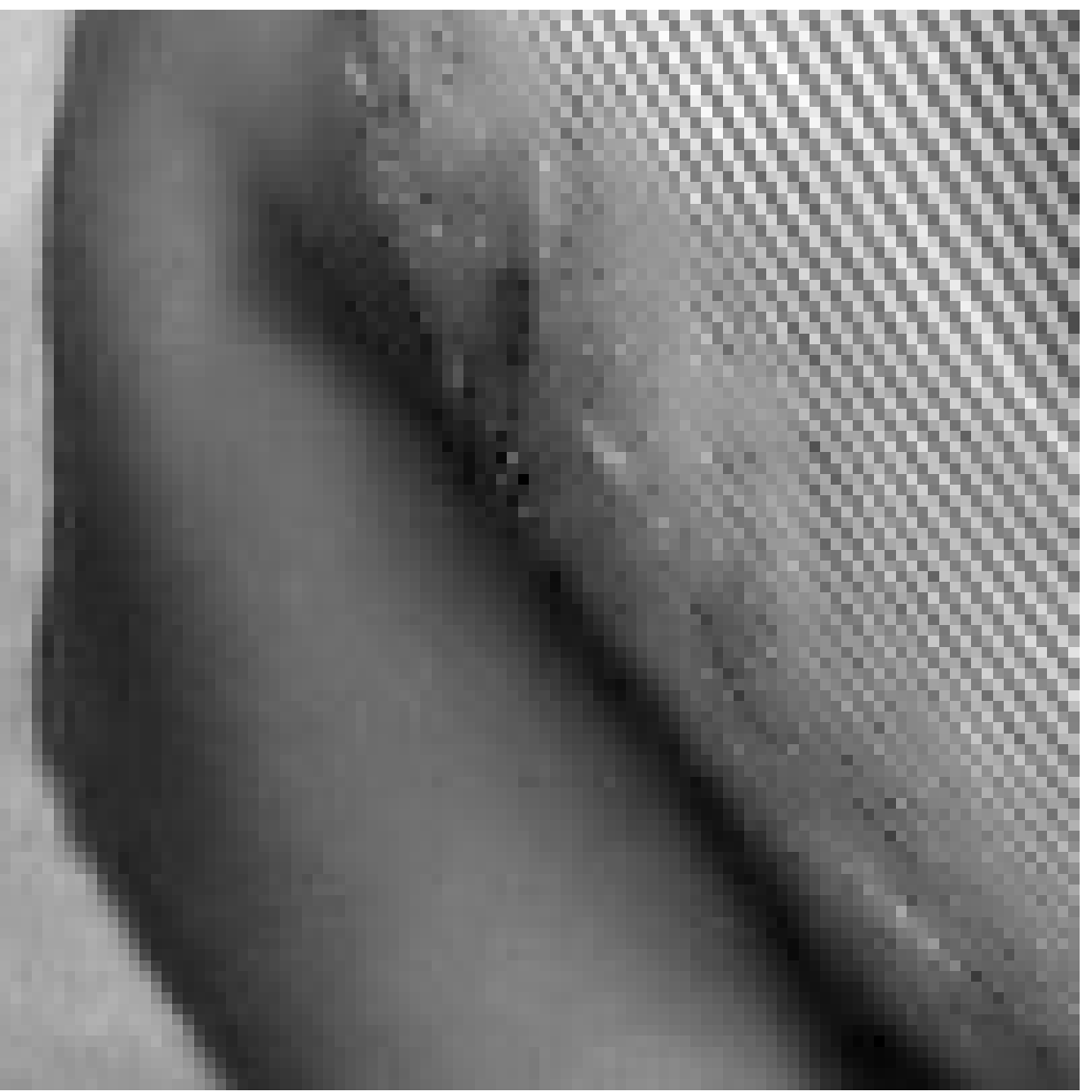}}\subfigure[EPLL (21.12 dB)]{\includegraphics[width=0.17\linewidth]{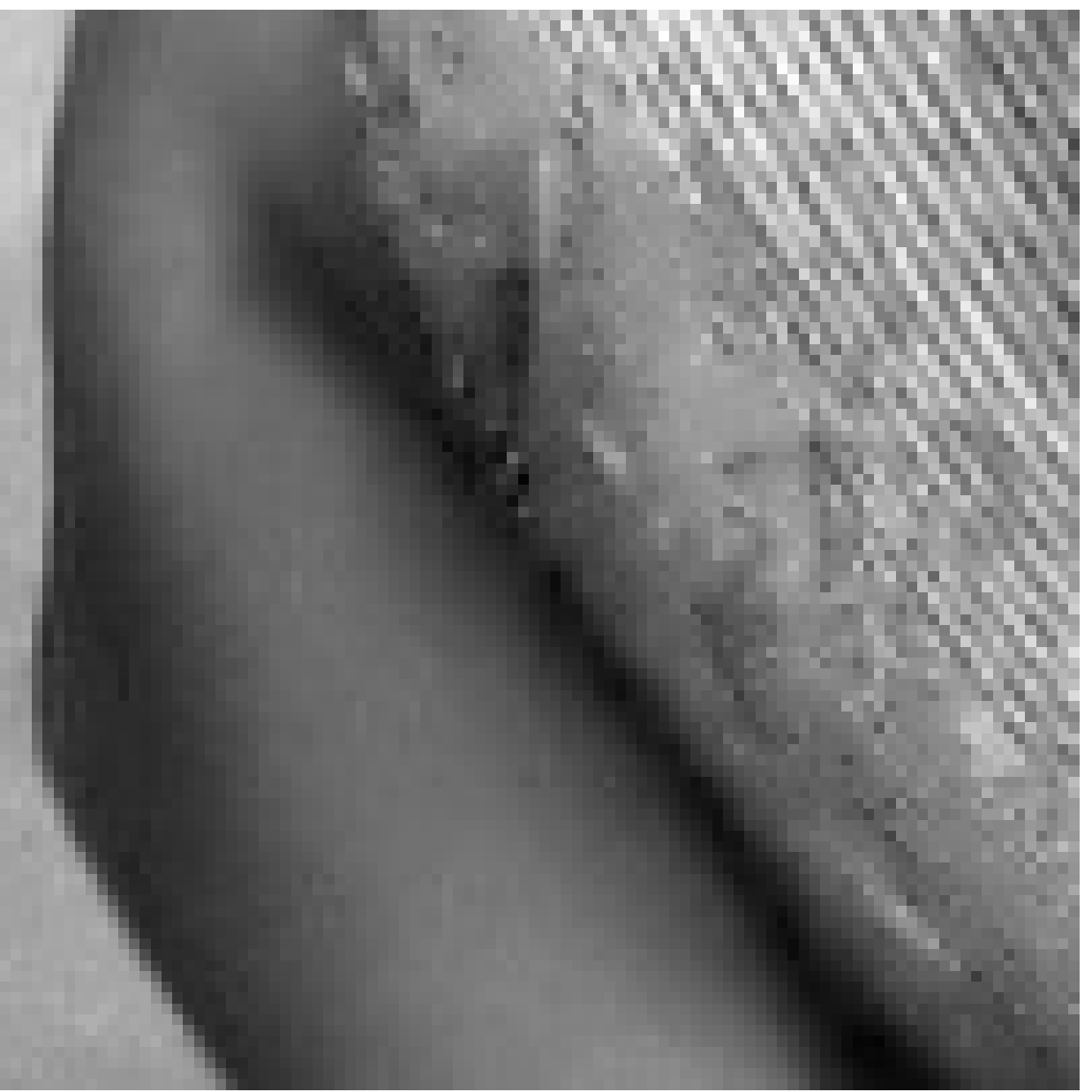}}\subfigure[E-PLE (23.12 dB)]{\includegraphics[width=0.17\linewidth]{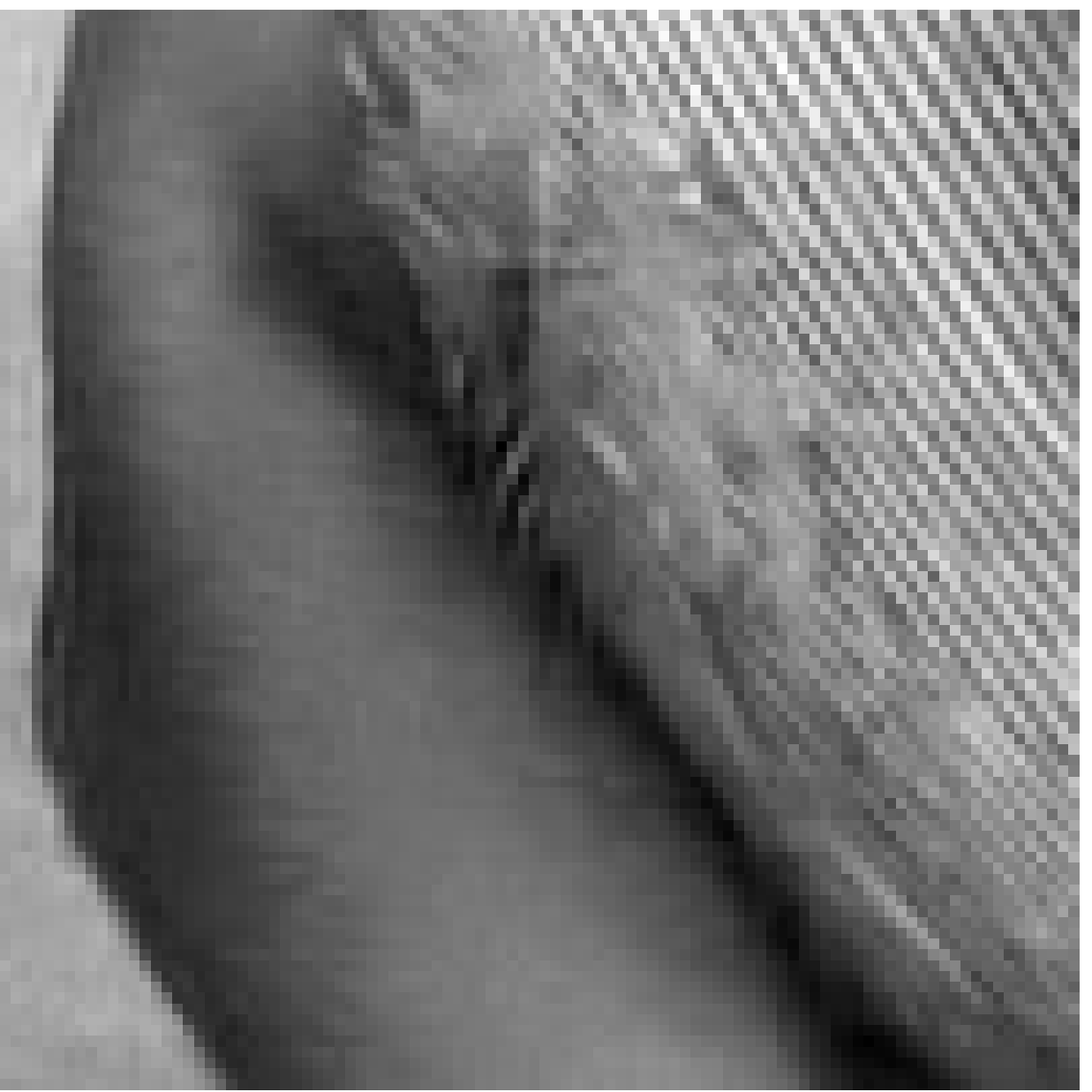}}
\includegraphics[width=0.17\linewidth]{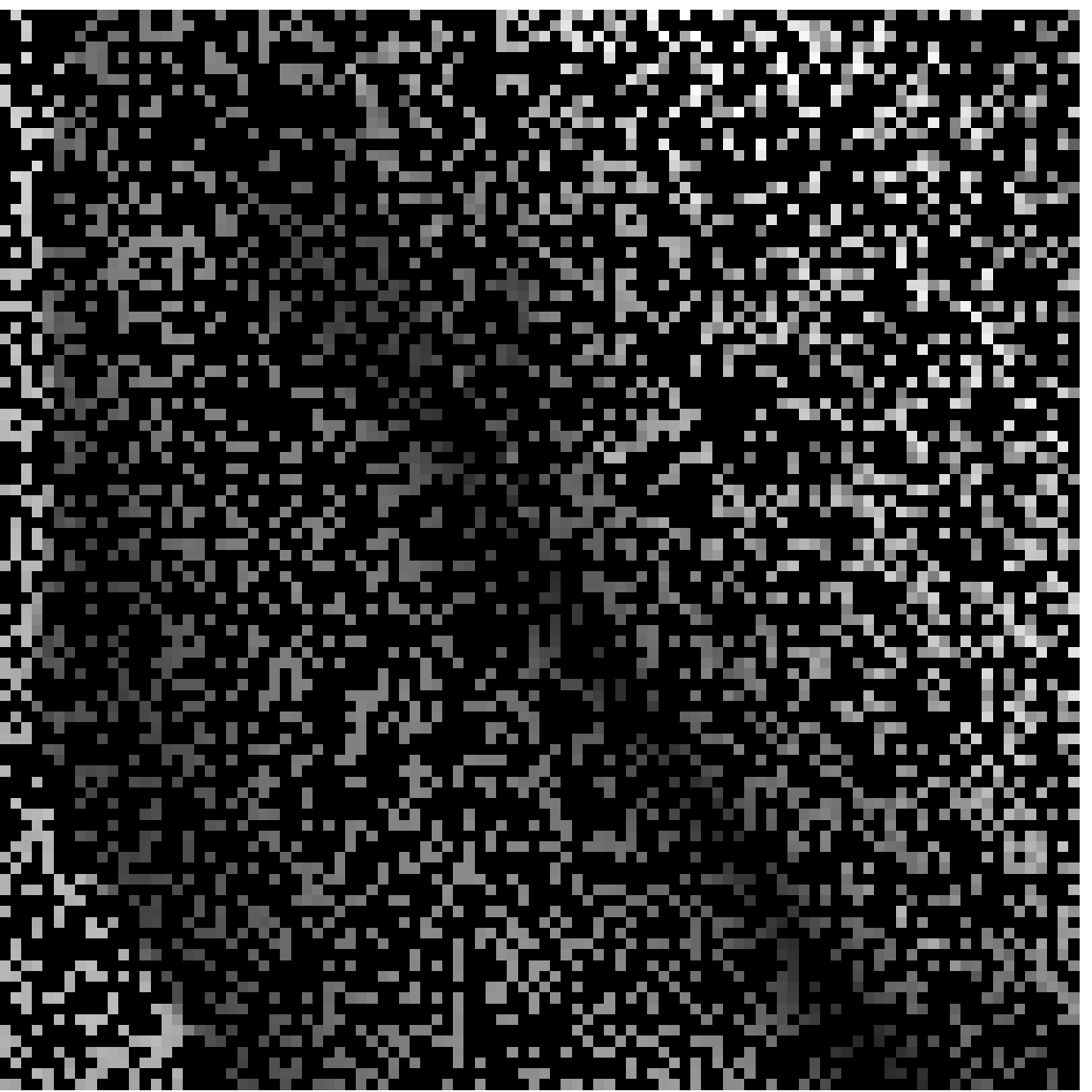}\includegraphics[width=0.17\linewidth]{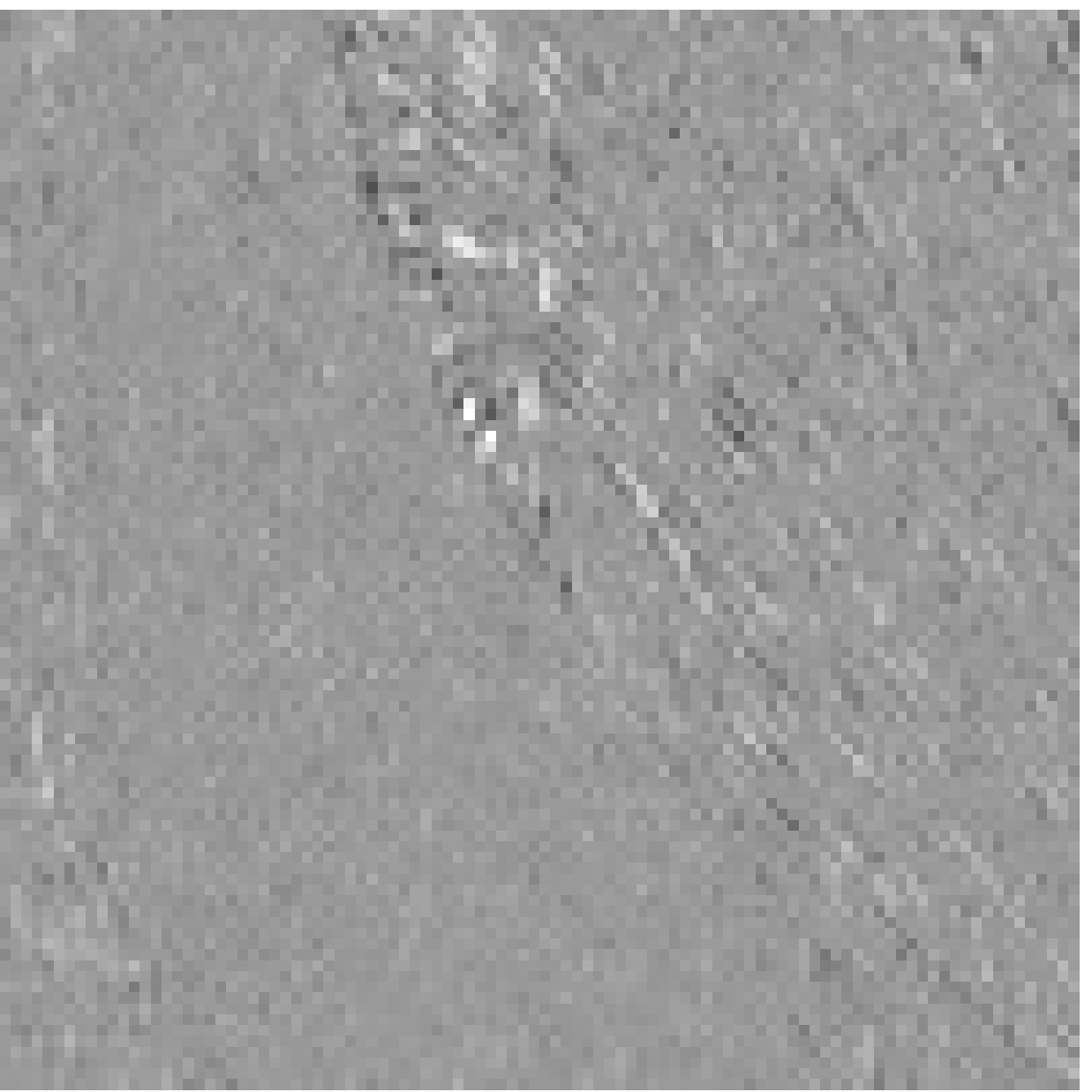}\includegraphics[width=0.17\linewidth]{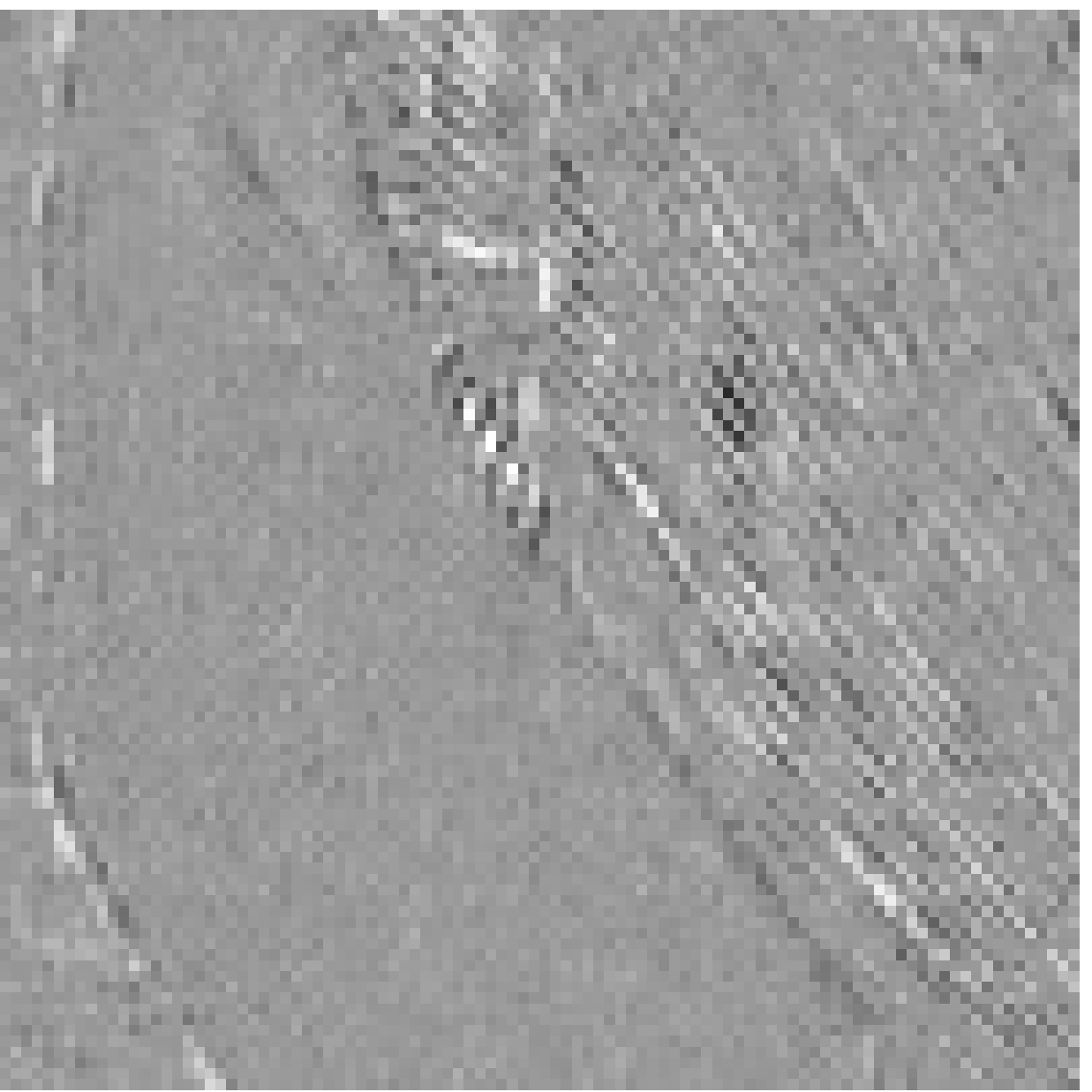}\includegraphics[width=0.17\linewidth]{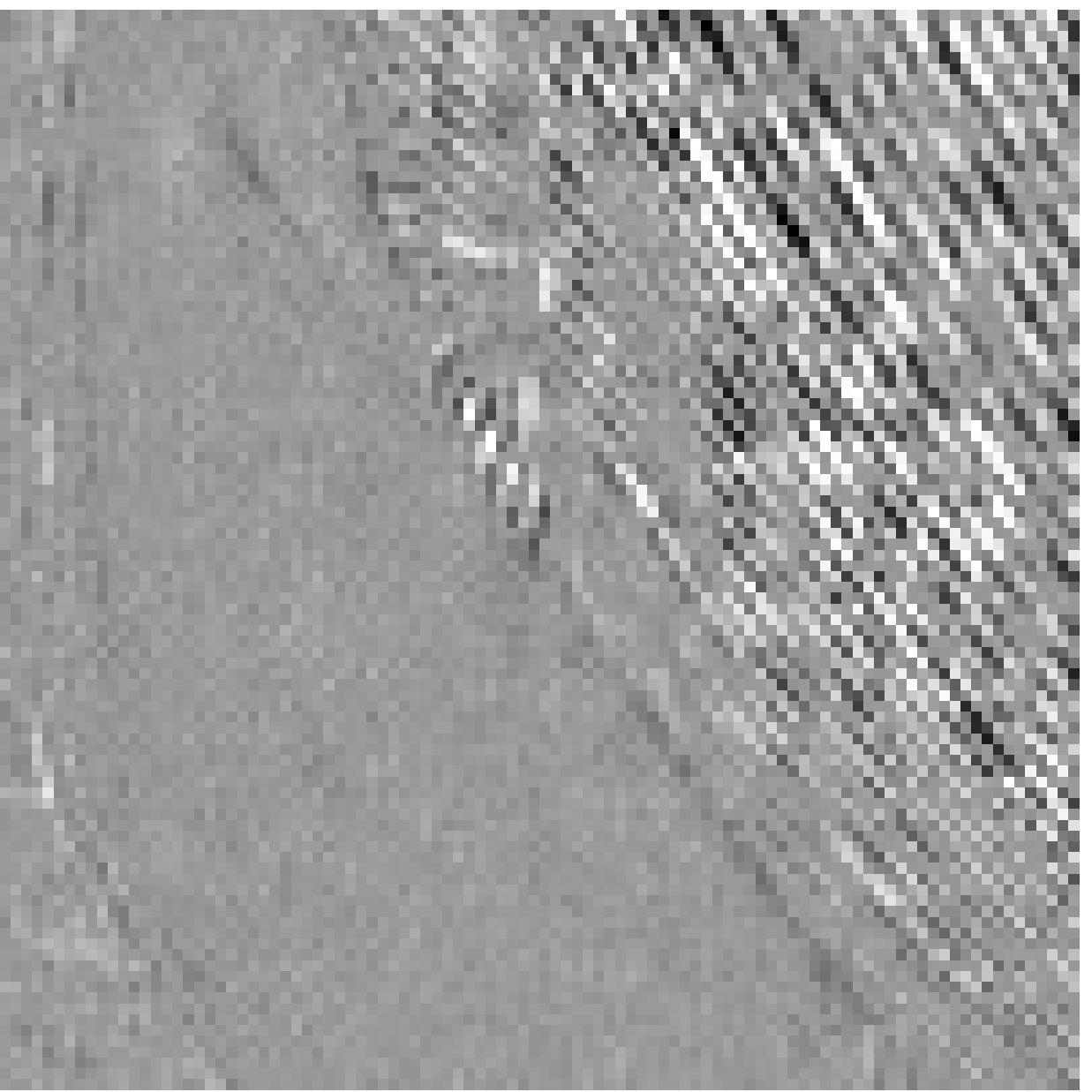}\includegraphics[width=0.17\linewidth]{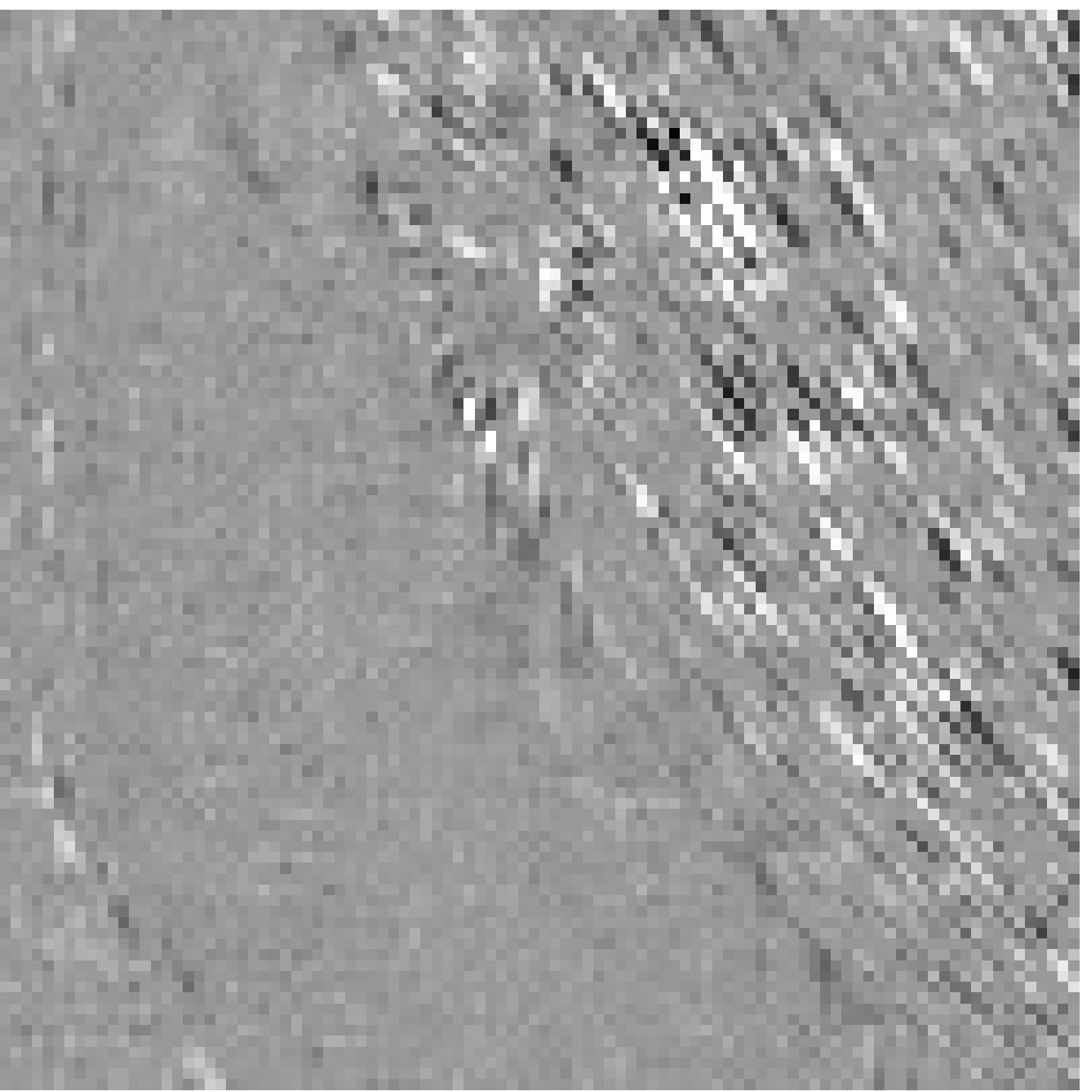}\\
\subfigure[Ground-truth]{\includegraphics[width=0.17\linewidth]{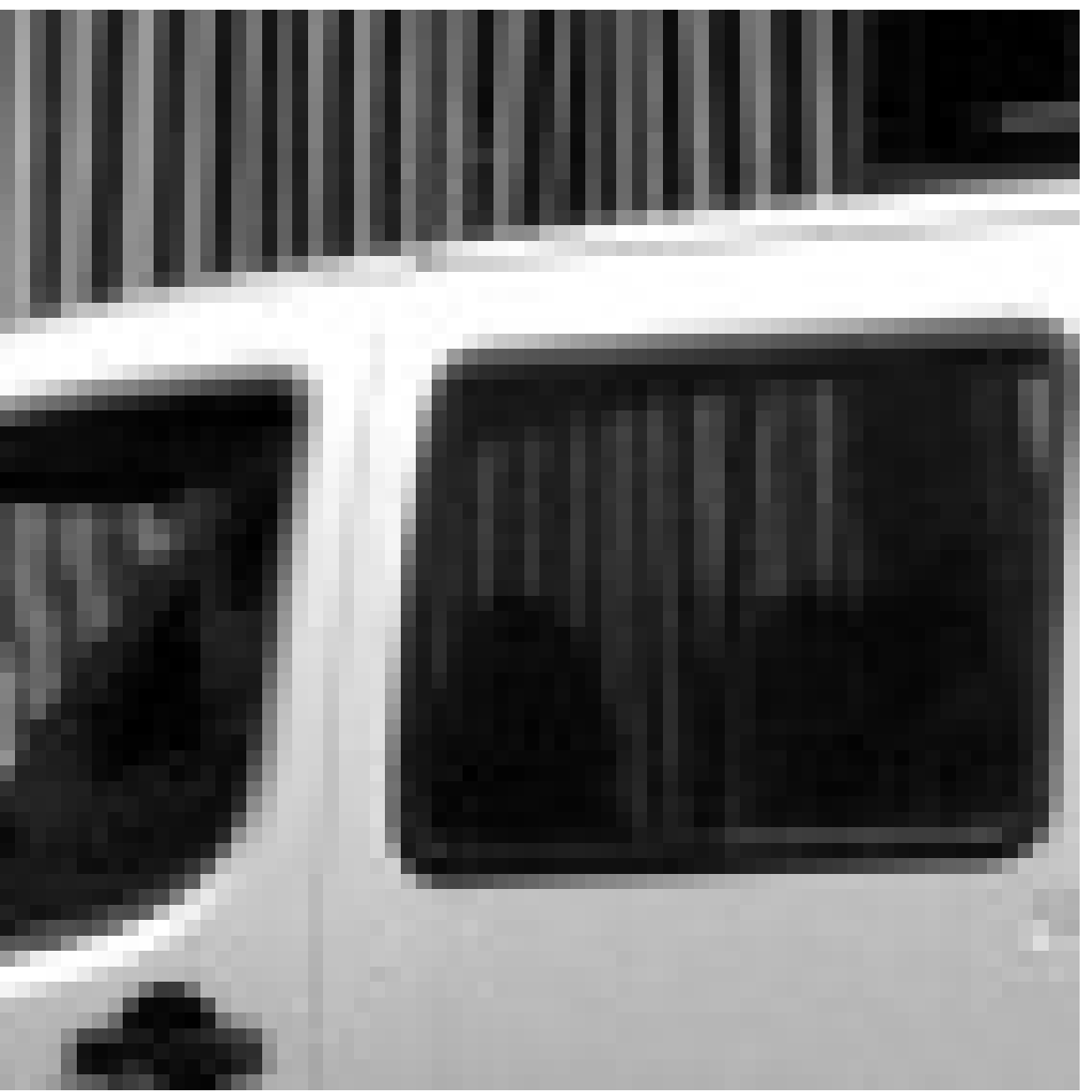}}\subfigure[\hpnlb{} (\textbf{30.20 dB})]{\includegraphics[width=0.17\linewidth]{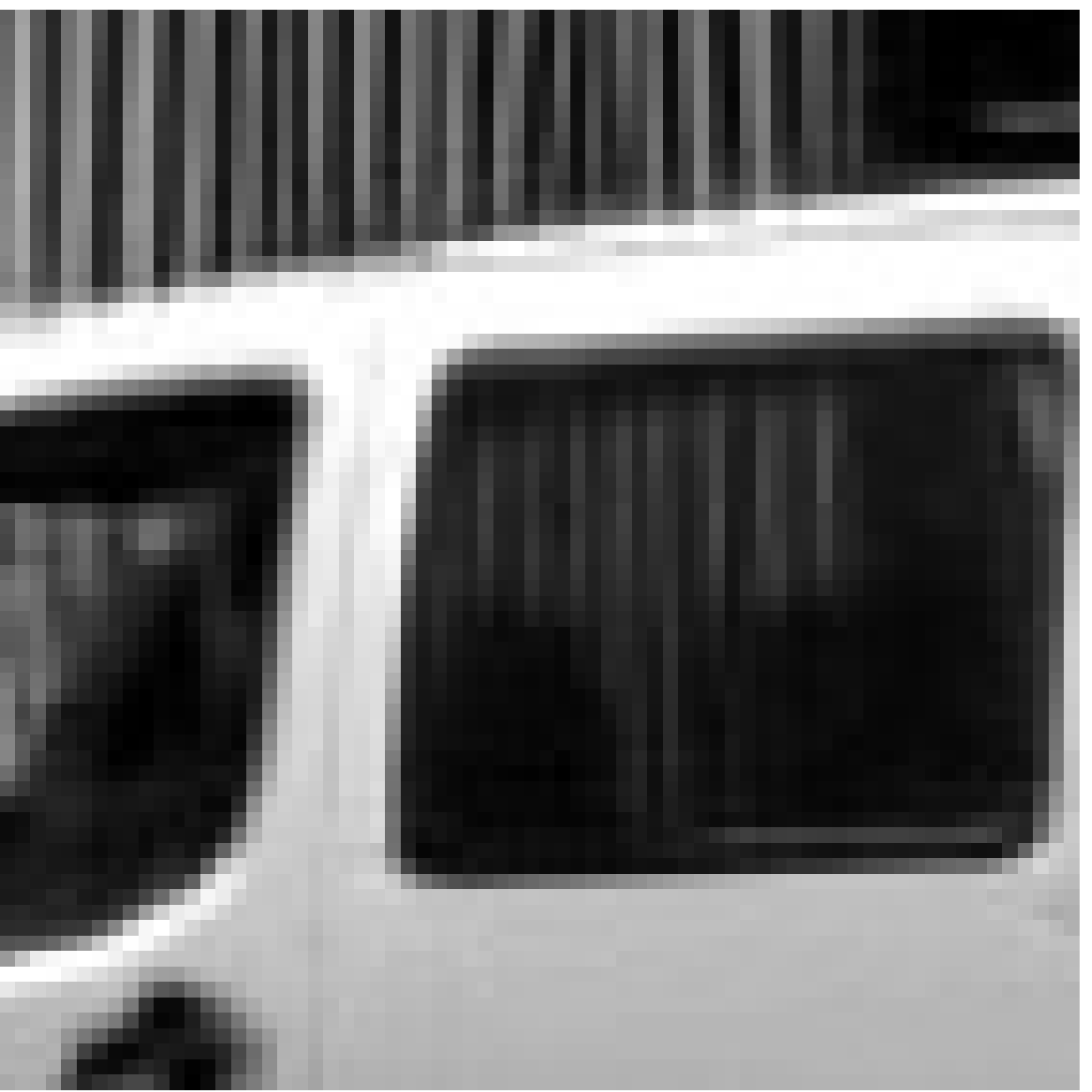}}\subfigure[PLE (27.89 dB)]{\includegraphics[width=0.17\linewidth]{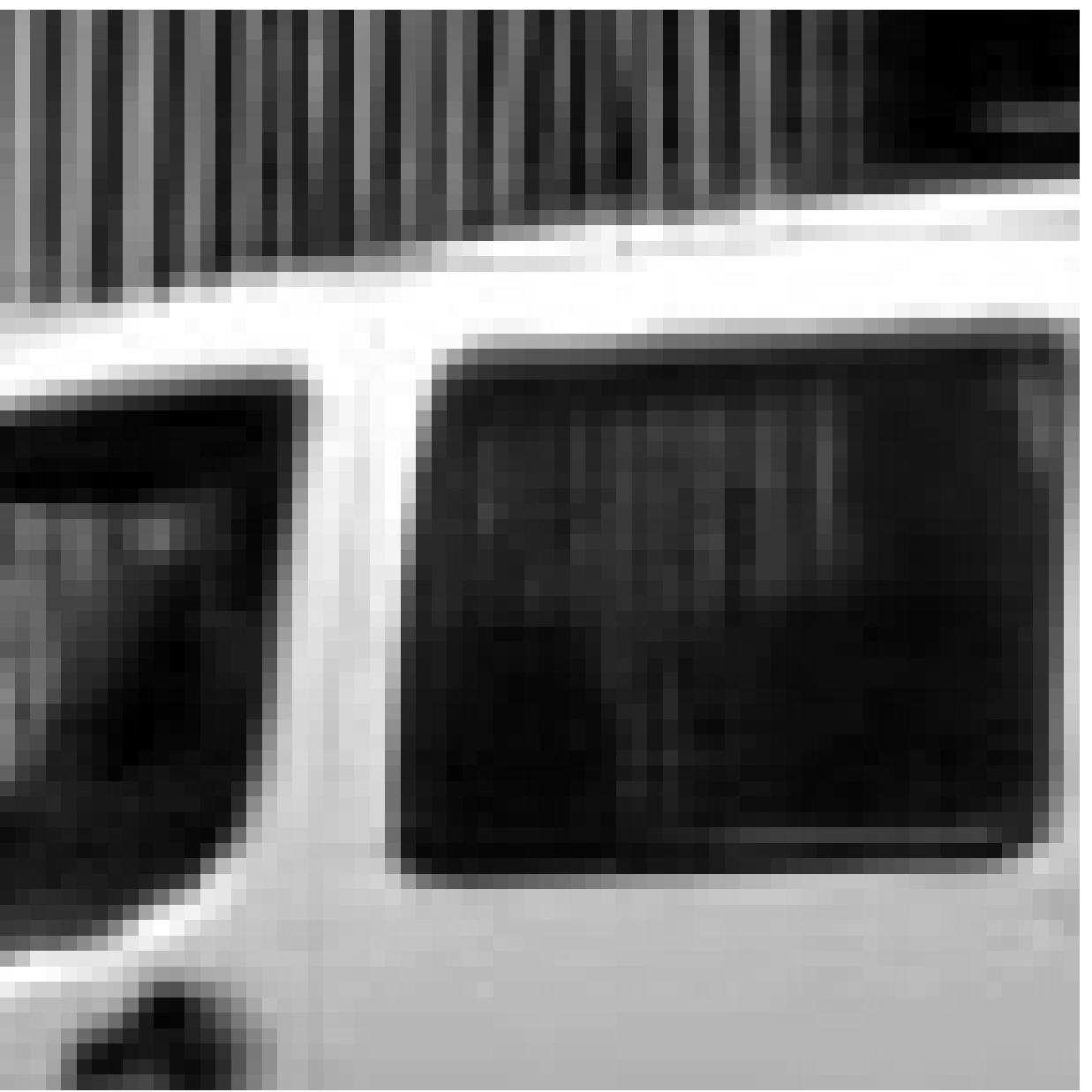}}\subfigure[EPLL (27.83 dB)]{\includegraphics[width=0.17\linewidth]{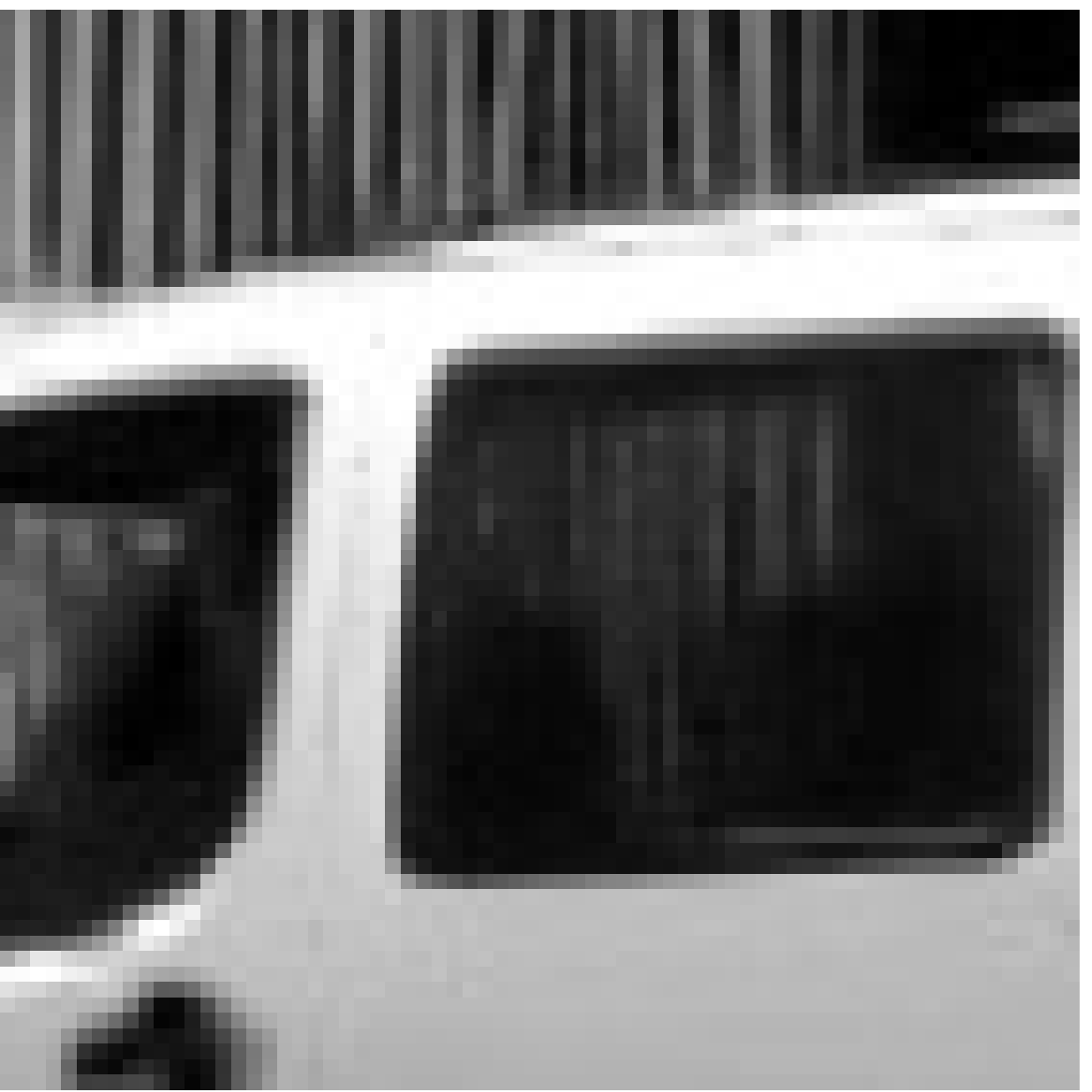}}\subfigure[E-PLE (26.79 dB)]{\includegraphics[width=0.17\linewidth]{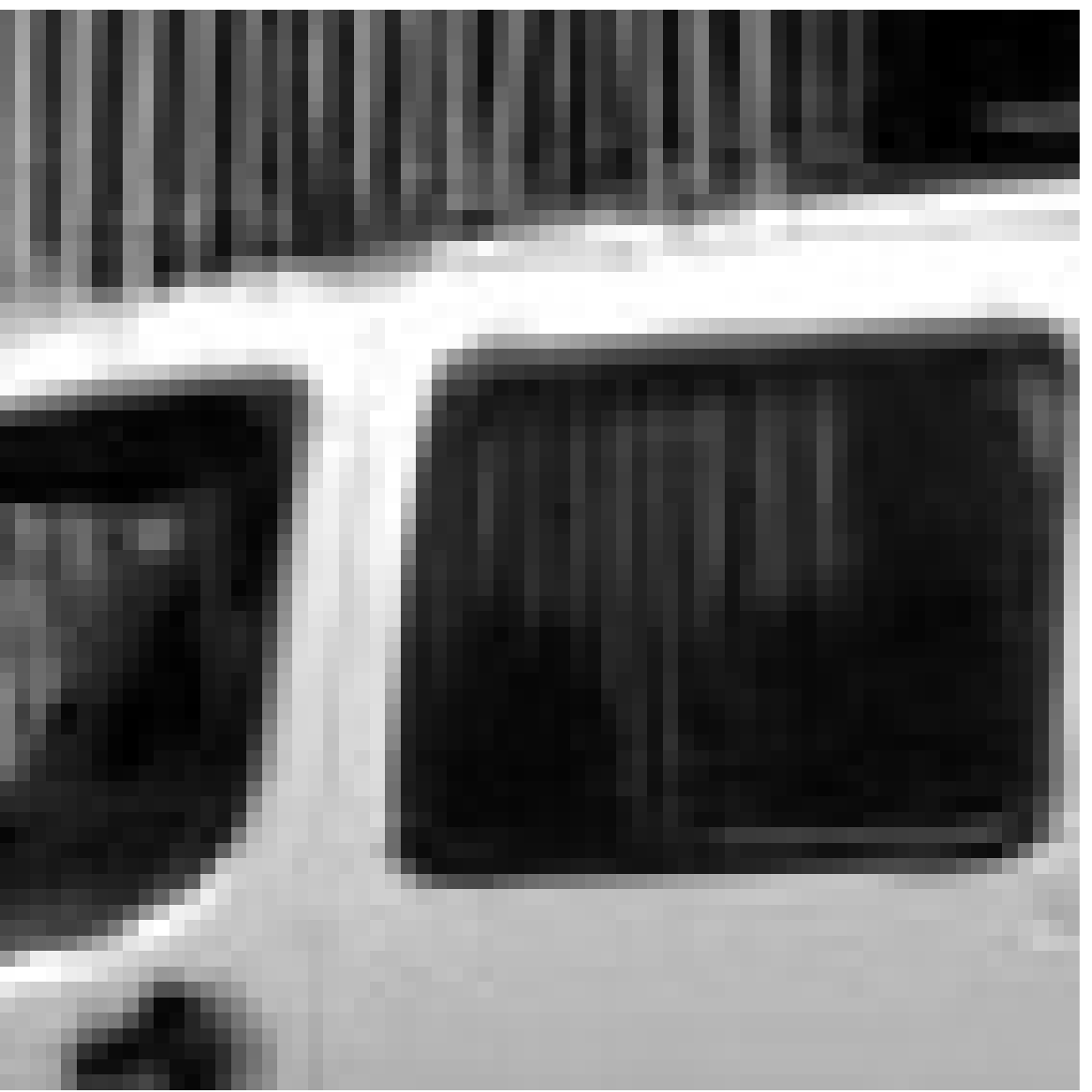}}\\
\includegraphics[width=0.17\linewidth]{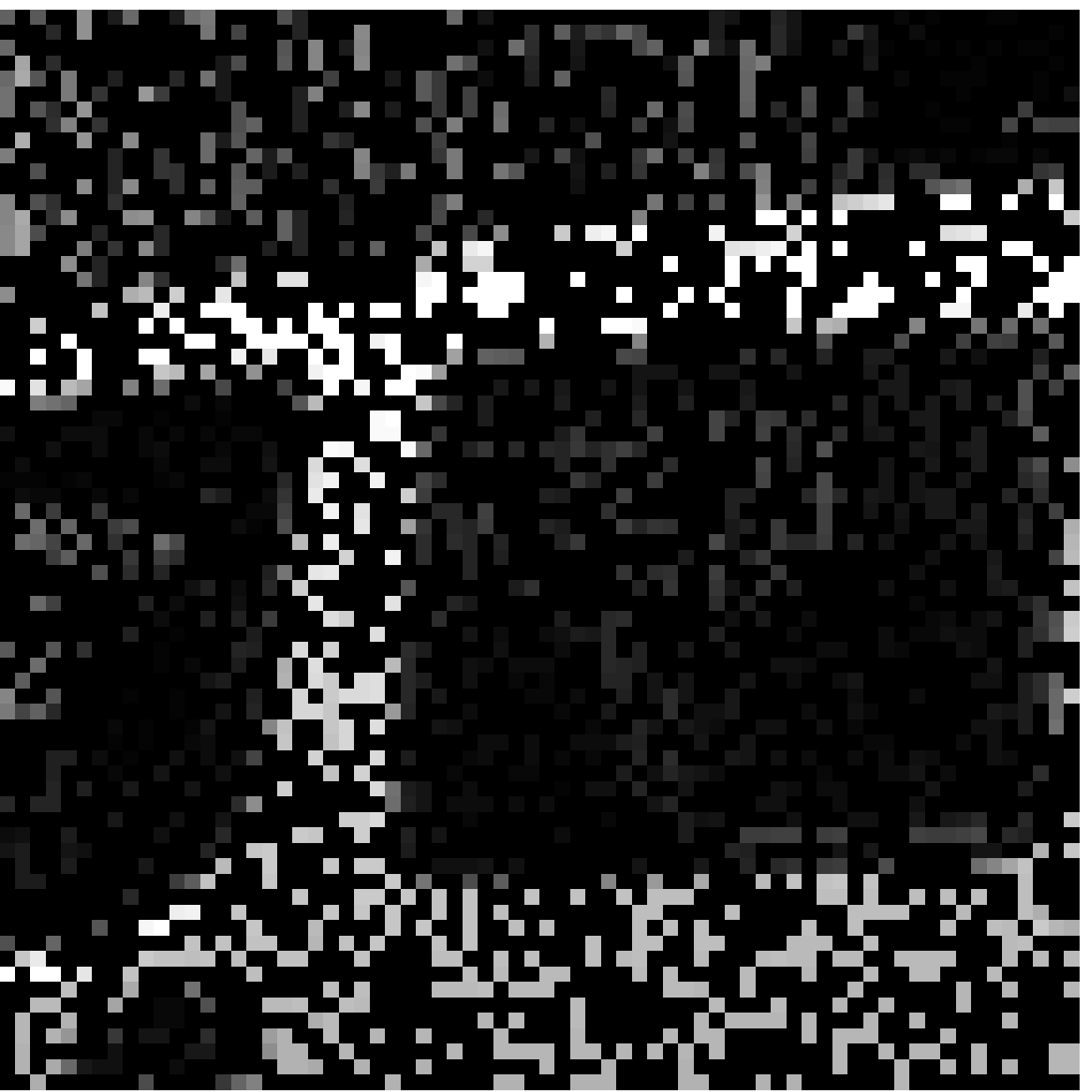}\includegraphics[width=0.17\linewidth]{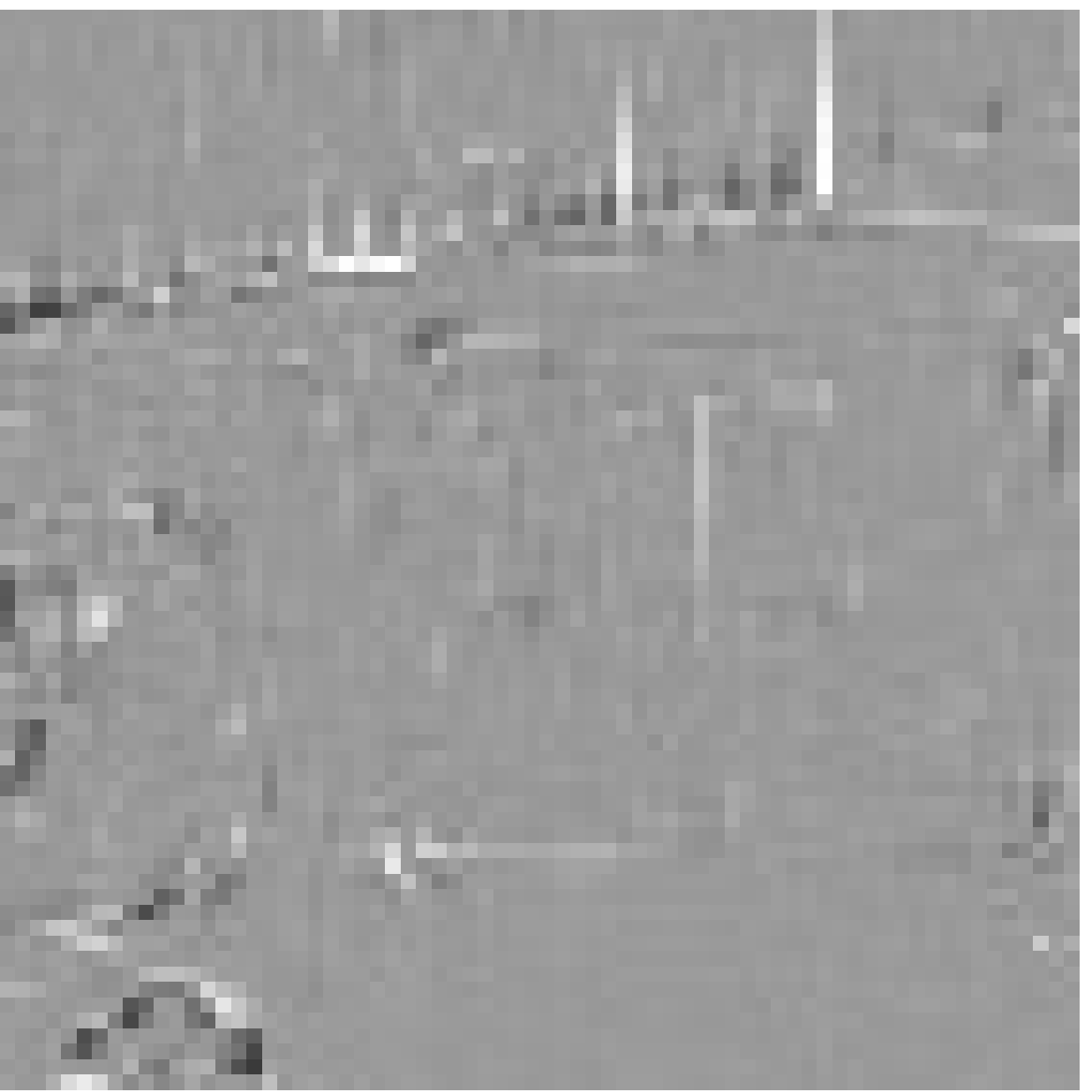}\includegraphics[width=0.17\linewidth]{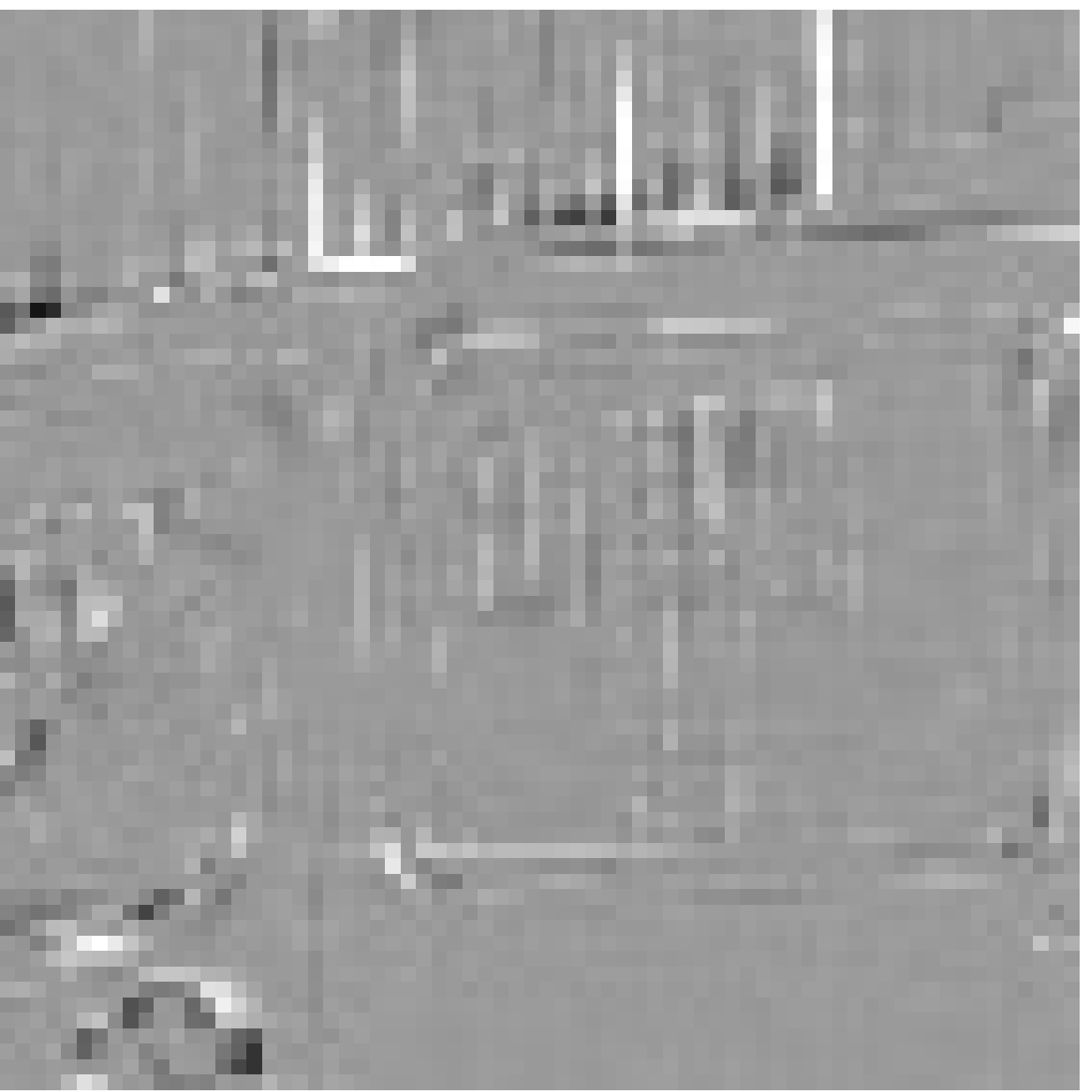}\includegraphics[width=0.17\linewidth]{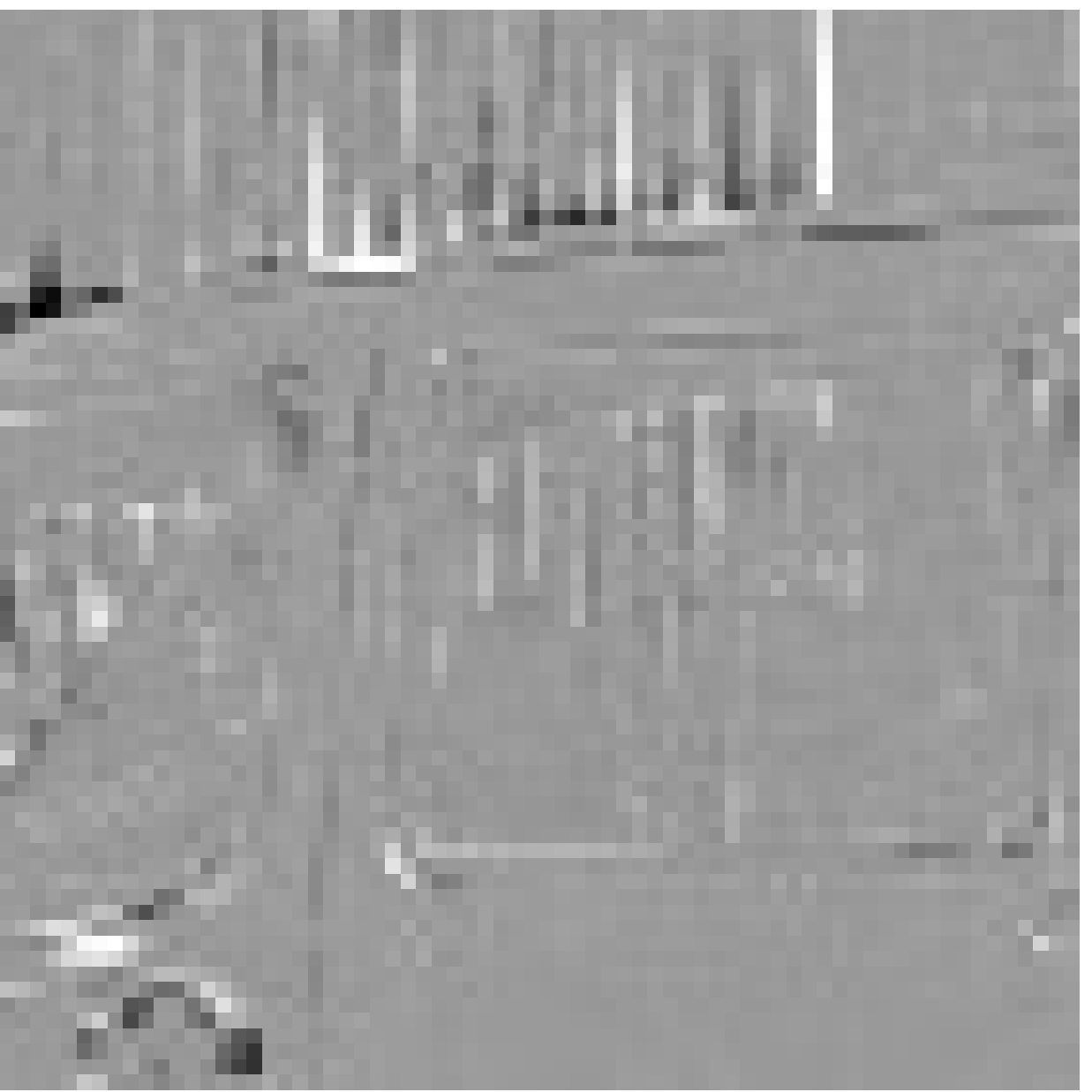}\includegraphics[width=0.17\linewidth]{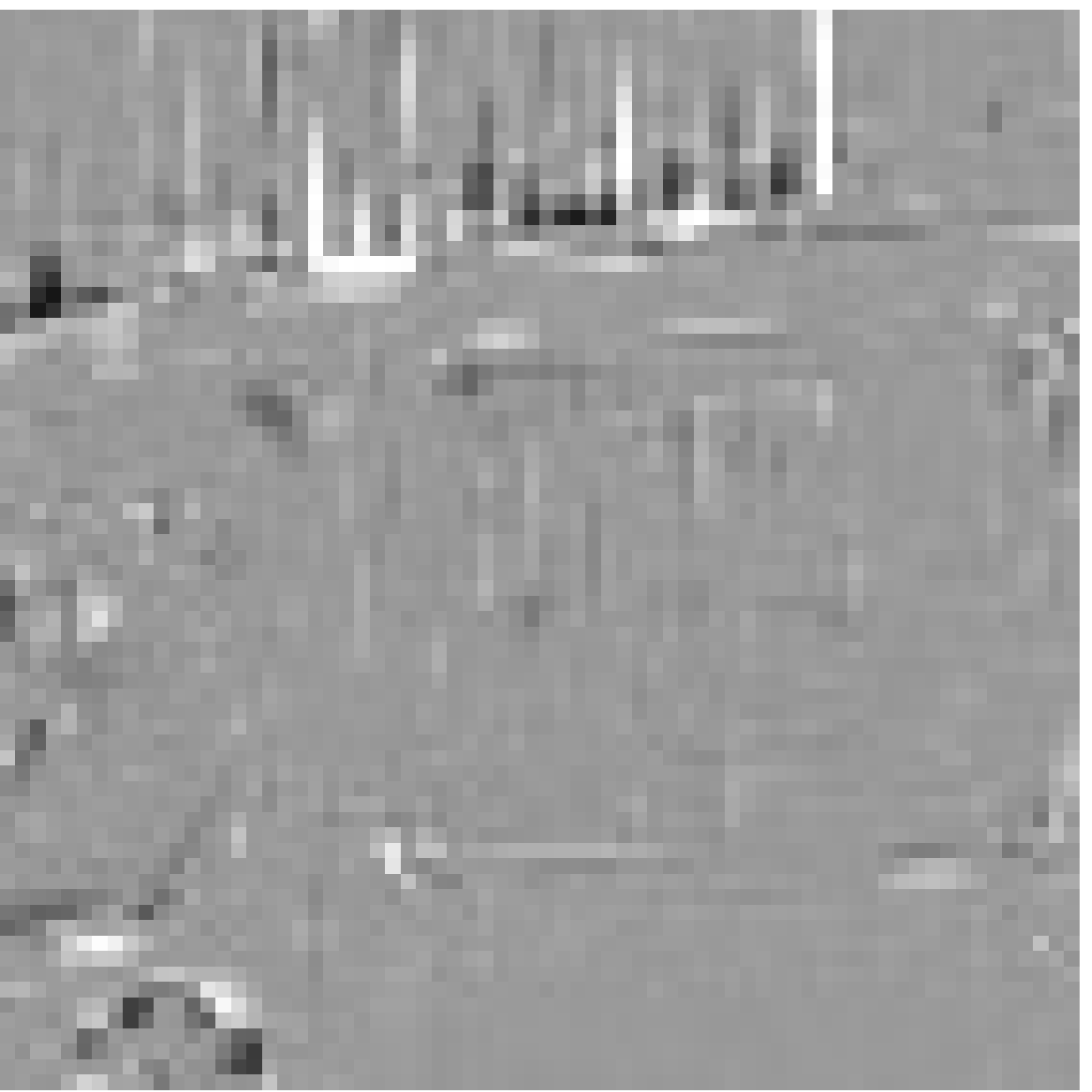}\\
\caption{\textbf{Synthetic data. Interpolation with 70\% of randomly missing pixels}. \textbf{Left to right:} (first row) Ground-truth (extract of barbara), result by \hpnlb{}, PLE, EPLL, E-PLE. (second row) input image, difference with respect to the ground-truth of each of the corresponding results. (third and fourth row) Idem for an extract of the traffic image. See Table~\ref{tab:psnrInterp} for the \psnr{} results for the complete images. Please see the digital copy for better details reproduction.}
\label{fig:syntheticExpsInterp}
\end{figure*}

\paragraph{Combined interpolation and denoising} For this experiment, the ground-truth images are corrupted with additive Gaussian noise with variance 10, and a random mask with 70\% and 90\% of missing pixels. The parameters for all methods are set as in the previous interpolation-only experiment. Table~\ref{tab:psnrInterp} summarizes the \psnr{} values obtained by each method. Figure~\ref{fig:syntheticExpsInterpDeno1} shows some extracts of the obtained results, the \psnr{} values for the extracts and the corresponding difference images with respect to the ground-truth. Once again, the results show that the proposed approach outperforms the others. Fine structures, such as the mast and the ropes of the ship, as well as textures, as in Barbara's headscarf, are much better preserved.
\begin{figure*}
\centering
\subfigure[Ground-truth]
{\includegraphics[width=0.17\linewidth]{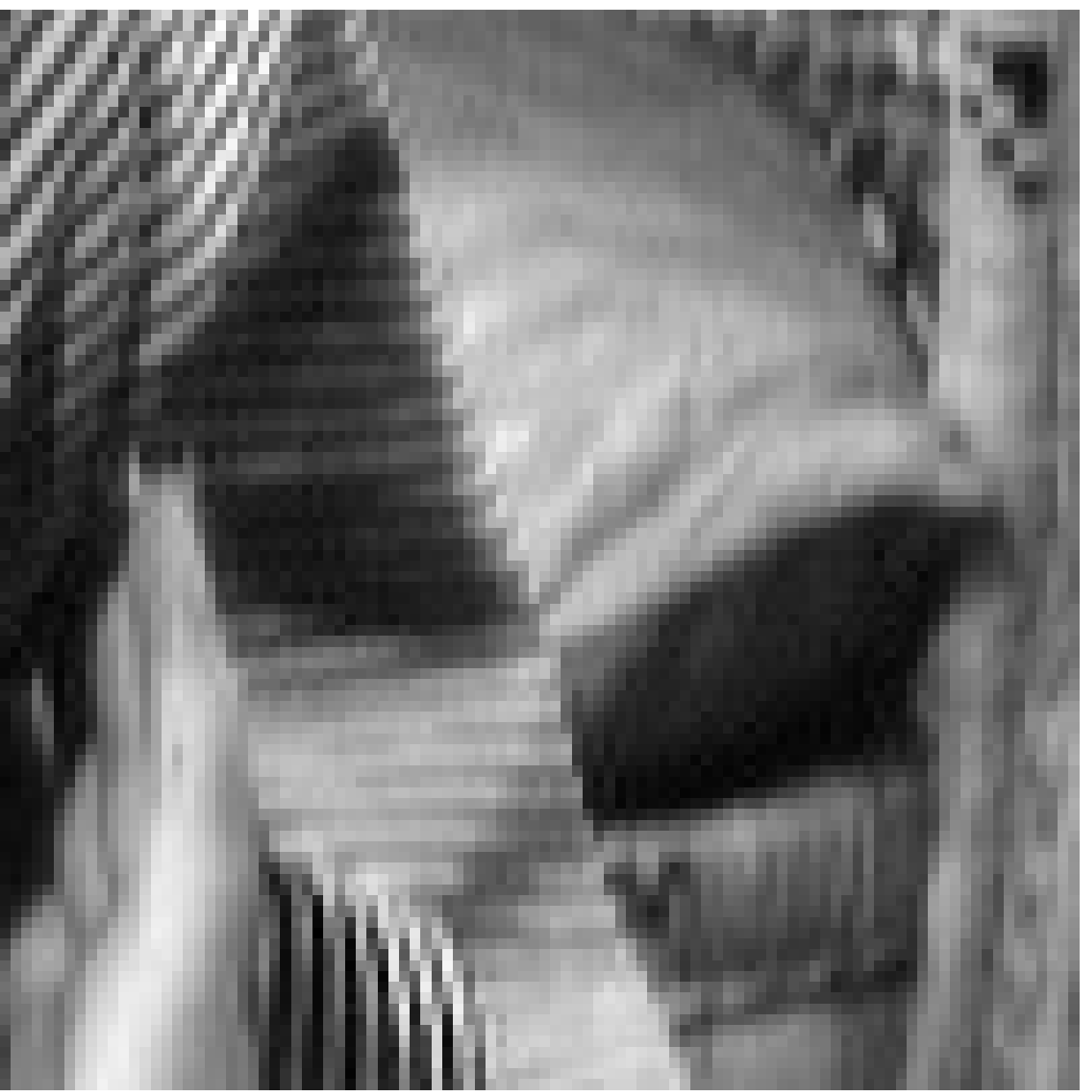}}\subfigure[\hpnlb{} (\textbf{26.20 dB})]
{\includegraphics[width=0.17\linewidth]{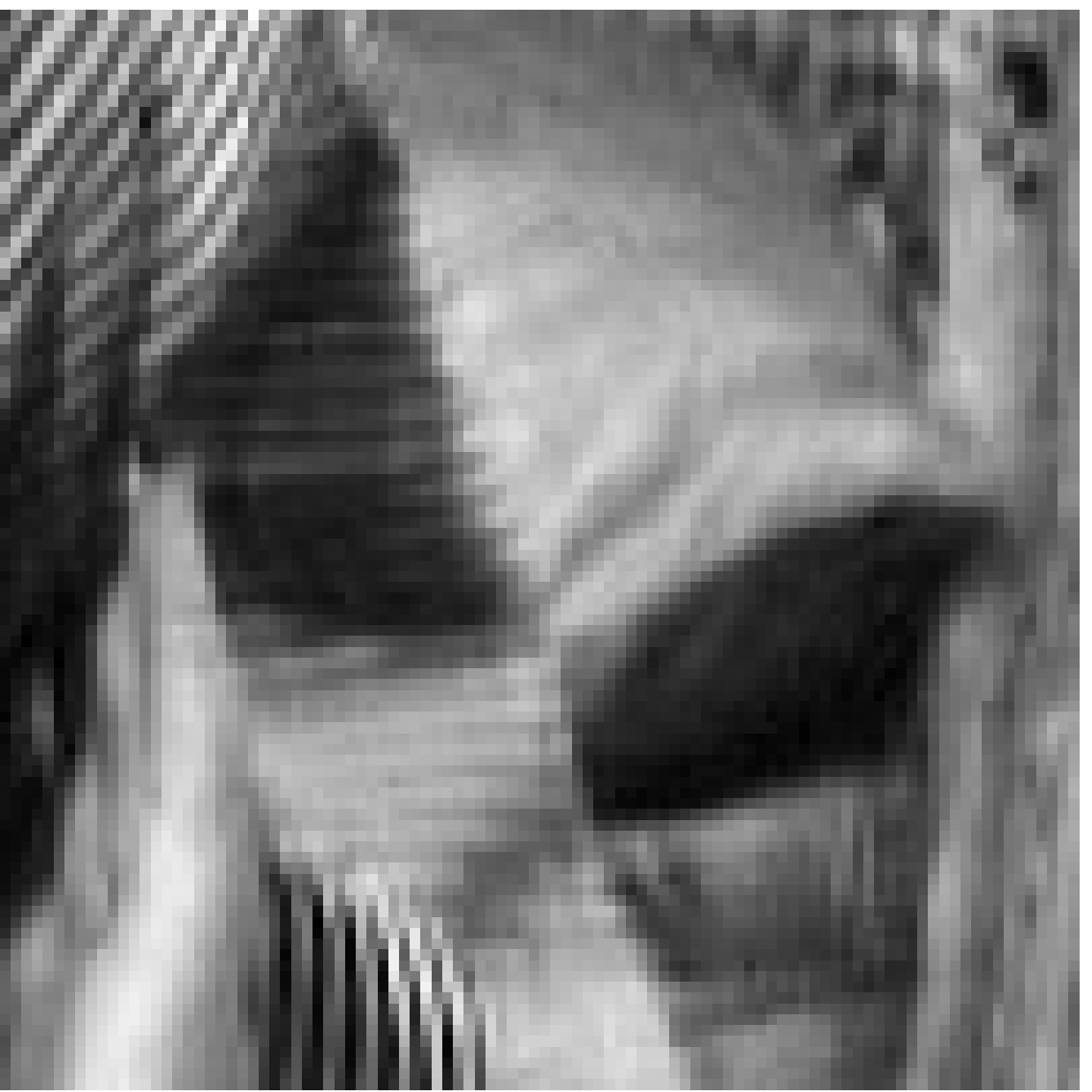}}\subfigure[PLE (24.76 dB)]{\includegraphics[width=0.17\linewidth]{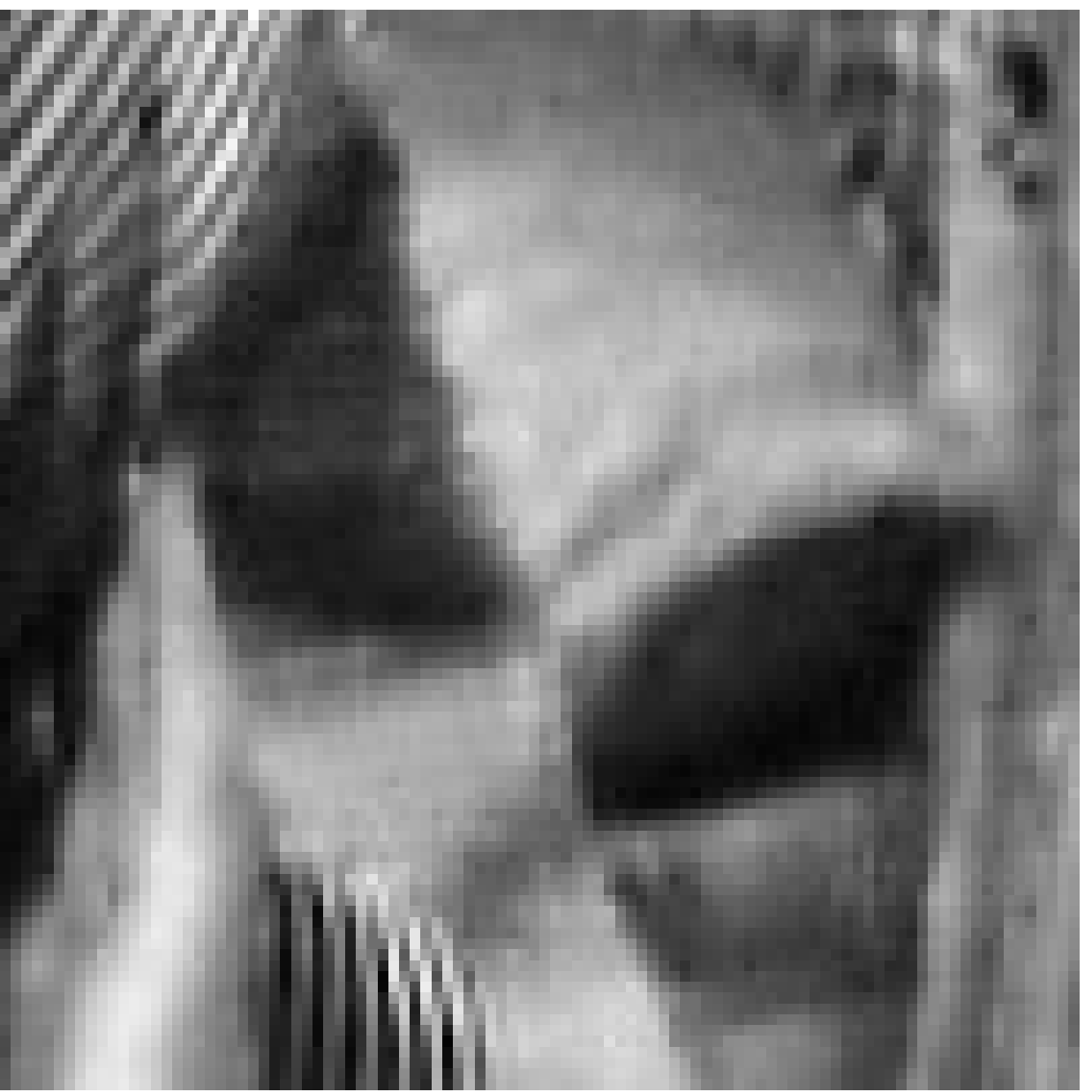}}\subfigure[EPLL (23.84 dB)]{\includegraphics[width=0.17\linewidth]{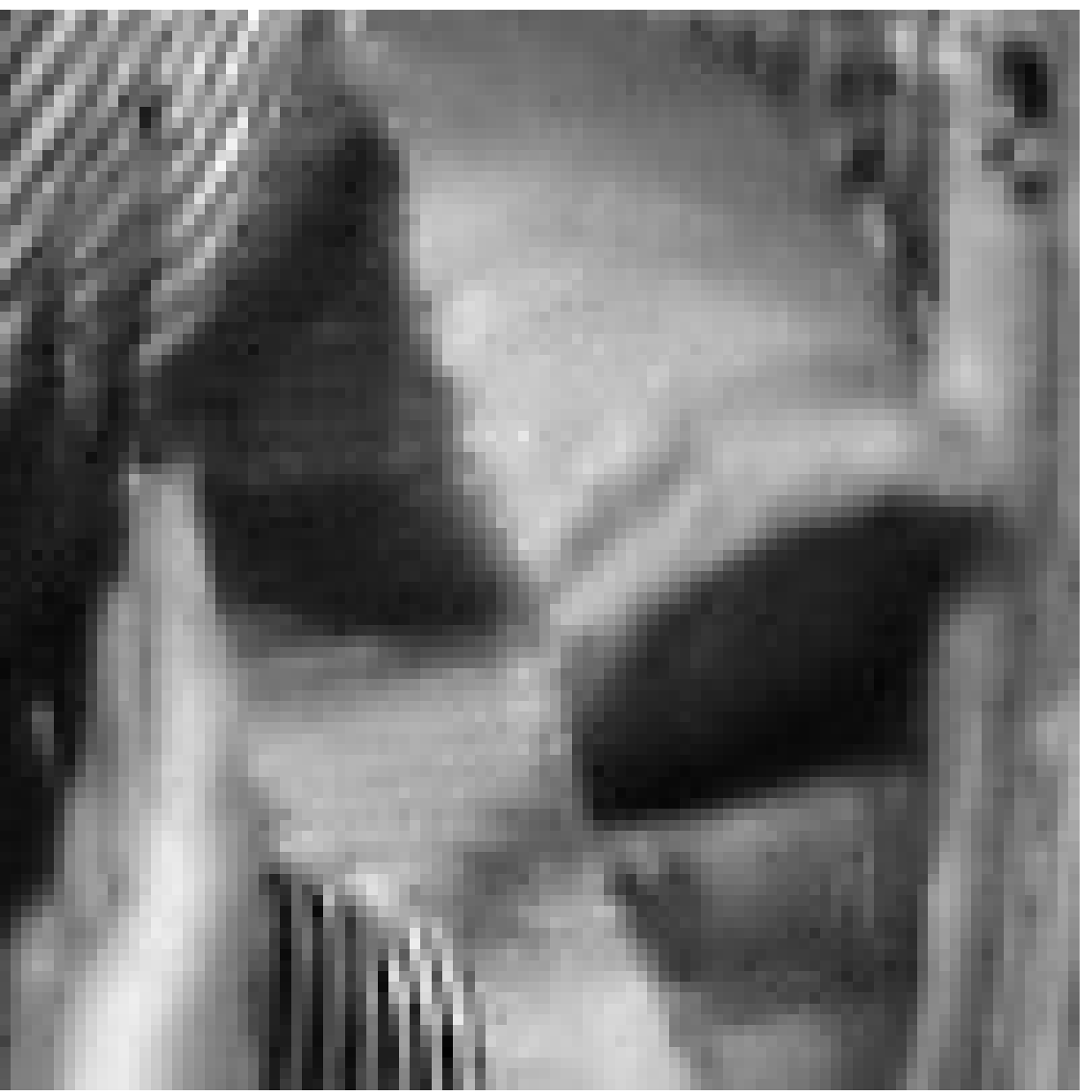}}\subfigure[E-PLE (23.60 dB)]{\includegraphics[width=0.17\linewidth]{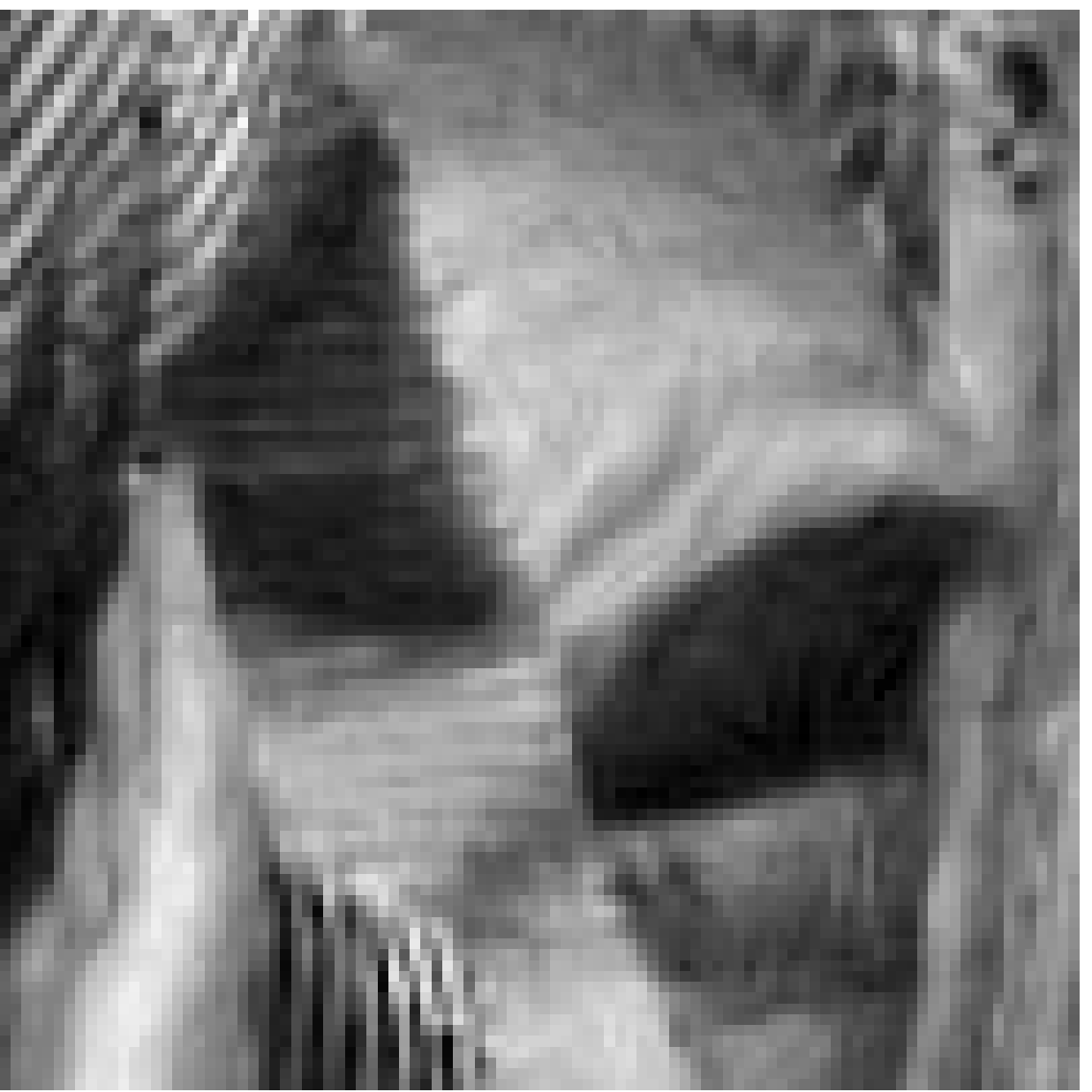}}
\includegraphics[width=0.17\linewidth]{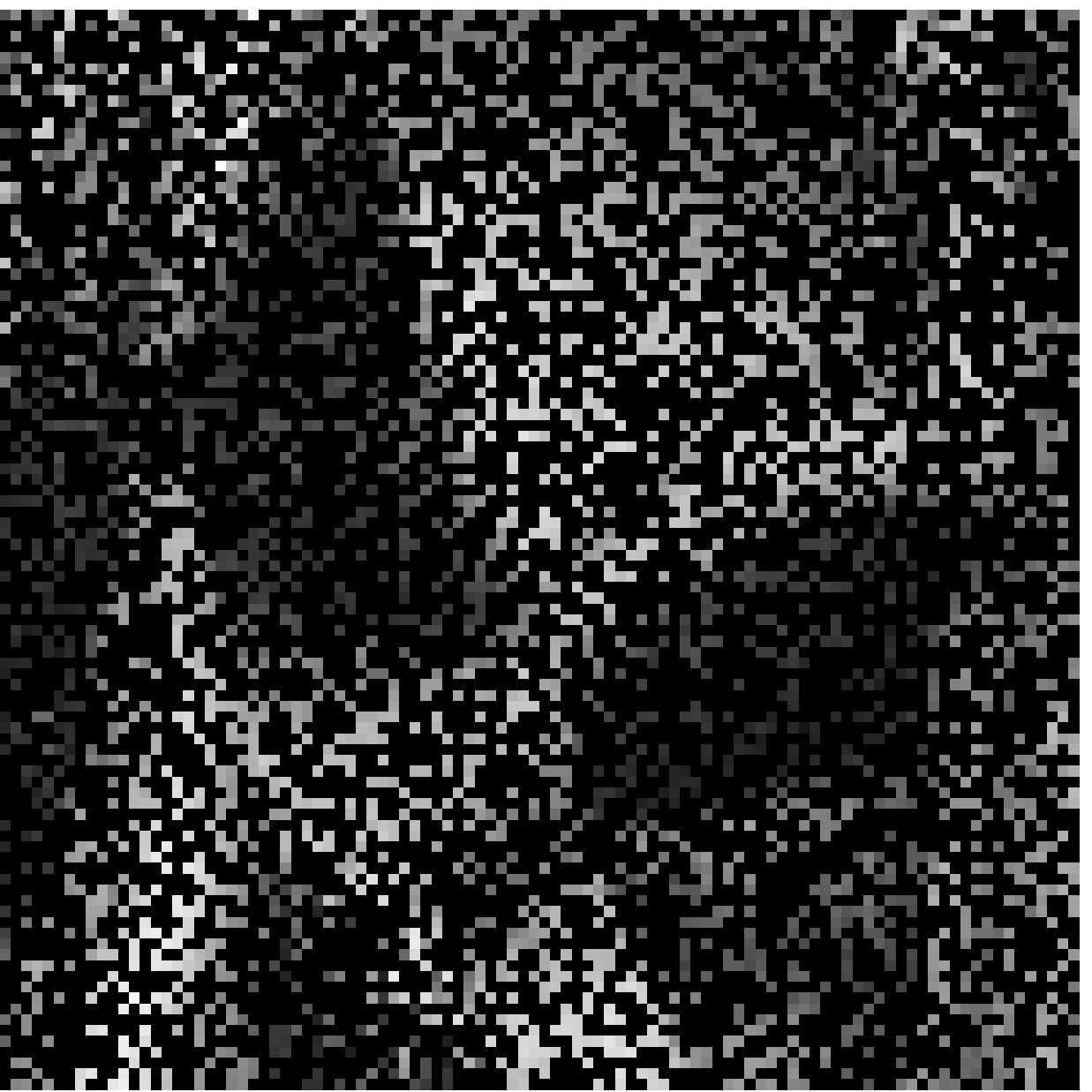}\includegraphics[width=0.17\linewidth]{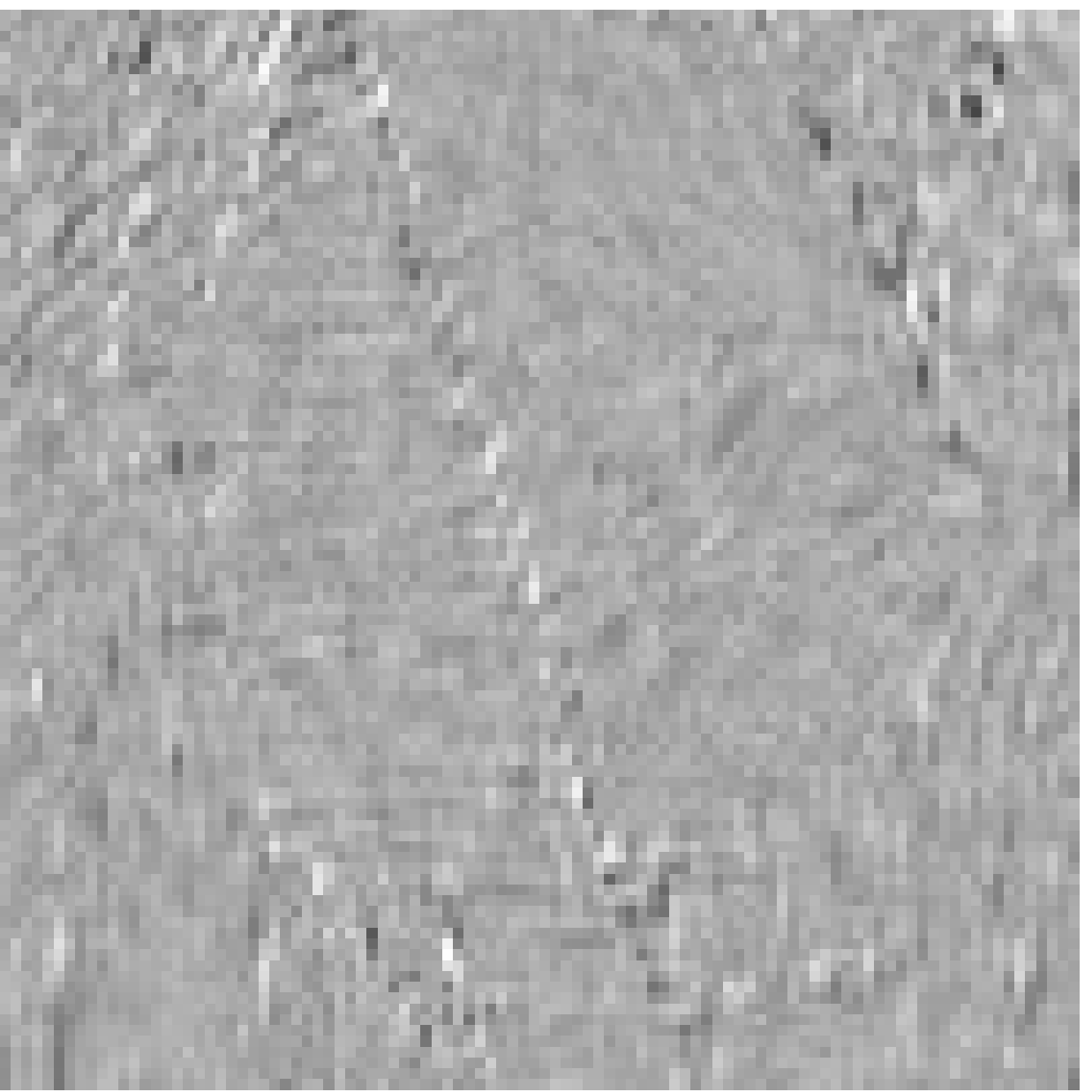}\includegraphics[width=0.17\linewidth]{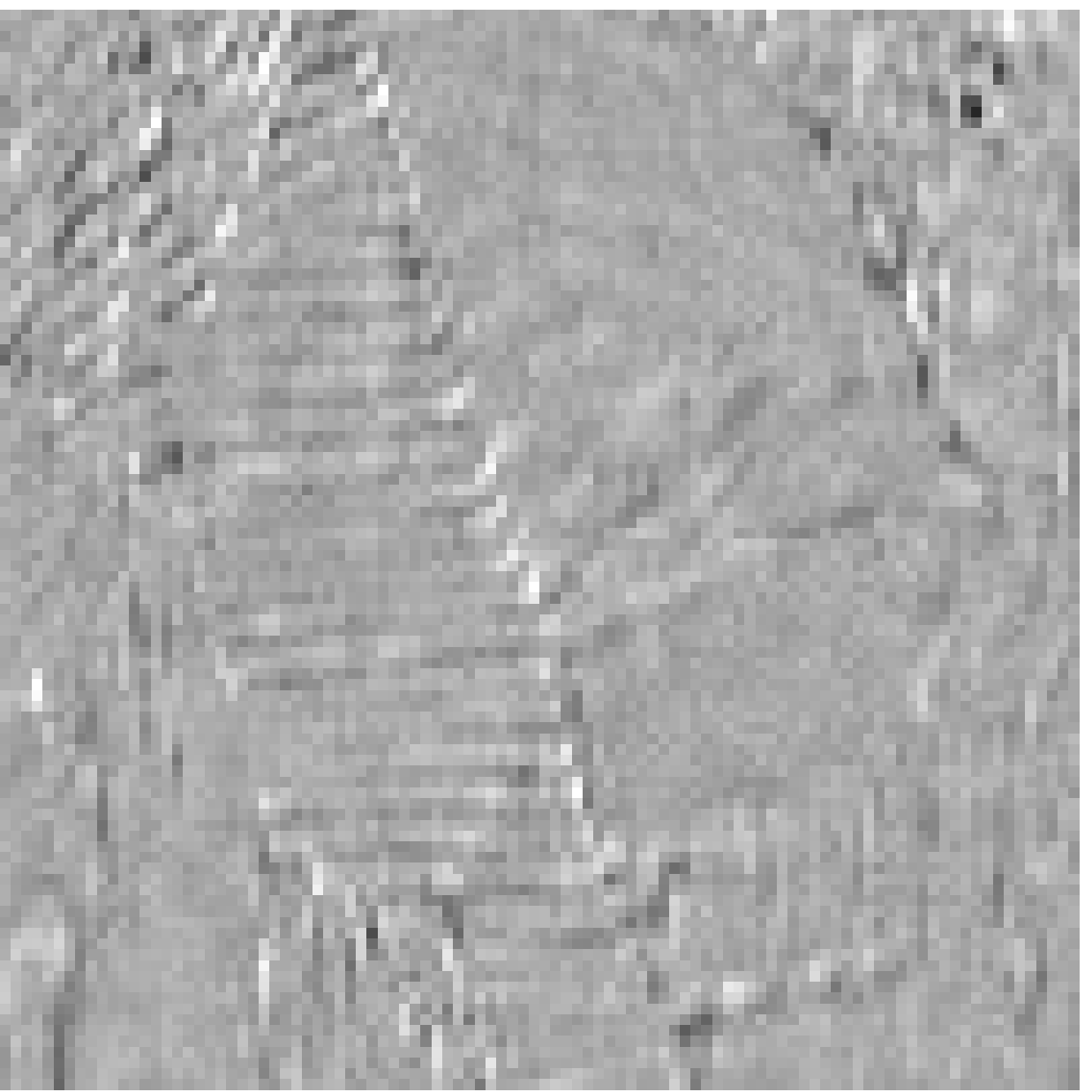}\includegraphics[width=0.17\linewidth]{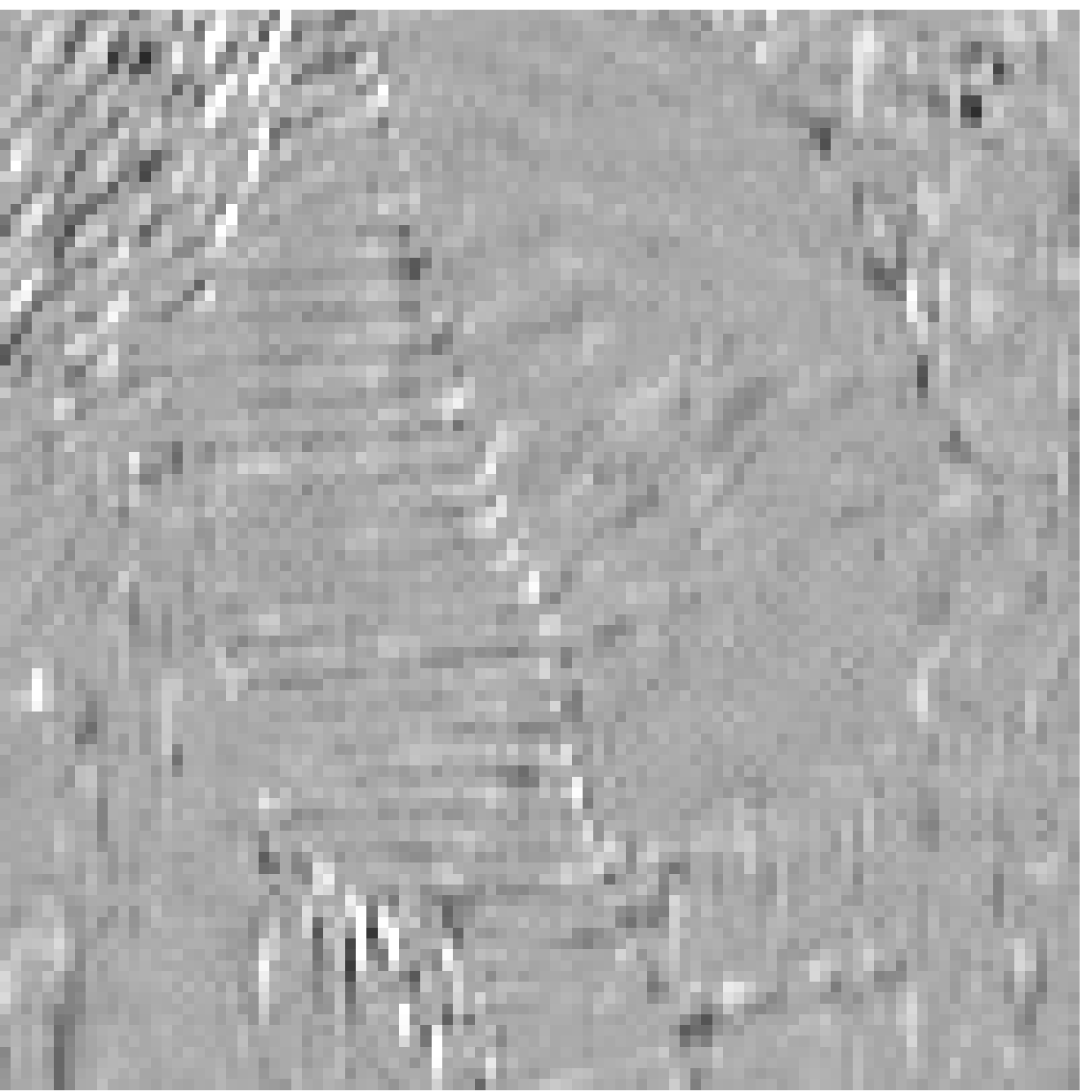}\includegraphics[width=0.17\linewidth]{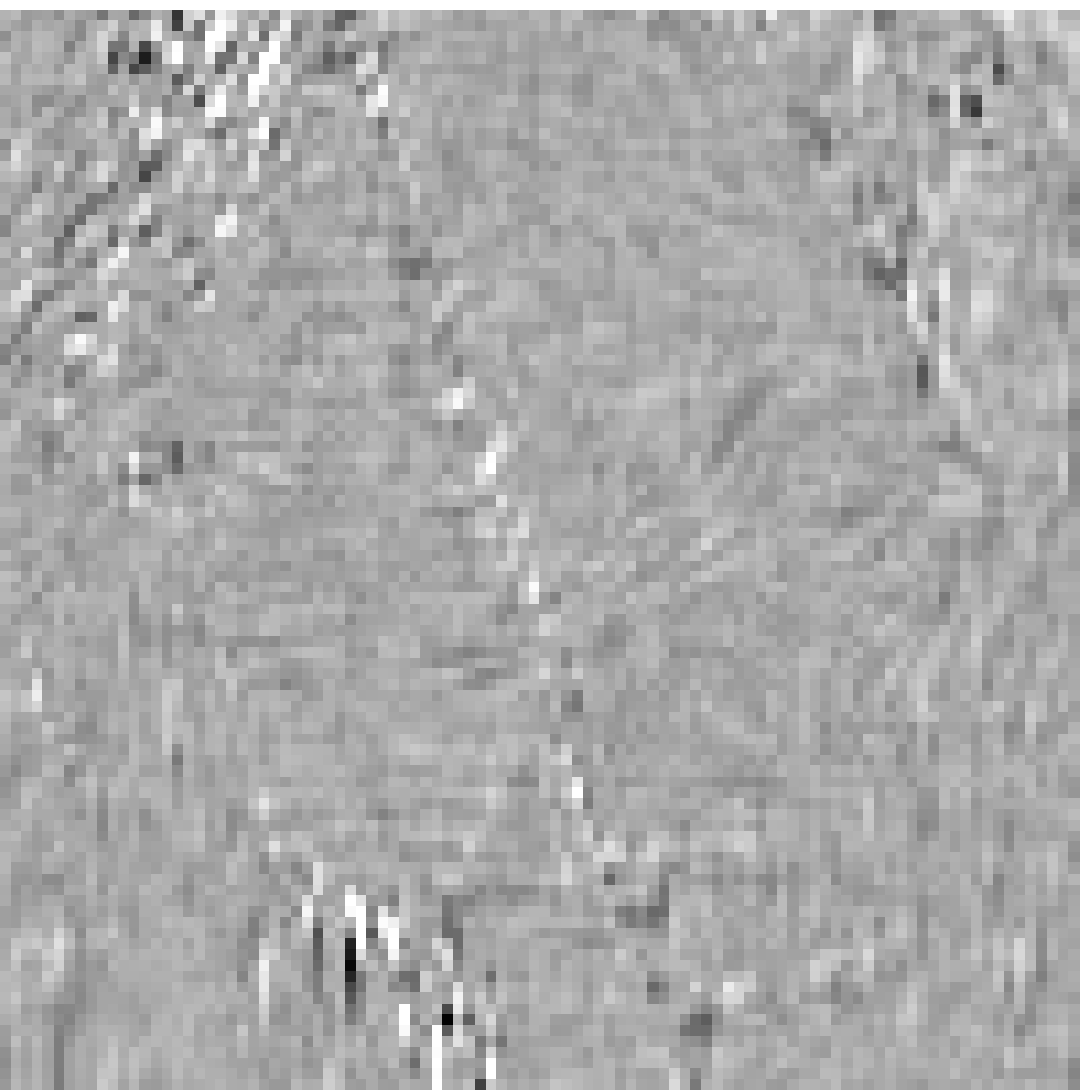}\\

\subfigure[Ground-truth]{\includegraphics[width=0.17\linewidth]{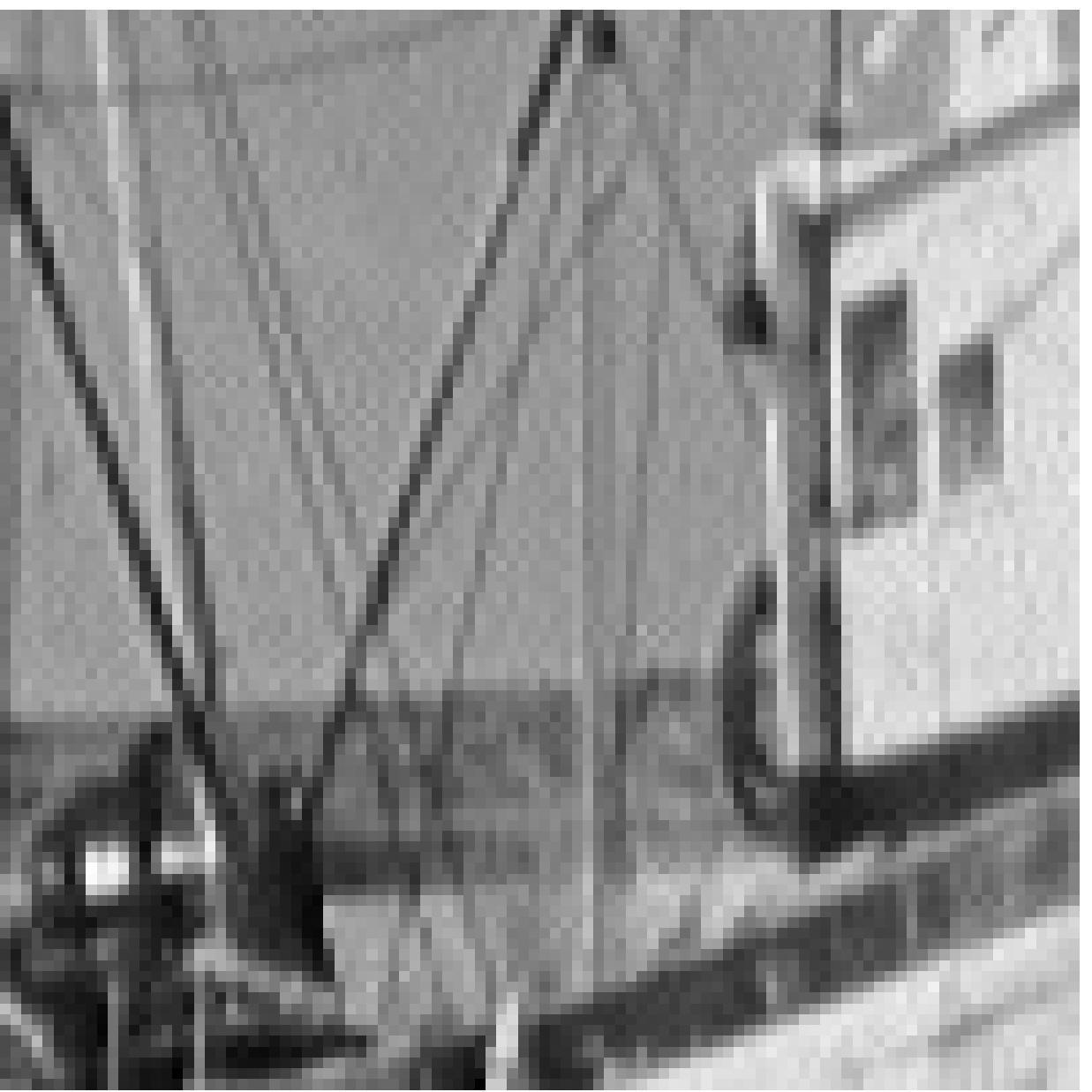}}\subfigure[\hpnlb{} (\textbf{28.34 dB})]{\includegraphics[width=0.17\linewidth]{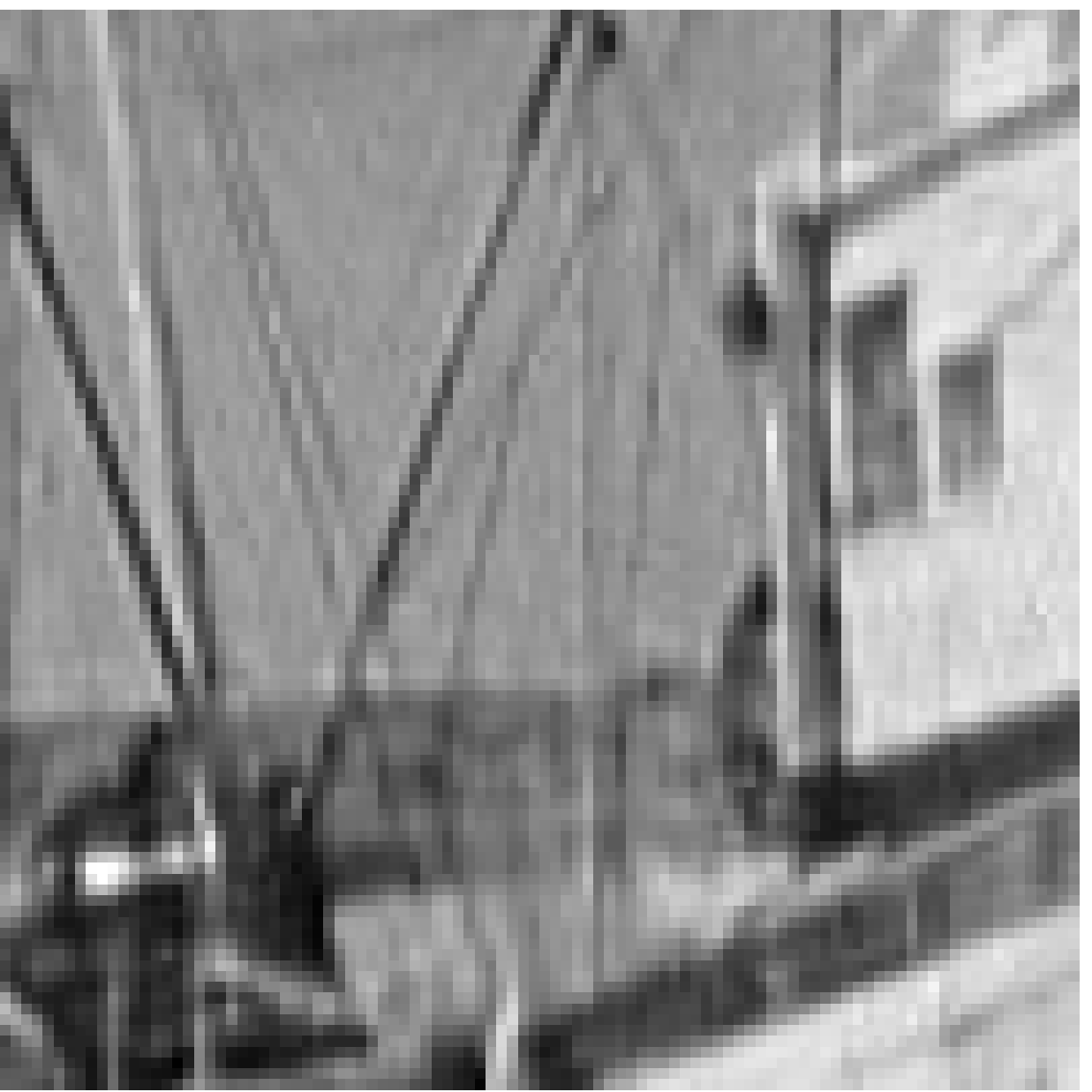}}\subfigure[PLE (27.50 dB)]{\includegraphics[width=0.17\linewidth]{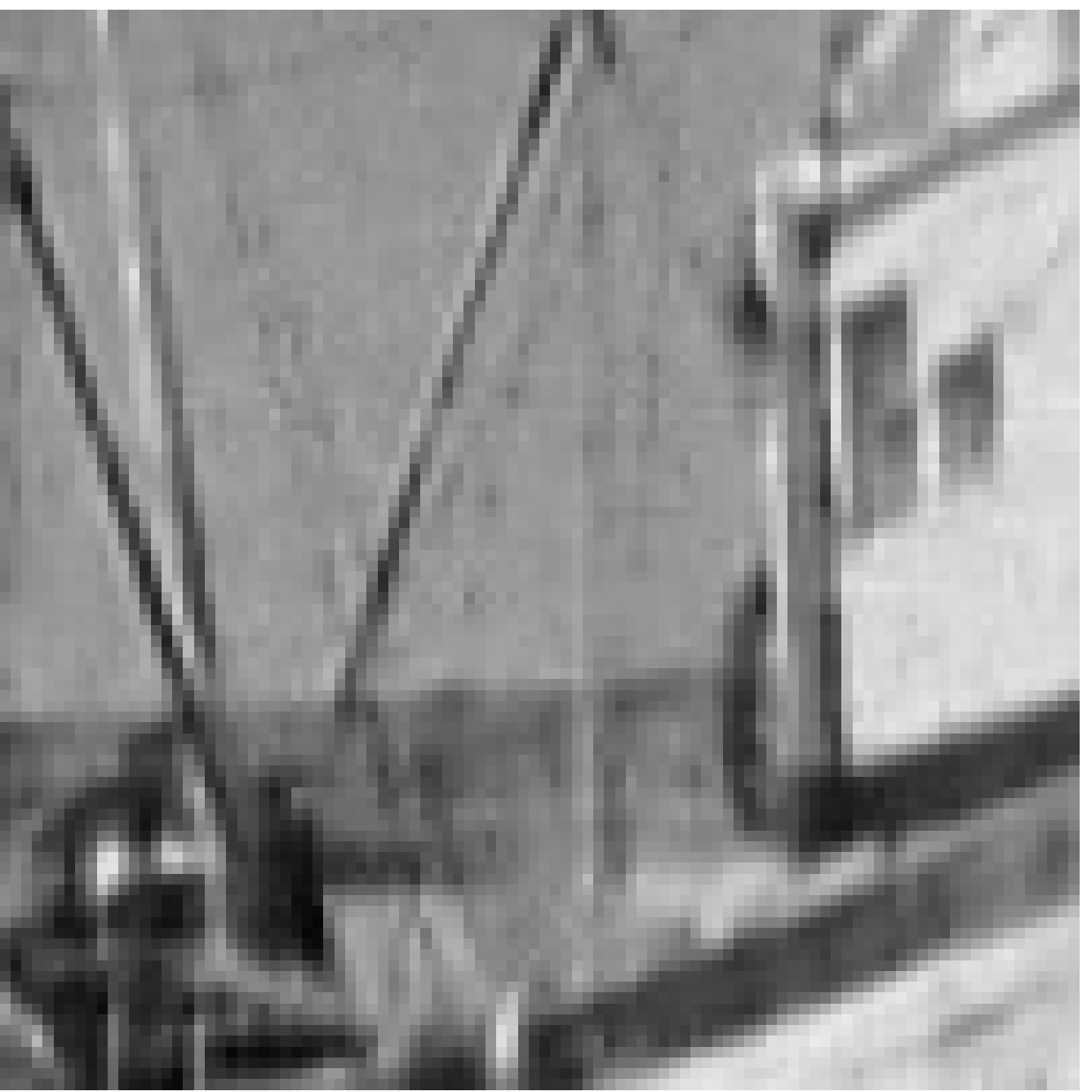}}\subfigure[EPLL (27.27 dB)]{\includegraphics[width=0.17\linewidth]{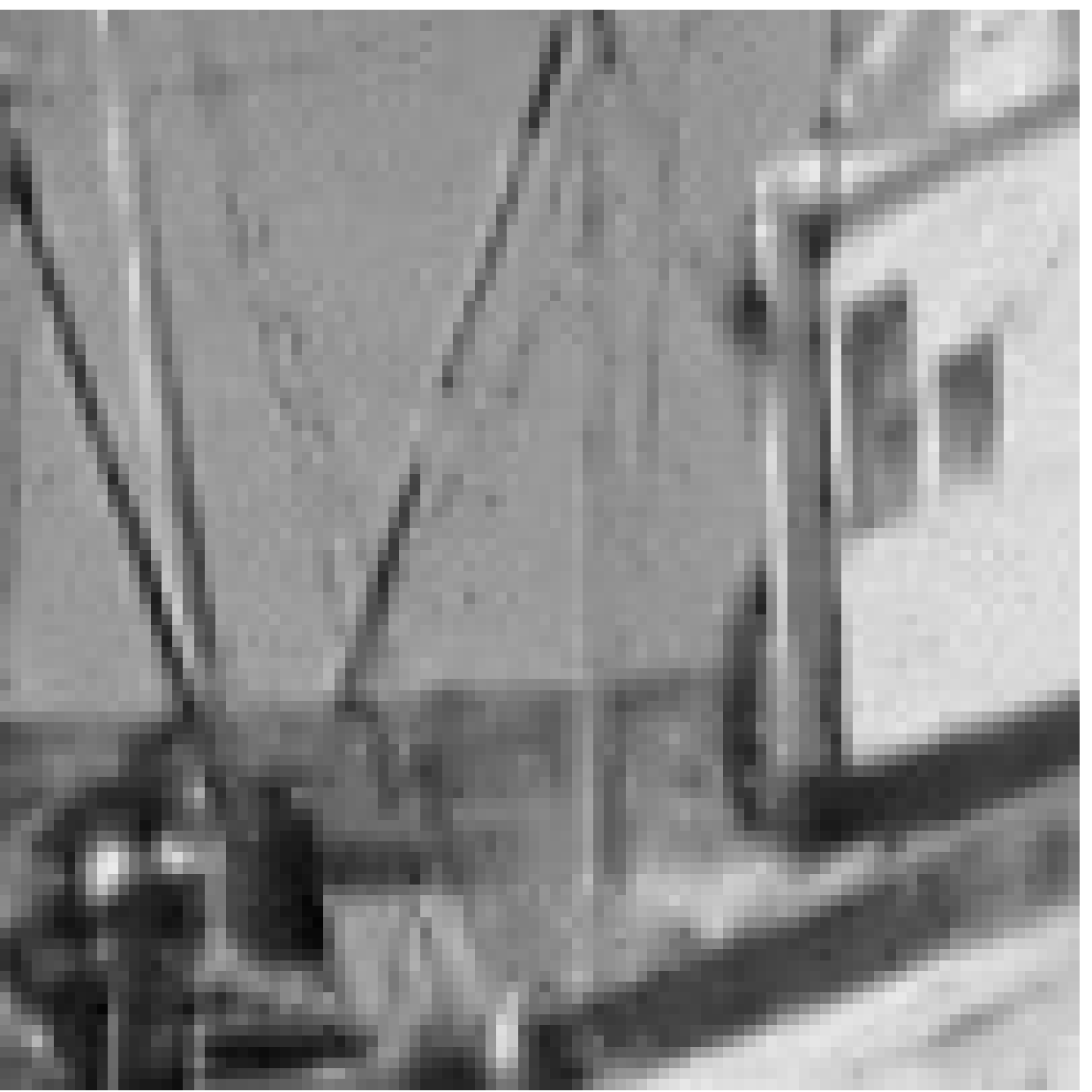}}\subfigure[E-PLE (26.83 dB)]{\includegraphics[width=0.17\linewidth]{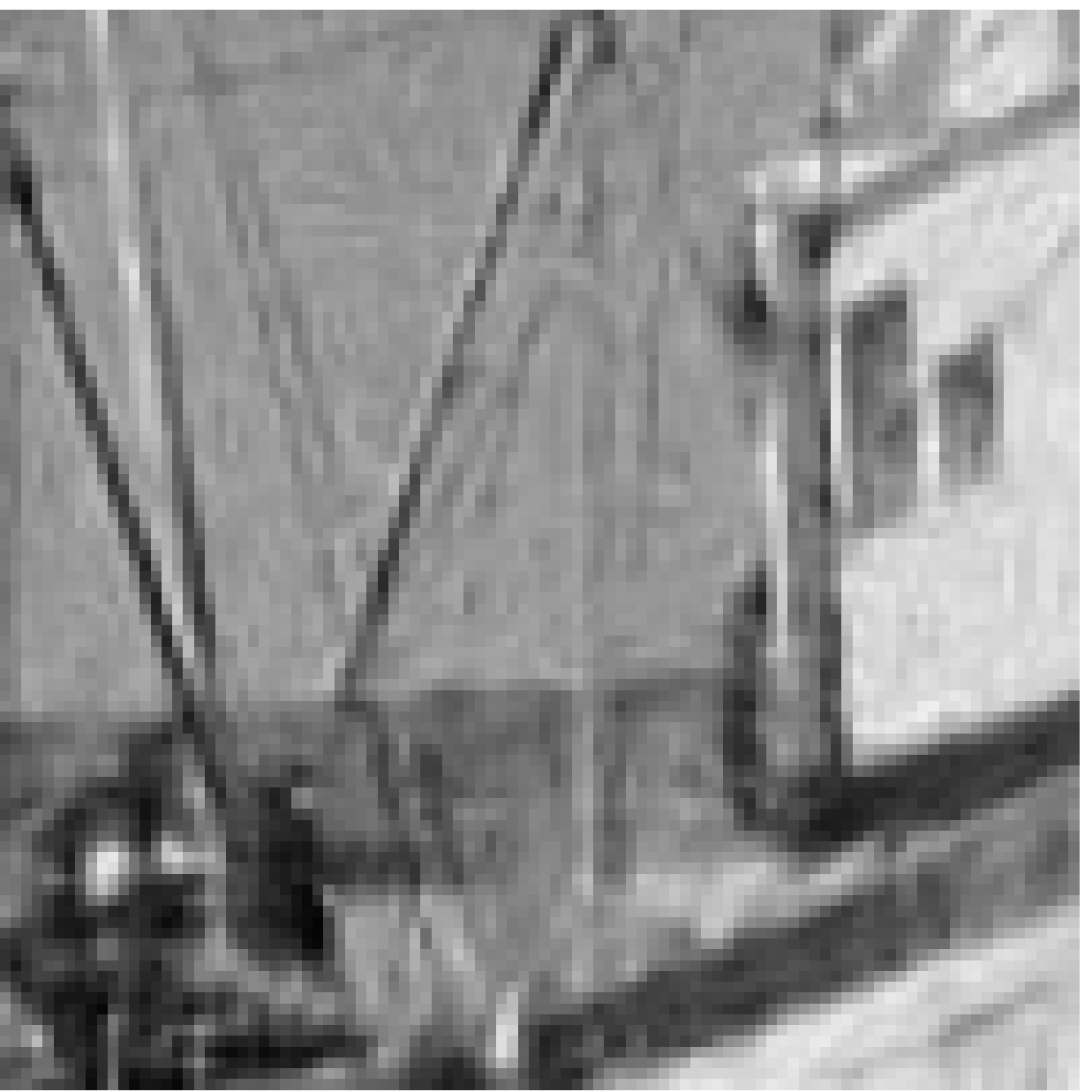}}
\includegraphics[width=0.17\linewidth]{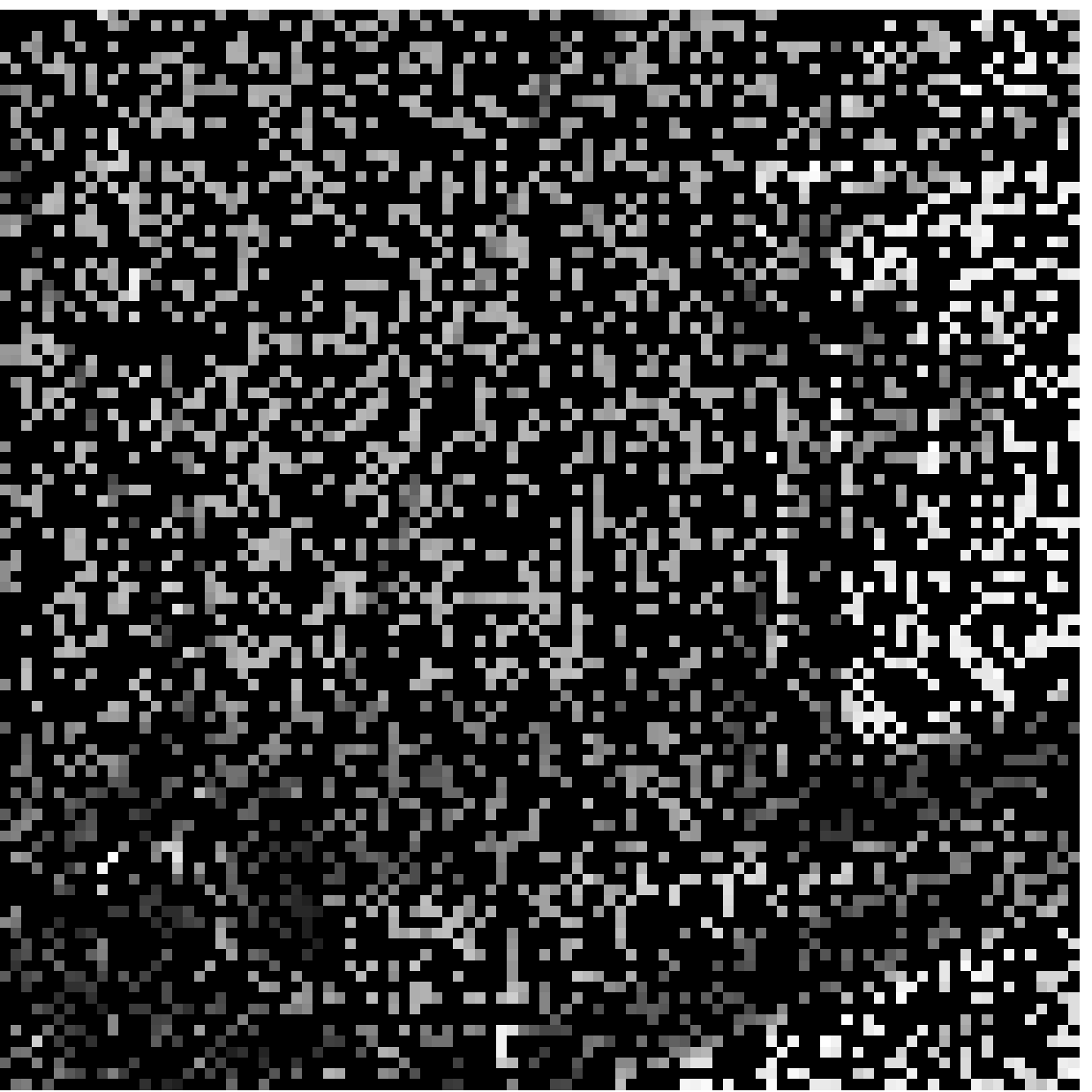}\includegraphics[width=0.17\linewidth]{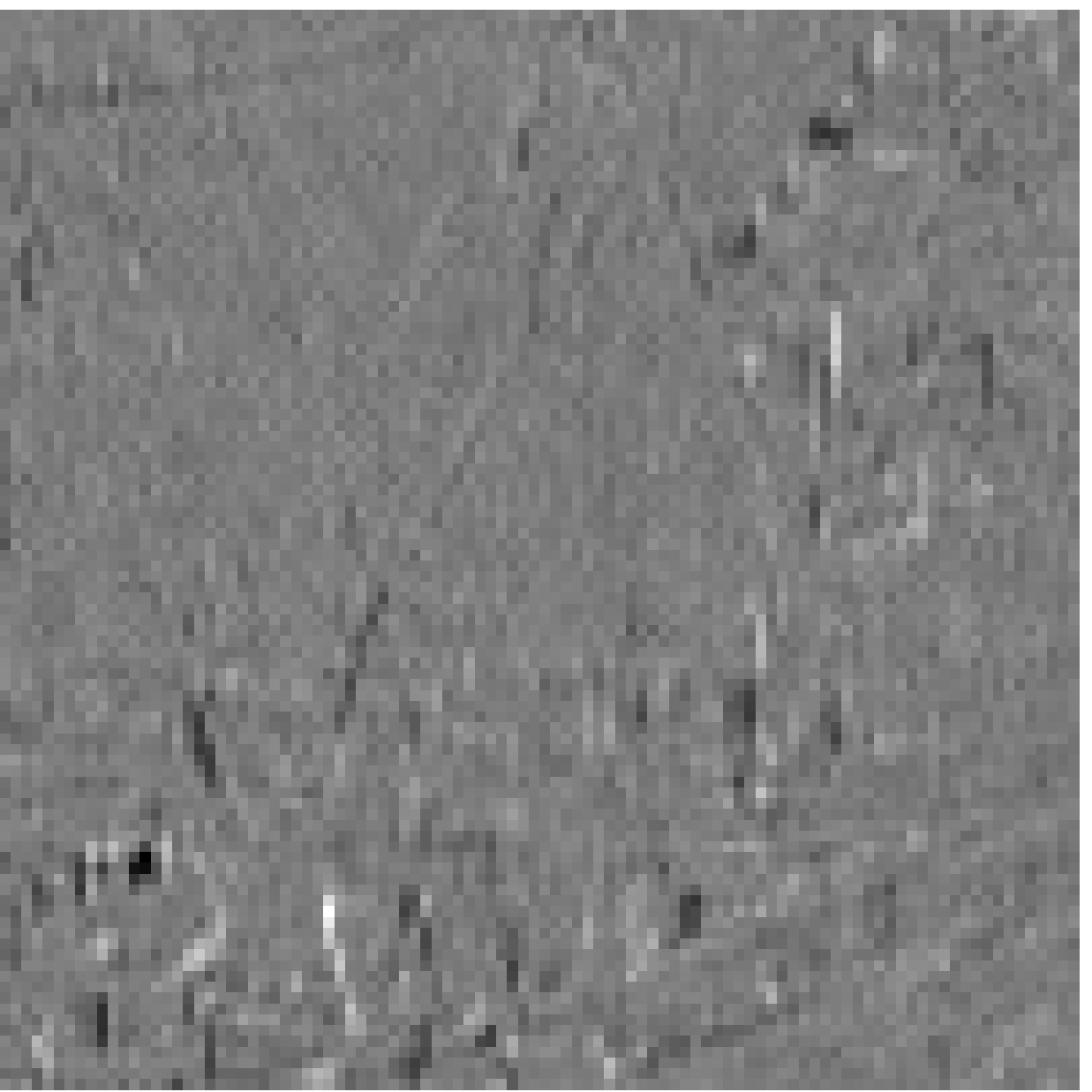}\includegraphics[width=0.17\linewidth]{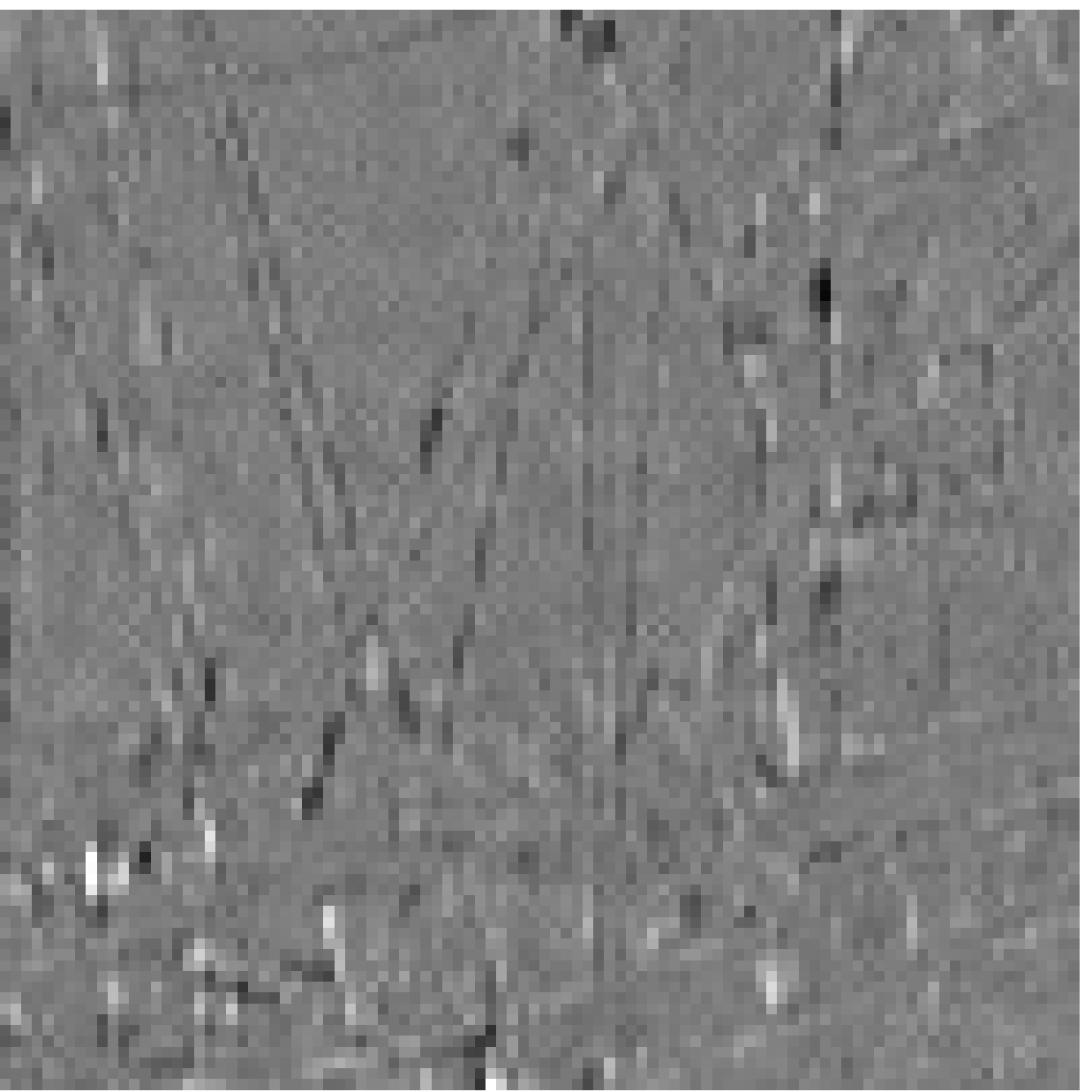}\includegraphics[width=0.17\linewidth]{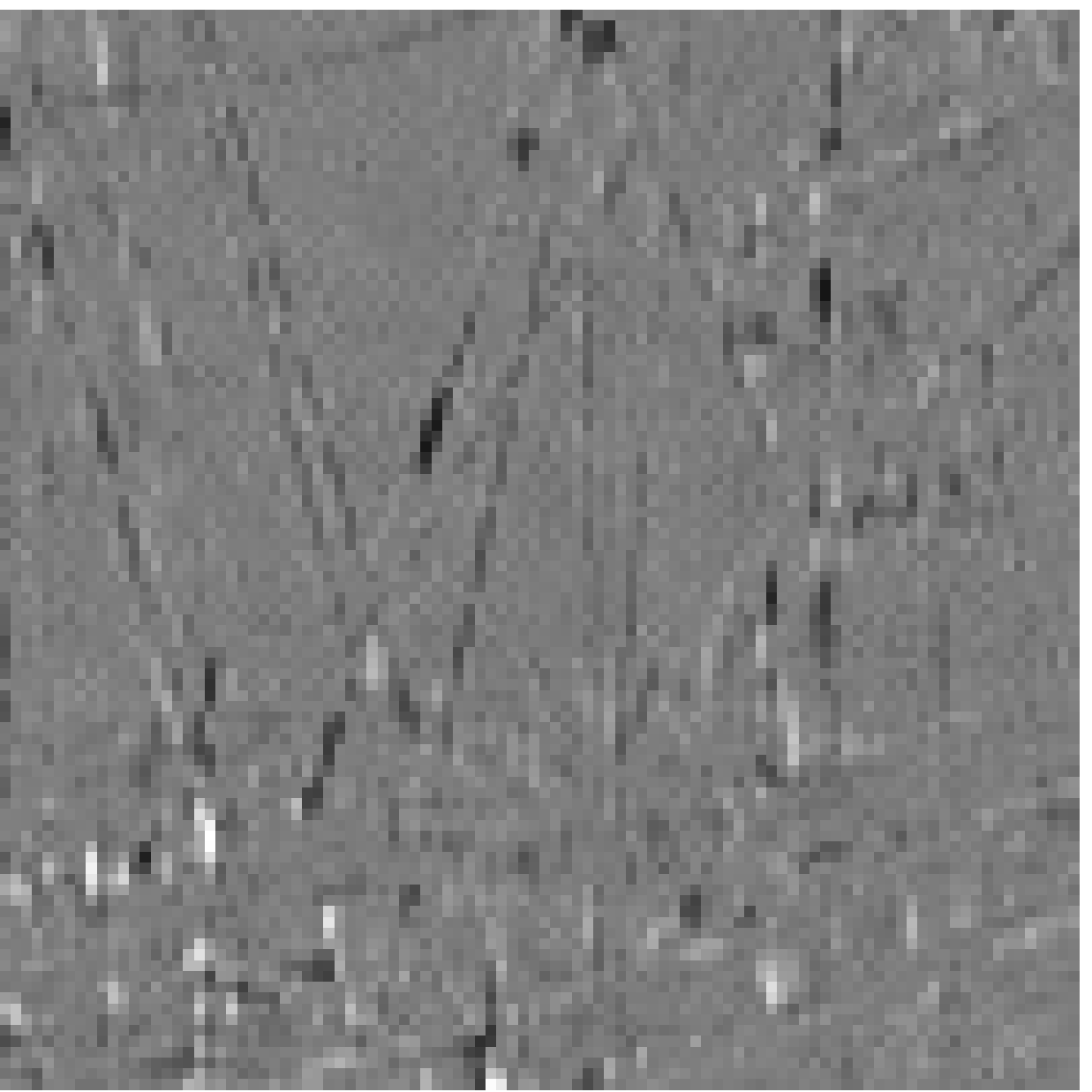}\includegraphics[width=0.17\linewidth]{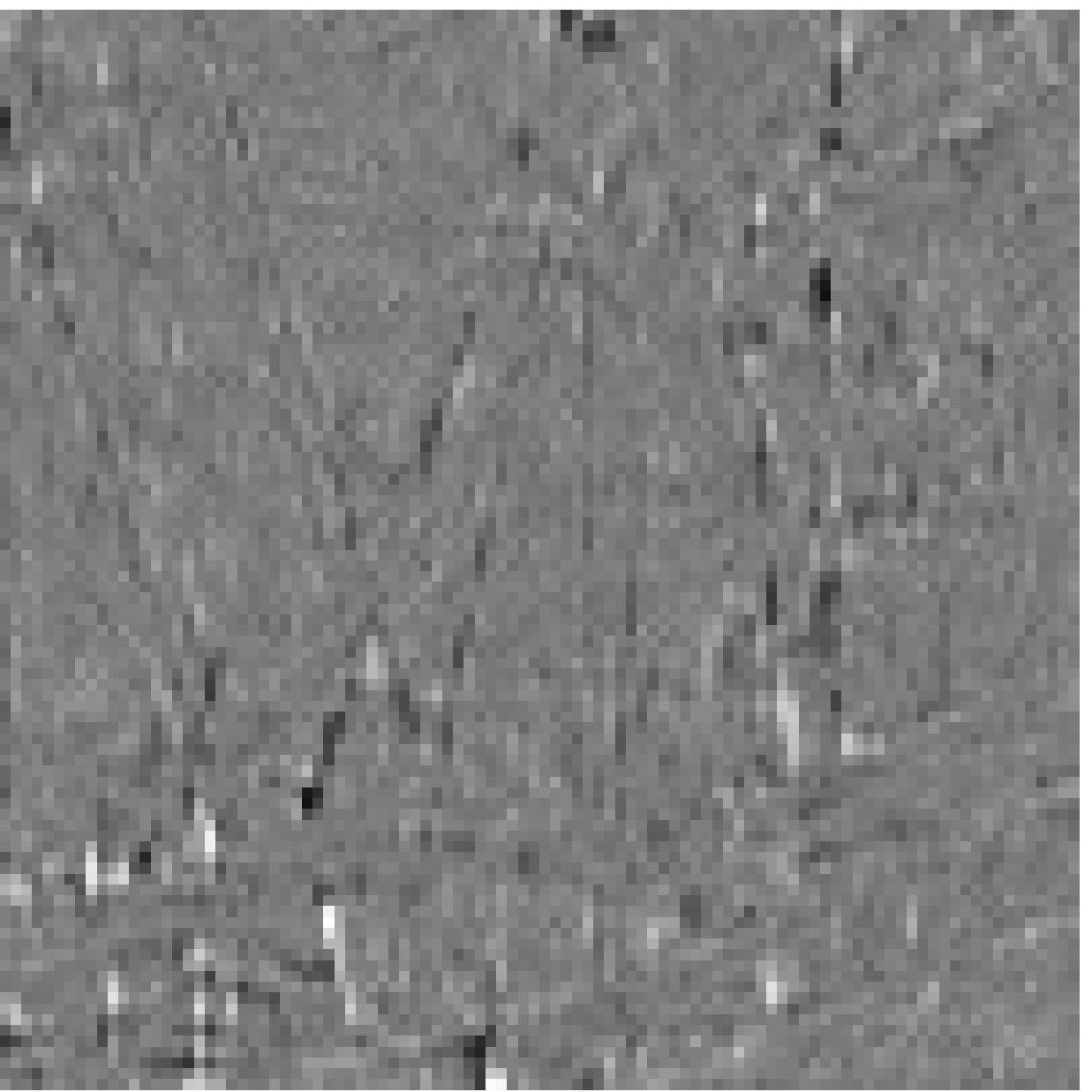}\\

\caption{\textbf{Synthetic data. Combined interpolation and denoising with 70\% of randomly missing pixels and additive Gaussian noise ($\sigma^2 = 10$)}. \textbf{Left to right:} (first row) Ground-truth (extract of barbara), result by \hpnlb{}, PLE, EPLL, E-PLE. (second row) input image, difference with respect to the ground-truth of each of the corresponding results. (third and fourth row) Idem for an extract of the boat image. See Table~\ref{tab:psnrInterp} for the \psnr{} results for the complete images. Please see the digital copy for better details reproduction.}
\label{fig:syntheticExpsInterpDeno1}
\end{figure*}

\paragraph{Denoising} For the denoising task the proposed approach should perform very similarly to the state-of-the-art denoising algorithm NLB~\cite{lebrun13}. The following experiments are conducted in order to verify this. 

The ground-truth images are corrupted with additive Gaussian noise with variance $\sigma^2=10, 30, 50,80$. The code provided by the authors~\cite{lebrun13IPOL} automatically sets the NLB parameters from the input $\sigma^2$ and the patch size, in this case $8 \times 8$. For this experiment, there are no unknown pixels to interpolate (the mask $\U$ is the identity matrix). 

The results of both methods are very similar if \hpnlb{} is initialized with the output of the first step of NLB~\cite{lebrun13} (instead of using the initialization described in Section~\ref{ssec:pasosAlgo}) and the parameters $\kappa$ and $\nu$ are large enough. In this case, $\mean_0$ and $\S_0$ are prioritized in equations~\eqref{eq:muhat} and~\eqref{eq:Lhat} and both algorithms are almost the same. That is what we observe in practice with $\alpha_H = \alpha_L = 100$, as demonstrated in the results summarized in Table~\ref{tab:psnrInterp}. The denoising performance of \hpnlb{} is degraded for small $\kappa$ and $\nu$ values. This is due to the fact that $\mean_0$ and $\S_0$, as well as $\mean$ and $\S$ in NLB, are computed from an oracle image resulting from the first restoration step. This restoration includes not only the denoising of each patch, but also an aggregation step that greatly improves the final result. Therefore, the contribution of the first term of~\eqref{eq:muhat} to the computation of $\hat{\mean}$ degrades the result compared to that of using $\mean_0$ only (i.e. using a large $\kappa$). 
\paragraph{Zooming} In order to evaluate the zooming capacity of the proposed approach, ground-truth images are downsampled by a factor 2 (no anti-aliasing filter is used) and the zooming is compared to the ground-truth. The results are compared with PLE, EPLL, E-PLE and Lanczos interpolation. Table~\ref{tab:psnrInterp} summarizes the obtained \psnr{} values.  Figure~\ref{fig:syntheticZooming1} shows extracts from the obtained results, the \psnr{} values for the extracts and the corresponding difference images with respect to the ground-truth. Again, \hpnlb{} yields a sharper reconstruction than the other methods. 
\begin{figure*}
\centering
\subfigure[Ground-truth]{\includegraphics[width=0.166\linewidth]{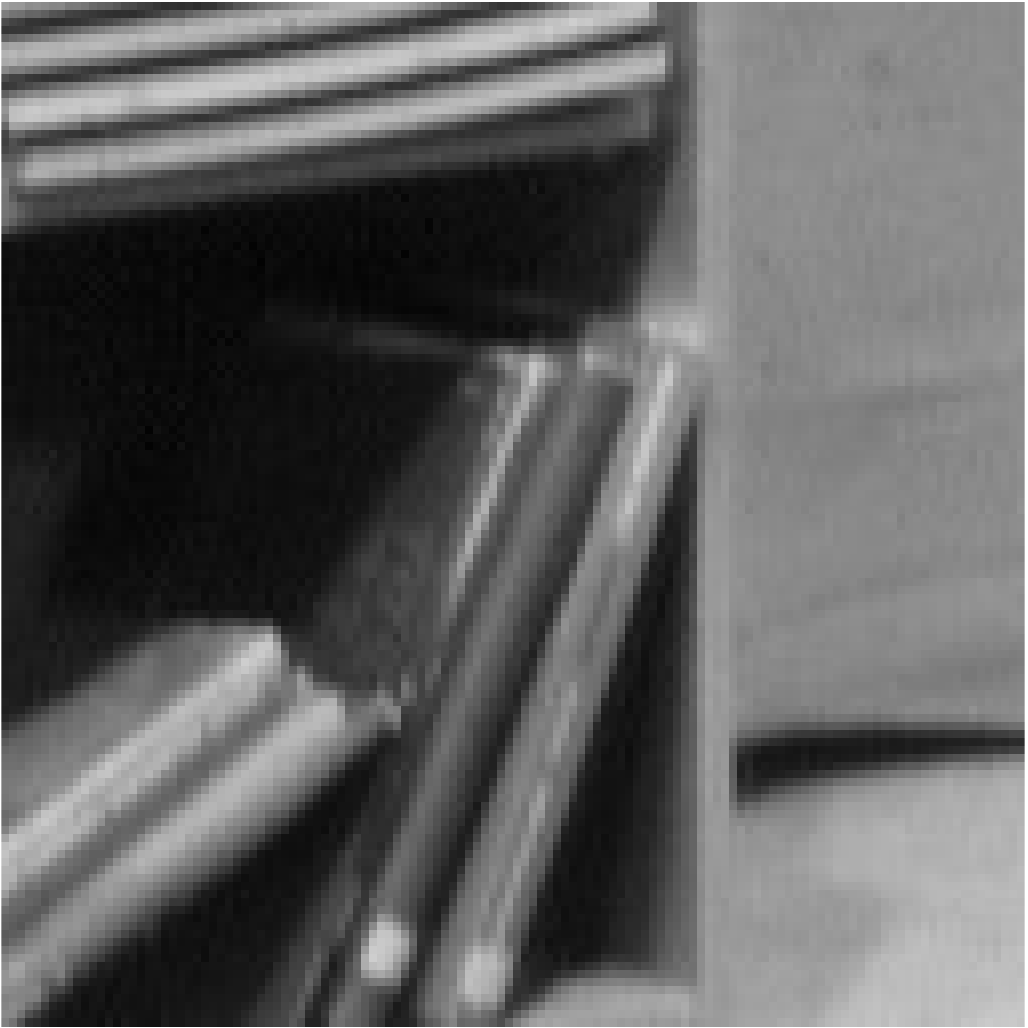}}\subfigure[\hpnlb{} (\textbf{38.17} dB)]{\includegraphics[width=0.166\linewidth]{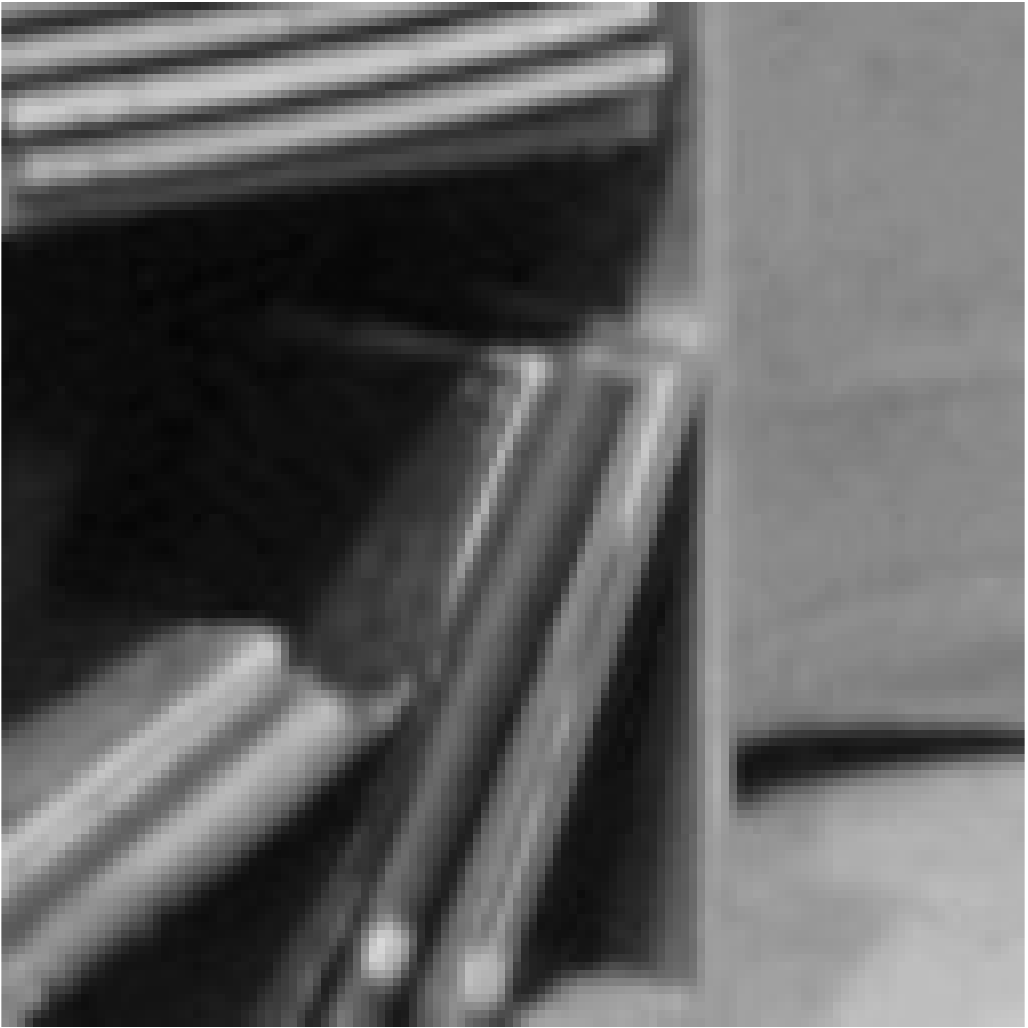}}\subfigure[PLE (37.11 dB)]{\includegraphics[width=0.166\linewidth]{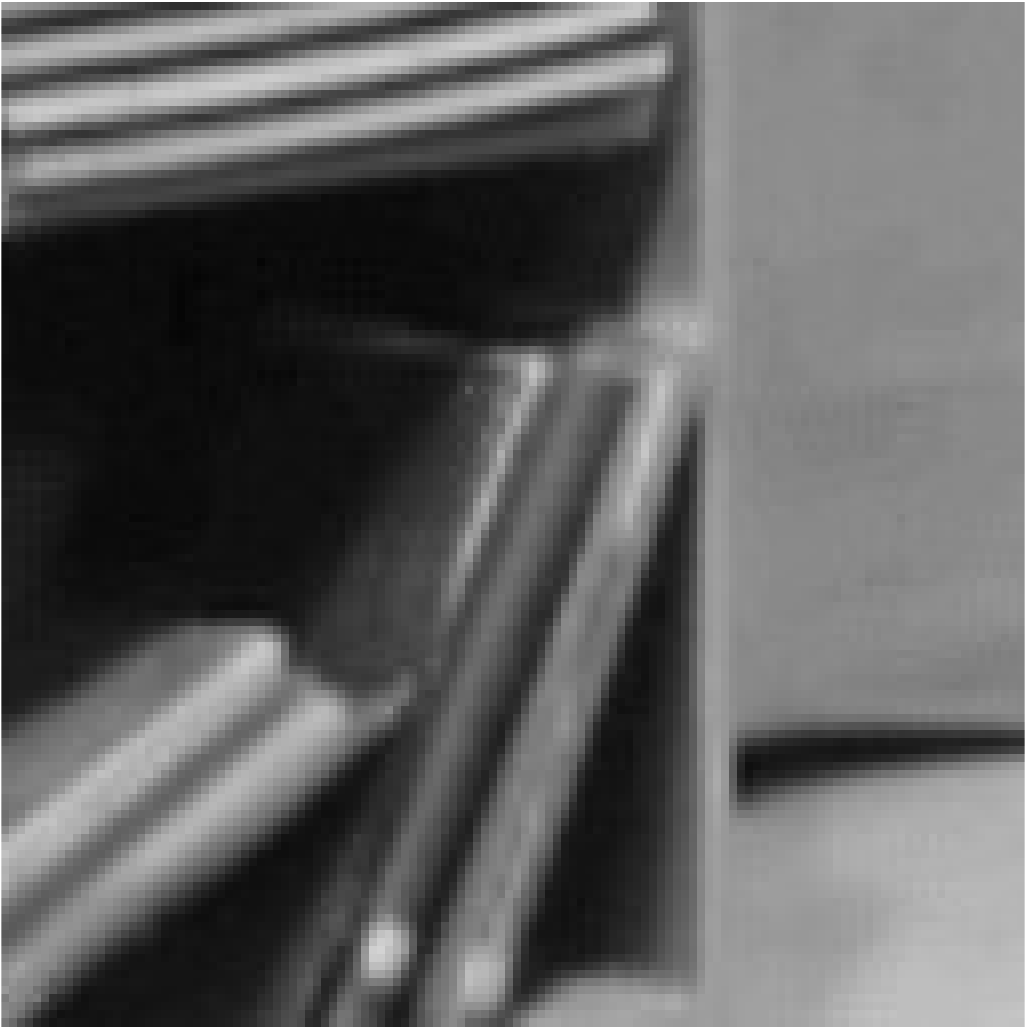}}\subfigure[EPLL (31.34 dB)]{\includegraphics[width=0.166\linewidth]{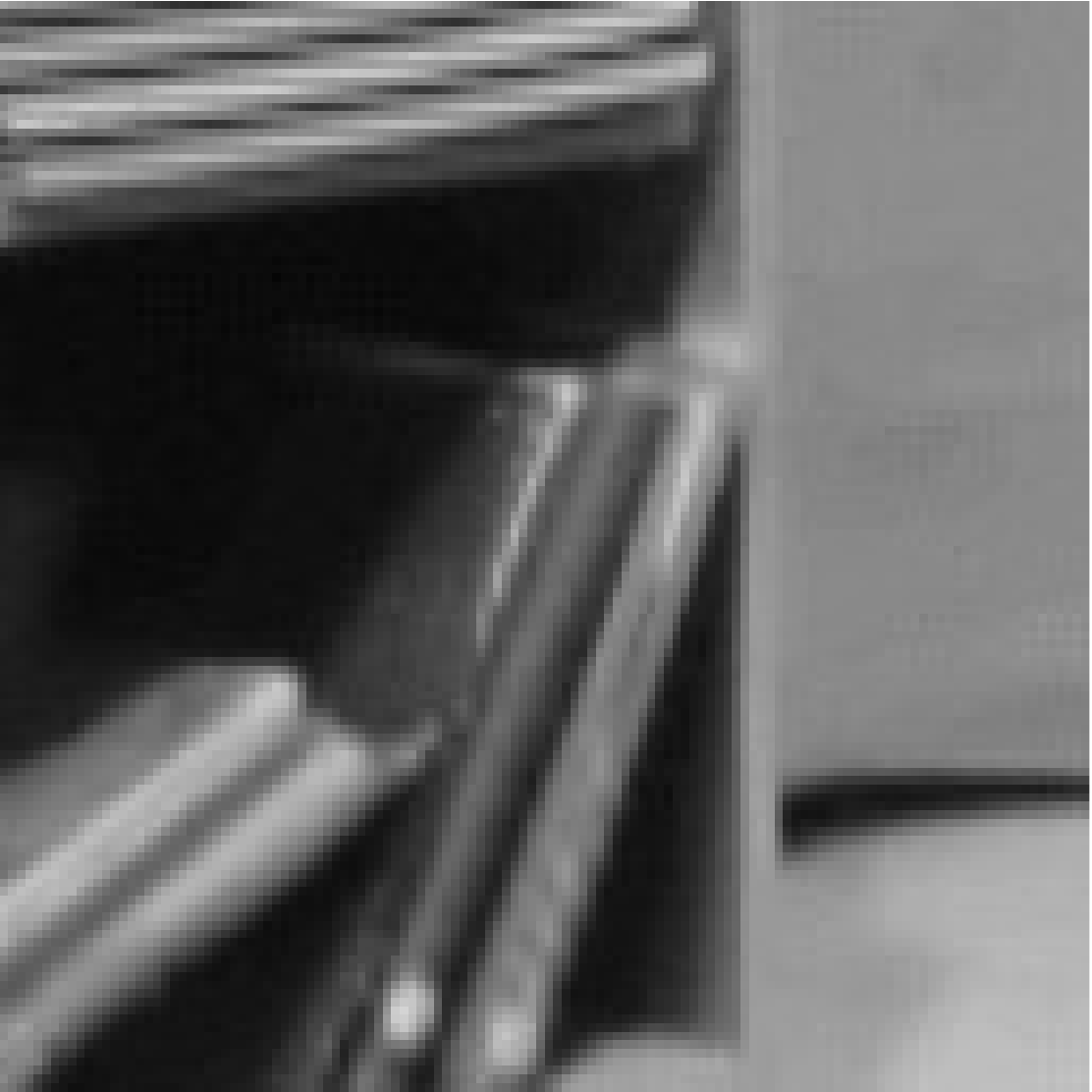}}\subfigure[E-PLE (36.51 dB)]{\includegraphics[width=0.166\linewidth]{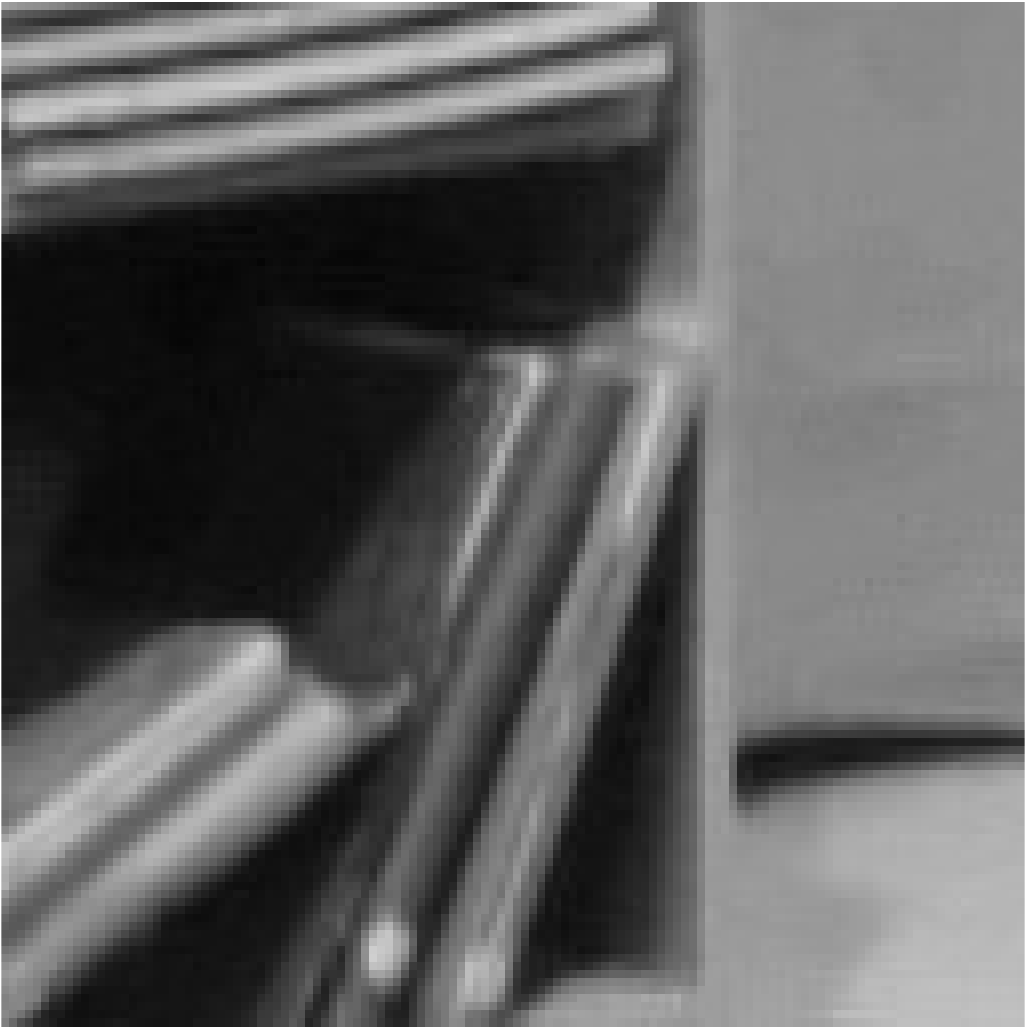}}\subfigure[Lanczos (28.01 dB)]{\includegraphics[width=0.166\linewidth]{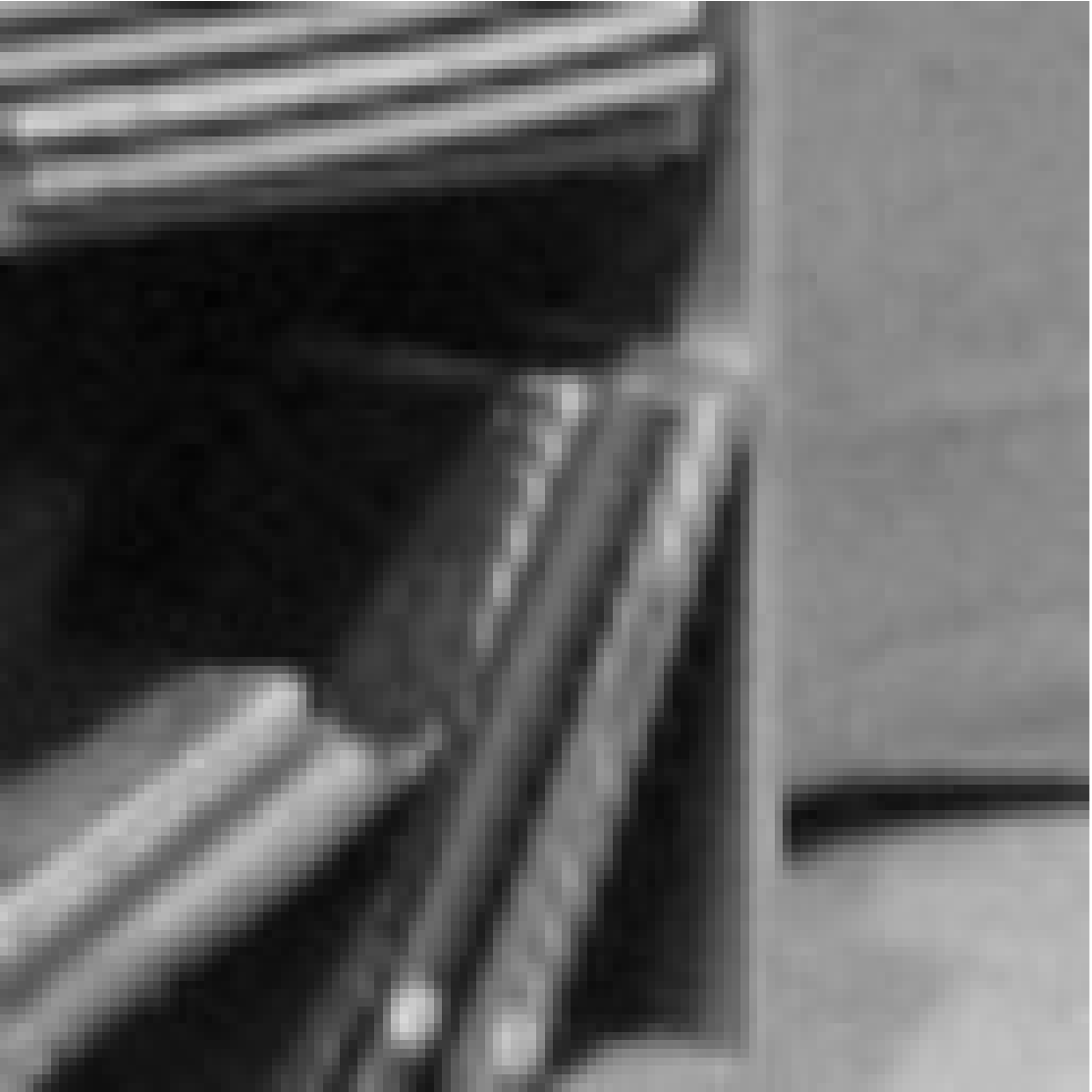}}

\includegraphics[width=0.166\linewidth]{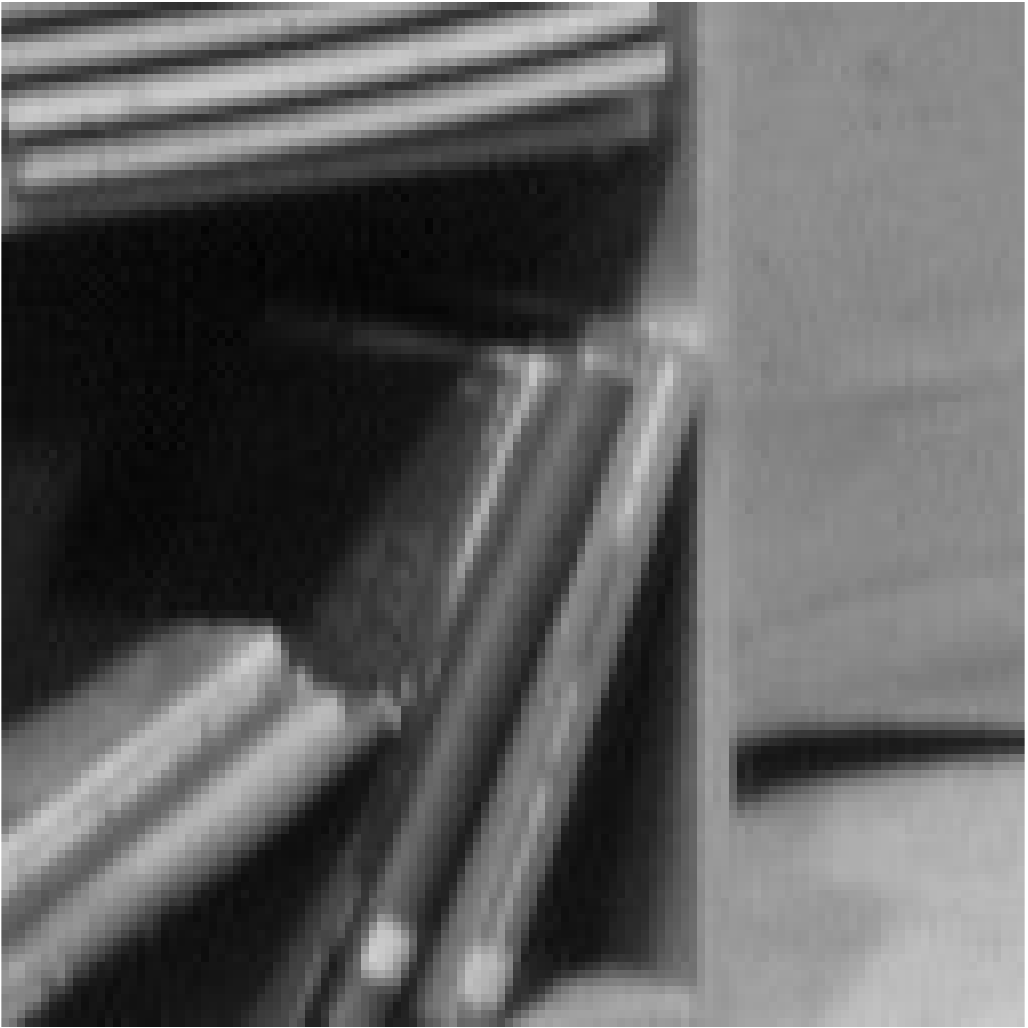}\includegraphics[width=0.166\linewidth]{{{nl_ple/barbara/zoom/aguer050}.pdf}}\includegraphics[width=0.166\linewidth]{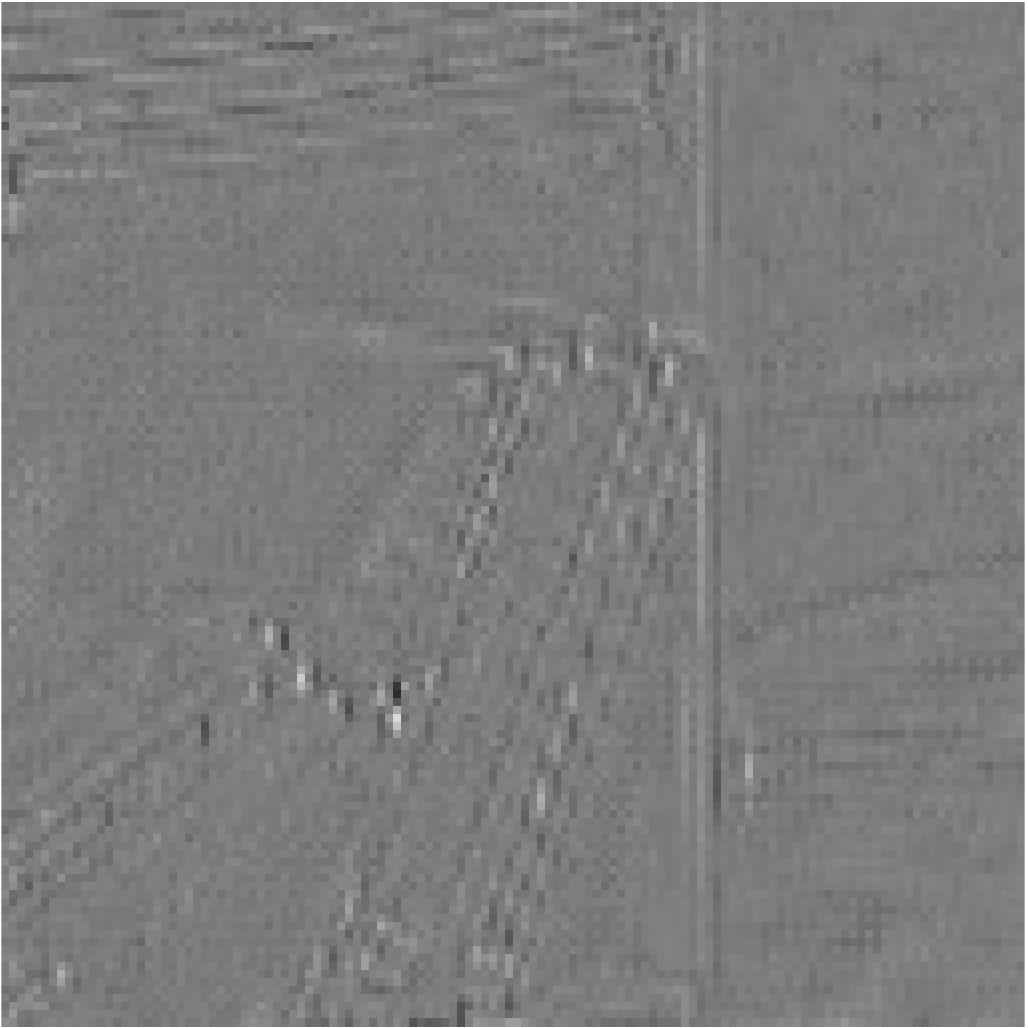}\includegraphics[width=0.166\linewidth]{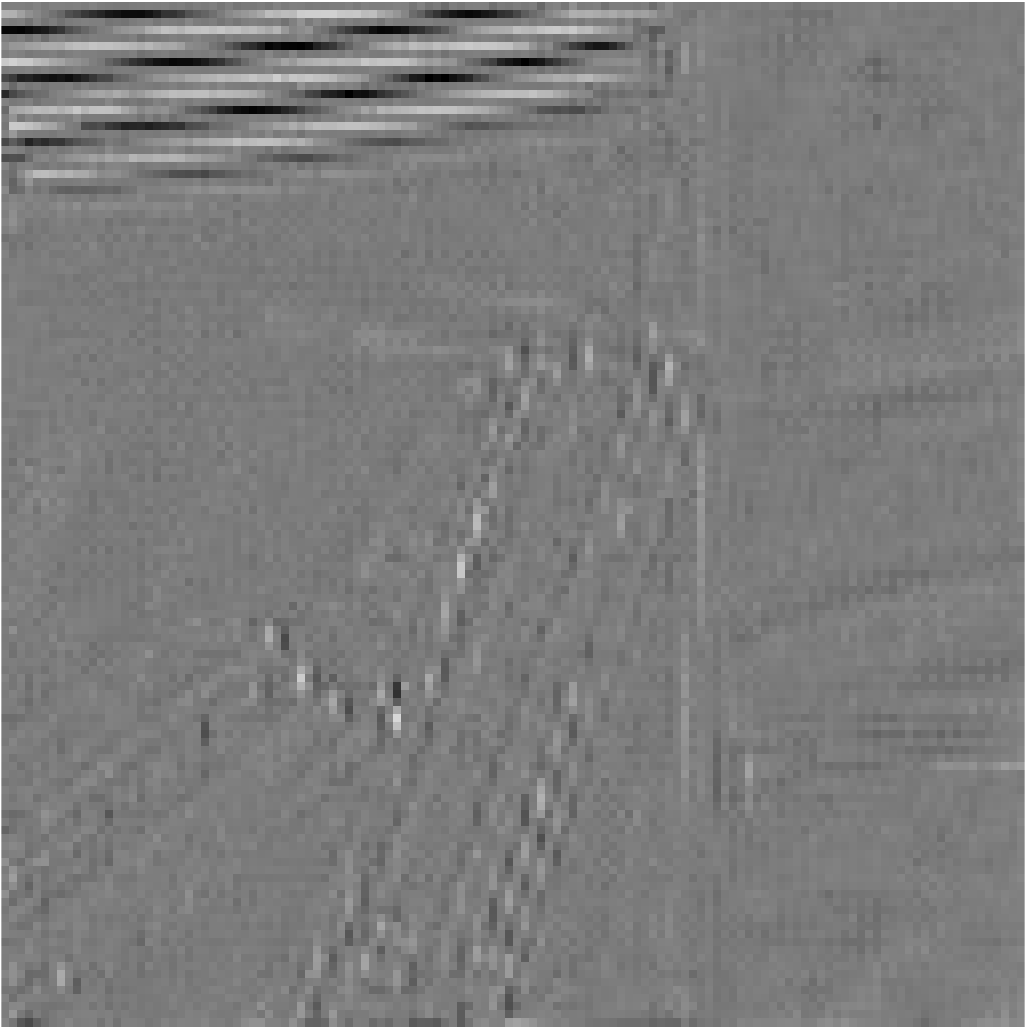}\includegraphics[width=0.166\linewidth]{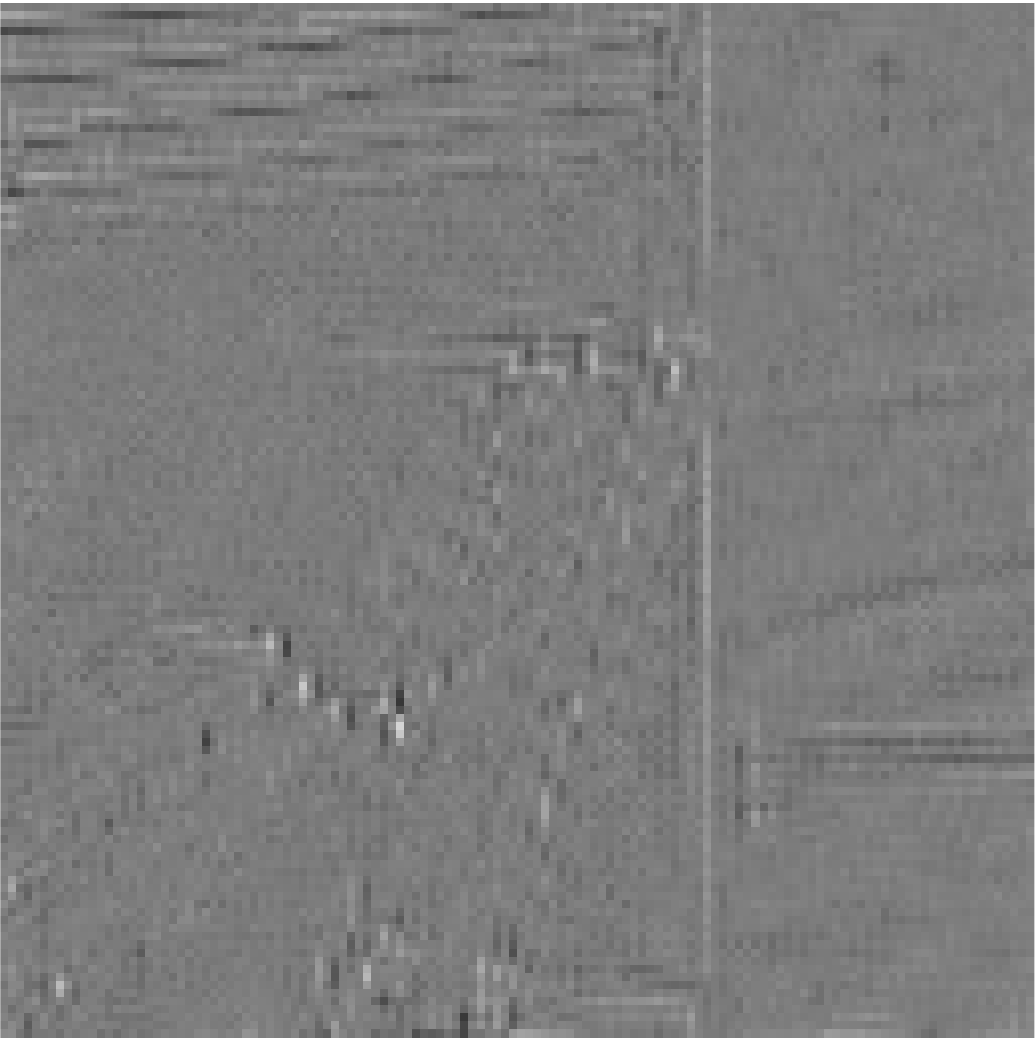}\includegraphics[width=0.166\linewidth]{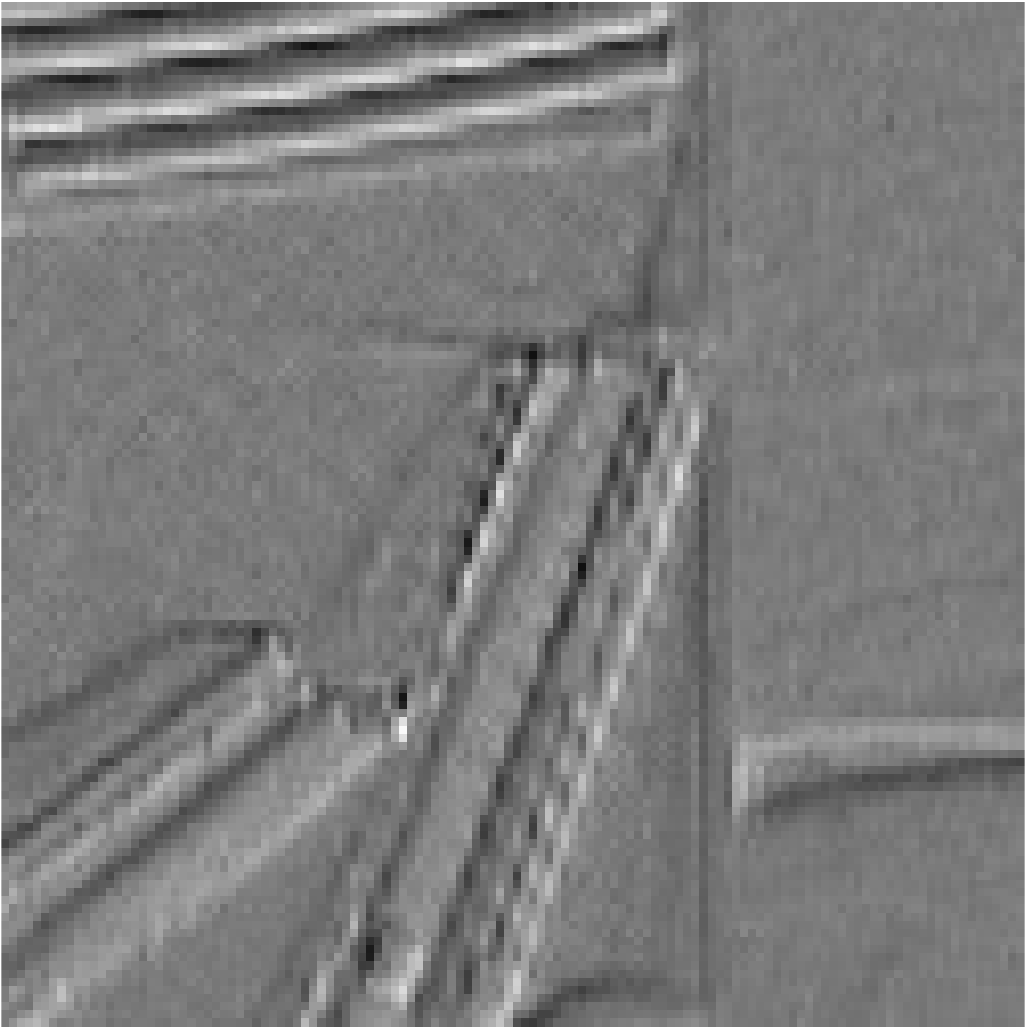}

\caption{\textbf{Synthetic data. Zooming $\times 2$. Left to right:} (first row) Ground-truth high resolution image (extract of barbara). Result by \hpnlb{}, PLE, EPLL, E-PLE, lanczos interpolation. (second row) Input low-resolution image, difference with respect to the ground-truth of each of the corresponding results. Please see the digital copy for better details reproduction.}
\label{fig:syntheticZooming1}
\end{figure*}

\subsection{Real data} 
\begin{figure}
\centering
\begin{minipage}[c]{.45\linewidth}
\begin{center}
\includegraphics[width=\linewidth]{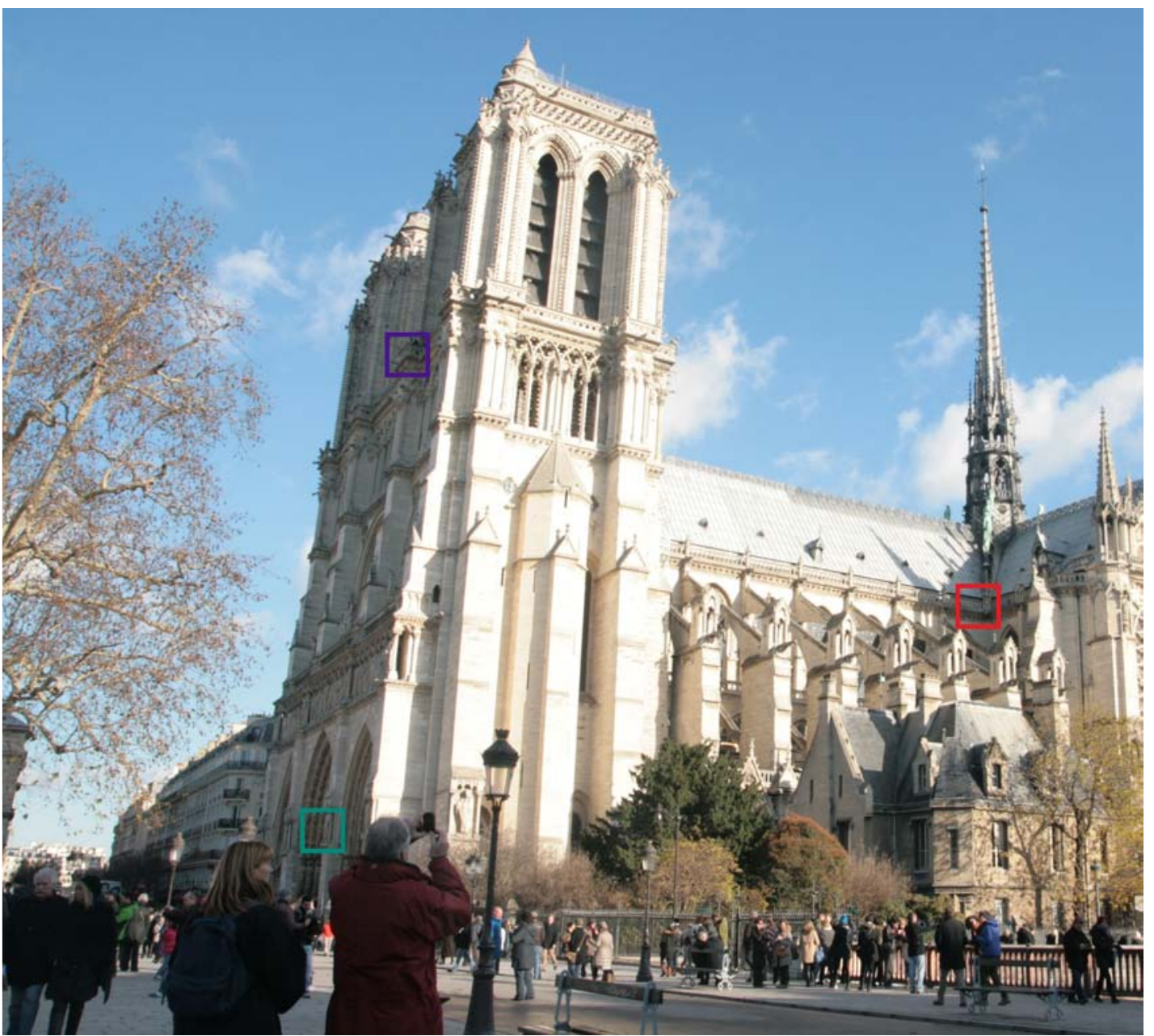}
\end{center}
\end{minipage}
\begin{minipage}[c]{.5\linewidth}
\includegraphics[width=.4\linewidth]{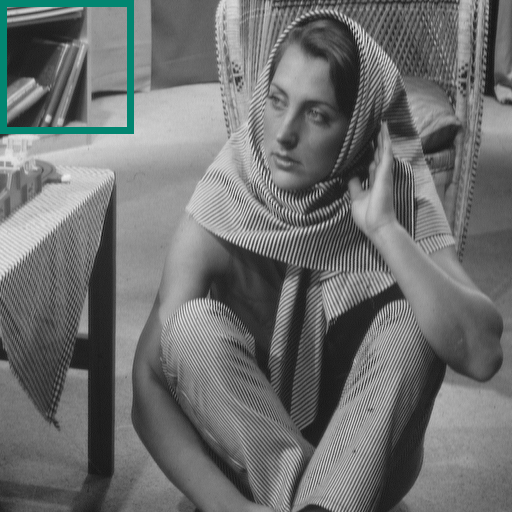}\hspace{.5pt}\includegraphics[width=.4\linewidth]{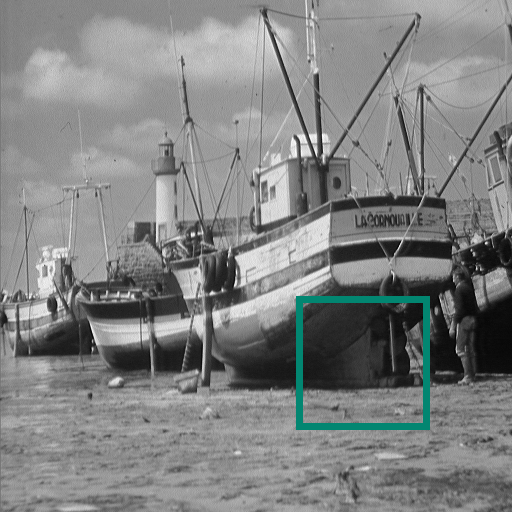}
\vspace{.5pt}

\includegraphics[width=.4\linewidth]{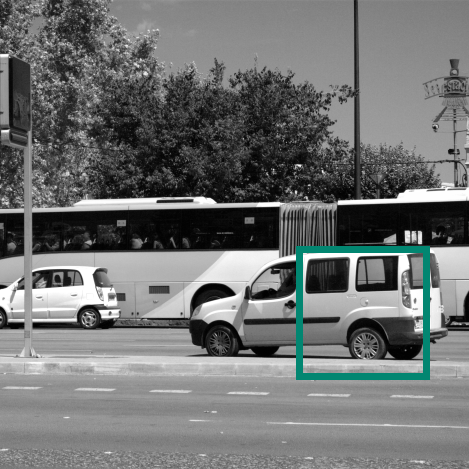}
\end{minipage}
\caption{\textbf{Left.} \textbf{Real data.} JPEG version of the raw image used in the experiments presented in Section~\ref{exps:realData}. The boxes show the extracts displayed in Figure~\ref{fig:realDataZooming}. \textbf{Right.} \textbf{Synthetic data.} Ground-truth images used in the experiments presented in Section~\ref{ssec:synthExps}. The green boxes indicate the extracts used for the zooming experiments.}
\label{fig:gtruthRealData}
\end{figure}
\label{exps:realData}
For this experiment, we use raw images captured with a Canon 400D camera (set to ISO 400 and exposure time 1/160 seconds). The main noise sources for CMOS sensors are: the Poisson photon shot noise, which can be approximated by a Gaussian distribution with equal mean and variance; the thermally generated readout noise, which is modeled as an additive Gaussian distributed noise and the spatially varying gain given by the photo response non uniformity (\prnu{})~\cite{aguerrebere12,aguerrebere13}. We thus consider the following noise model for the non saturated raw pixel value $\Zp(p)$ at position $p$ 
\begin{equation}
\Zp(p) \sim  \N(\a \gp_p \tau \Cp(p)  + \meanR, \a^2 \gp_p \tau \Cp(p)  + \sn),
\label{eq:modelZOrig}
\end{equation}
where $\a$ is the camera gain, $\gp_p$ models the \prnu{} factor, $\tau$ is the exposure time, $\Cp(p)$ is the irradiance reaching pixel $p$, $\meanR$ and $\sn$ are the readout noise mean and variance. The camera parameters have to be estimated by a calibration procedure~\cite{aguerrebere12}. The noise covariance matrix $\S_{N}$ is thus diagonal with entries that depend on the pixel value $(\S_{N})_p=  \a^2 \gp_p \tau \Cp(p)  + \sn$.

In order to evaluate the interpolation capacity of the proposed approach, we consider the pixels of the green channel only (i.e. 50\% of the pixels in the RGGB Bayer pattern) and interpolate the missing values. We compare the results to those obtained using an adaptation of PLE to images degraded with noise with variable variance (PLEV)~\cite{aguerrebere14ICCP}. The results for the EPLL and E-PLE methods are not presented here since these methods are not suited for this kind of noise. Figure~\ref{fig:realDataZooming} shows extracts of the obtained results (see Figure~\ref{fig:gtruthRealData} for a JPEG version of the raw image showing the location of the extracts). As it was already observed in the synthetic data experiments, fine details and edges are better preserved. Compare for example the reconstruction of the balcony edges and the wall structure in the first row of Figure~\ref{fig:realDataZooming}, as well as the structure of the roof and the railing in the second row of the same image.
\begin{figure*}
\centering
\fboxsep=0pt\fboxrule=1pt\fcolorbox{myViolet}{white}{\includegraphics[width=0.17\linewidth]{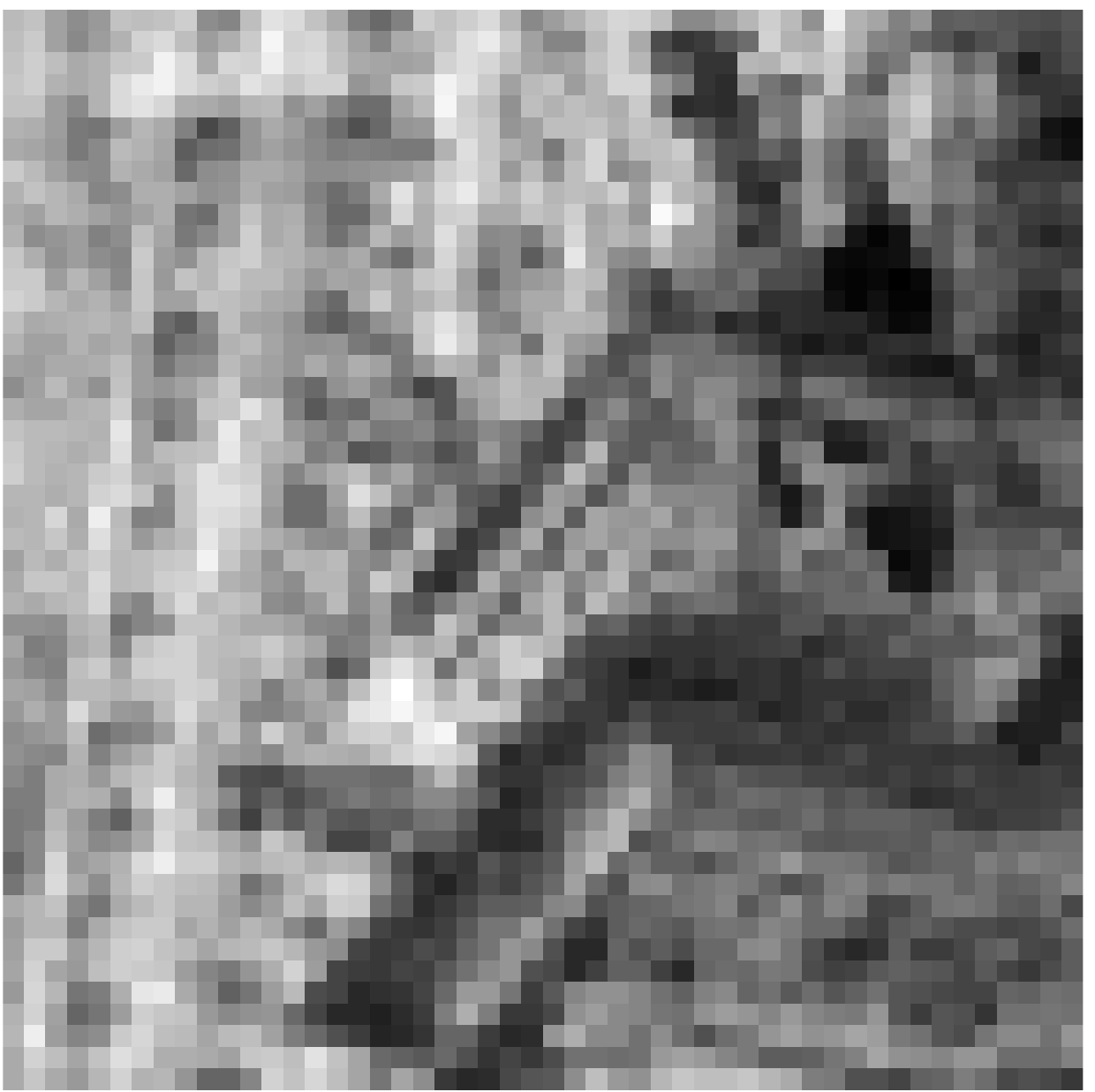}}\includegraphics[width=0.17\linewidth]{{{ndame/intDeno/ext_4/pdf/aguer060}.pdf}}\includegraphics[width=0.17\linewidth]{{{ndame/intDeno/ext_4/pdf/aguer061}.pdf}}\includegraphics[width=0.17\linewidth]{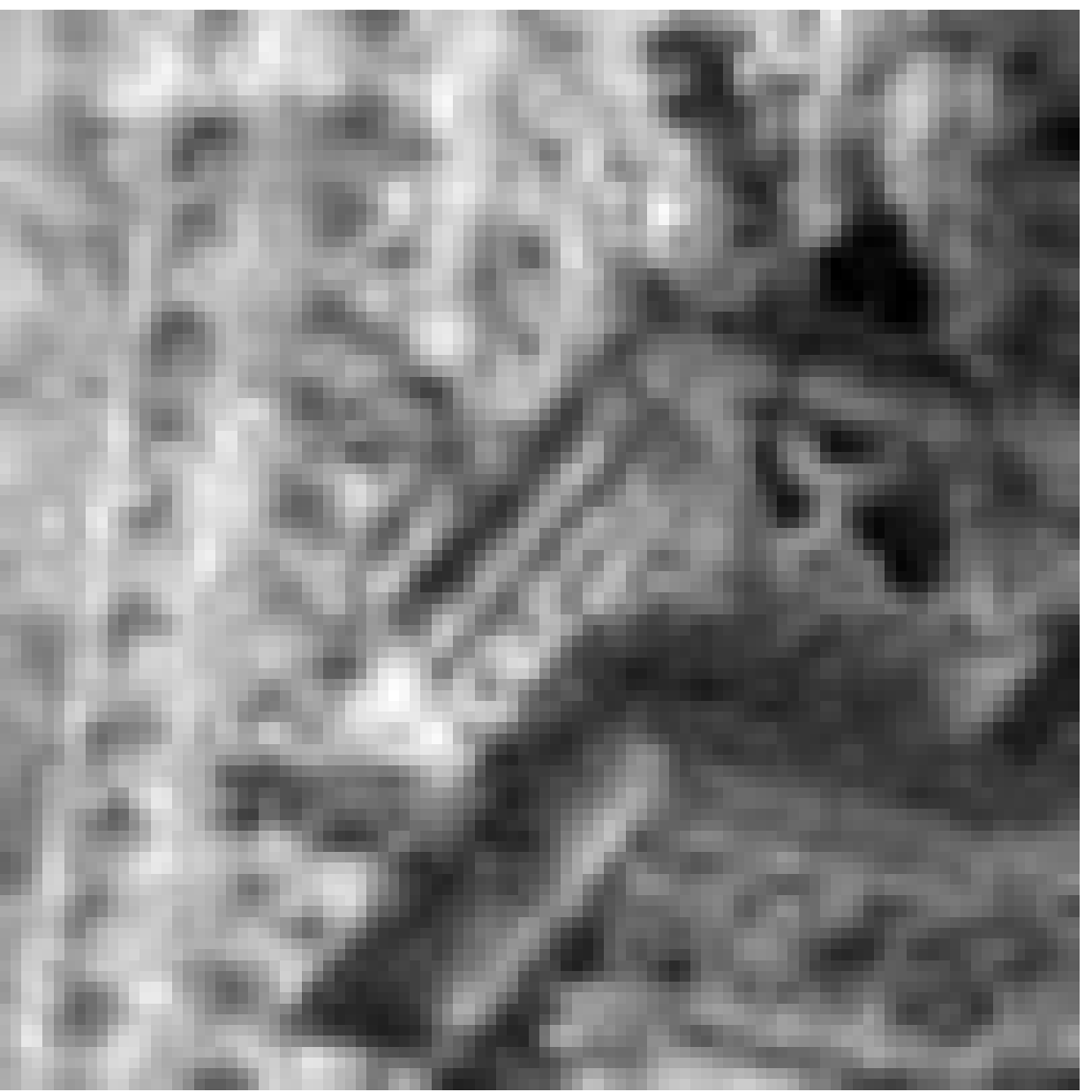}\includegraphics[width=0.17\linewidth]{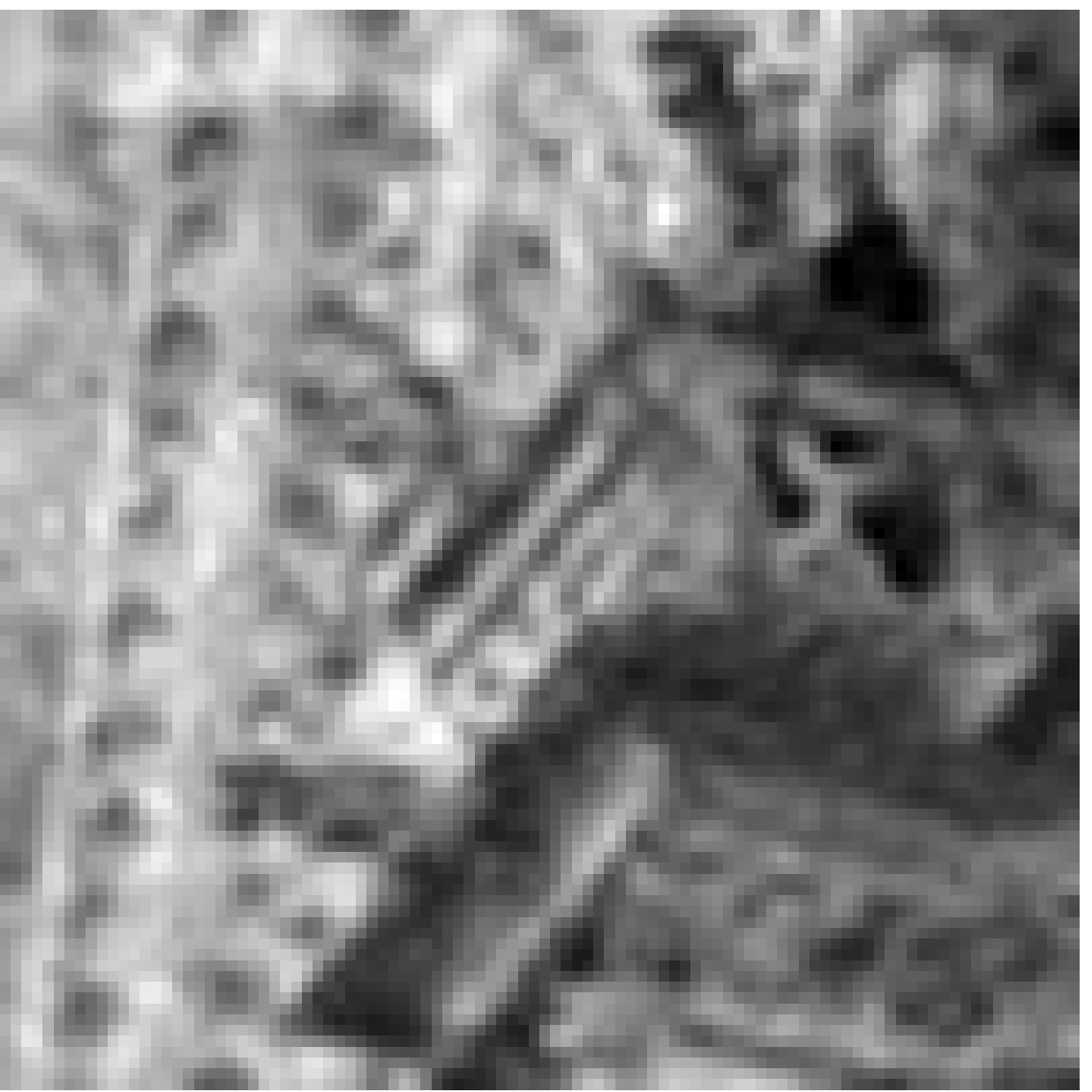} \\
\fboxsep=0pt\fboxrule=1pt\fcolorbox{myRed}{white}{\includegraphics[width=0.17\linewidth]{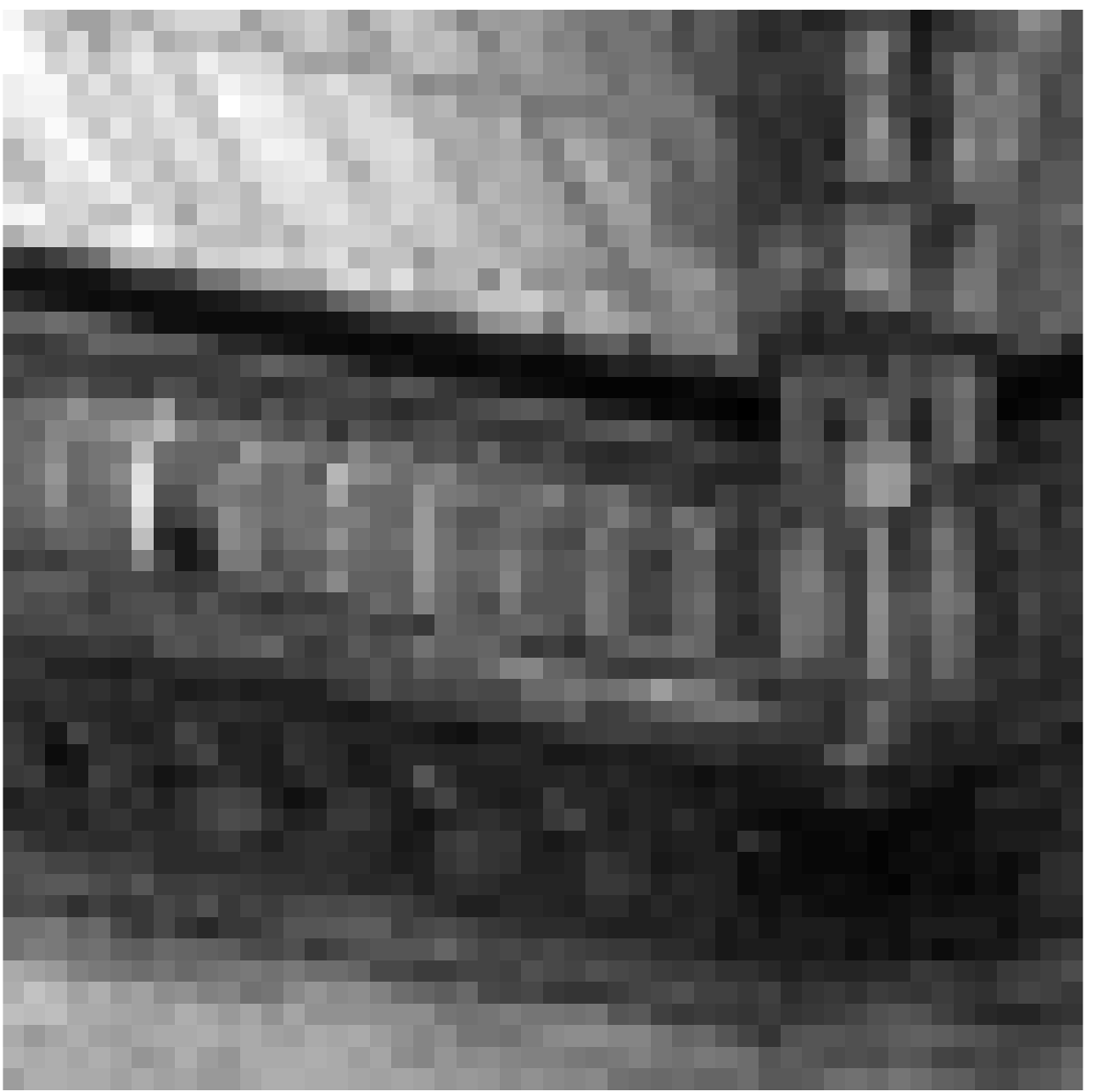}}\includegraphics[width=0.17\linewidth]{{{ndame/intDeno/ext_3/pdf/aguer065}.pdf}}\includegraphics[width=0.17\linewidth]{{{ndame/intDeno/ext_3/pdf/aguer066}.pdf}}\includegraphics[width=0.17\linewidth]{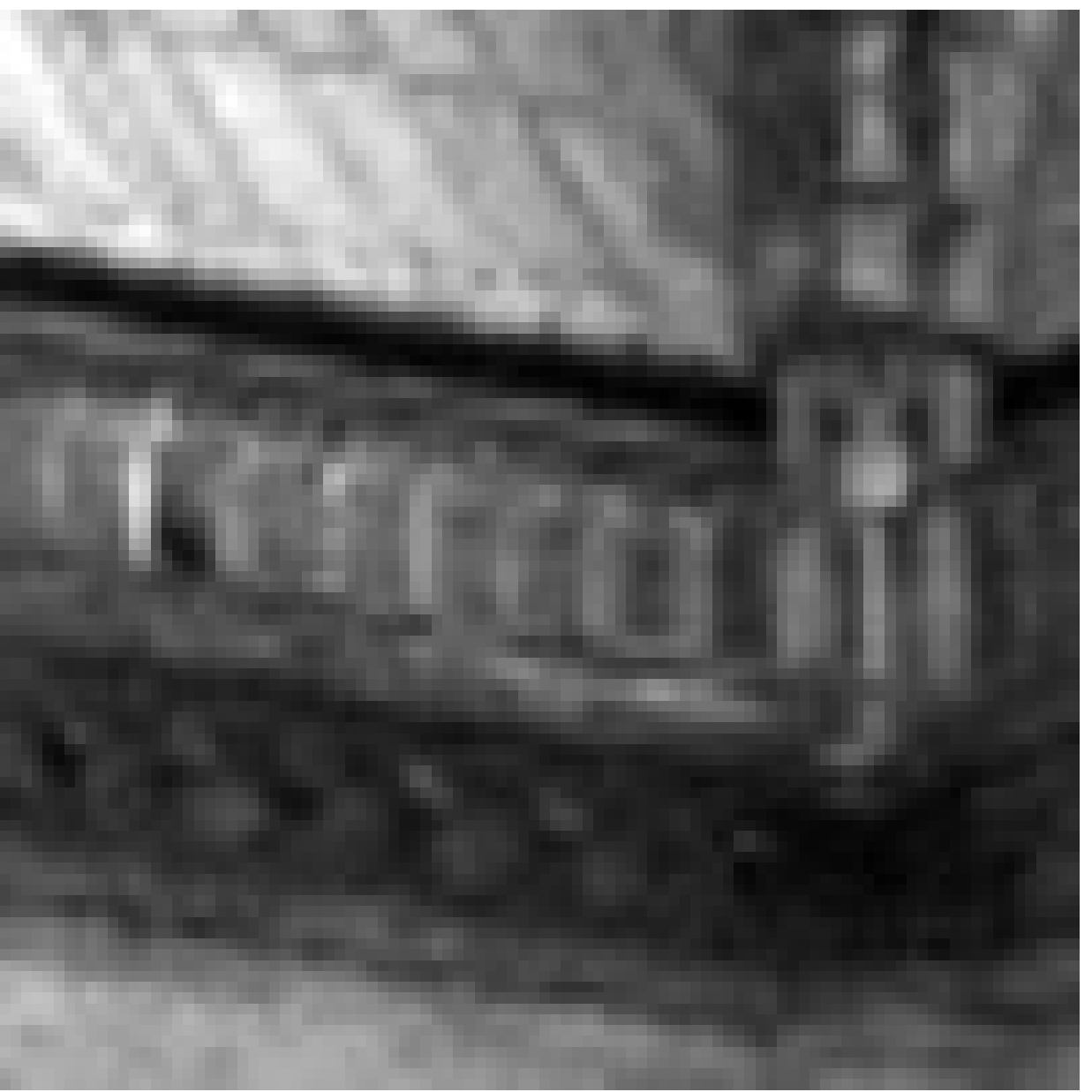}\includegraphics[width=0.17\linewidth]{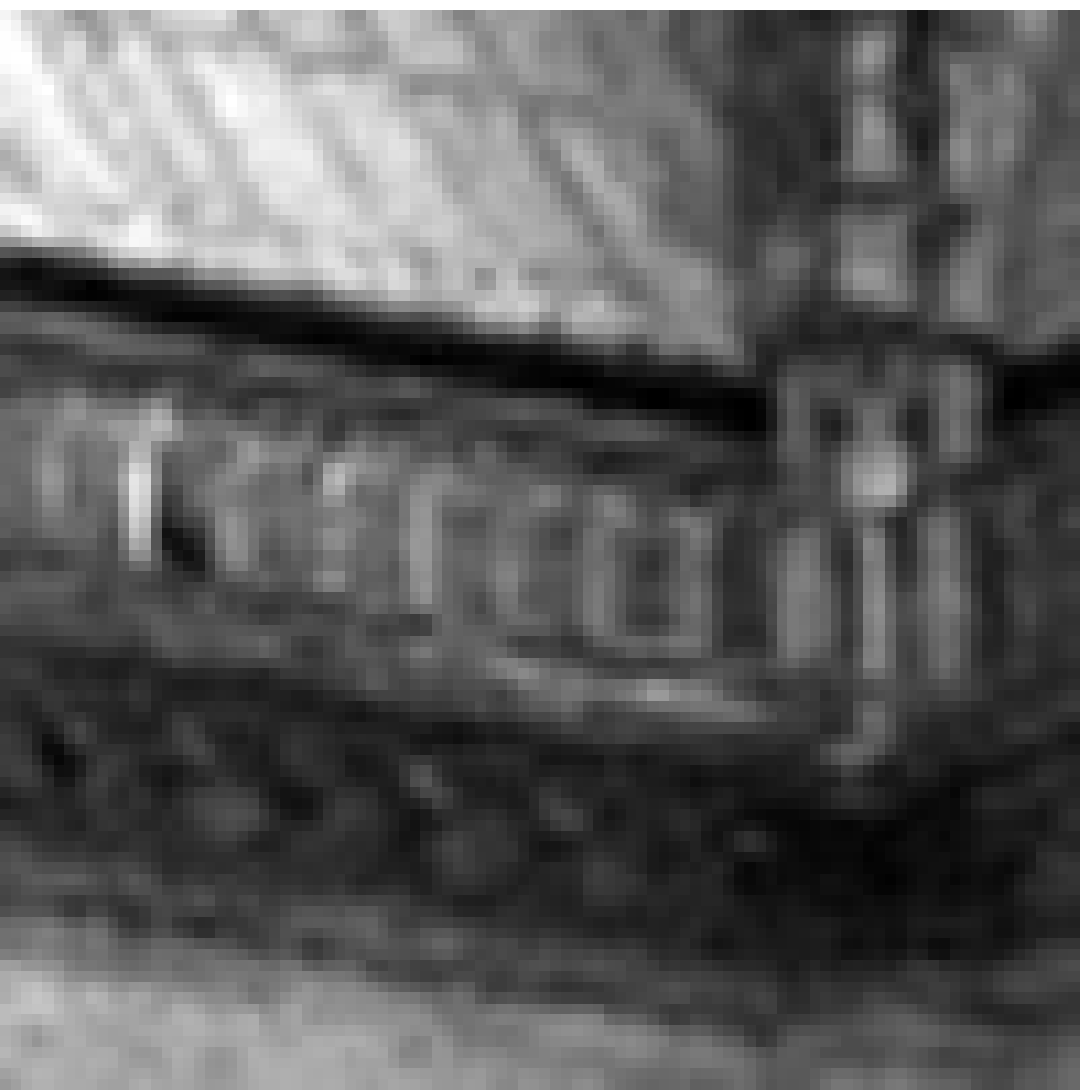} \\
\fboxsep=0pt\fboxrule=1pt\fcolorbox{myGreen}{white}{\includegraphics[width=0.17\linewidth]{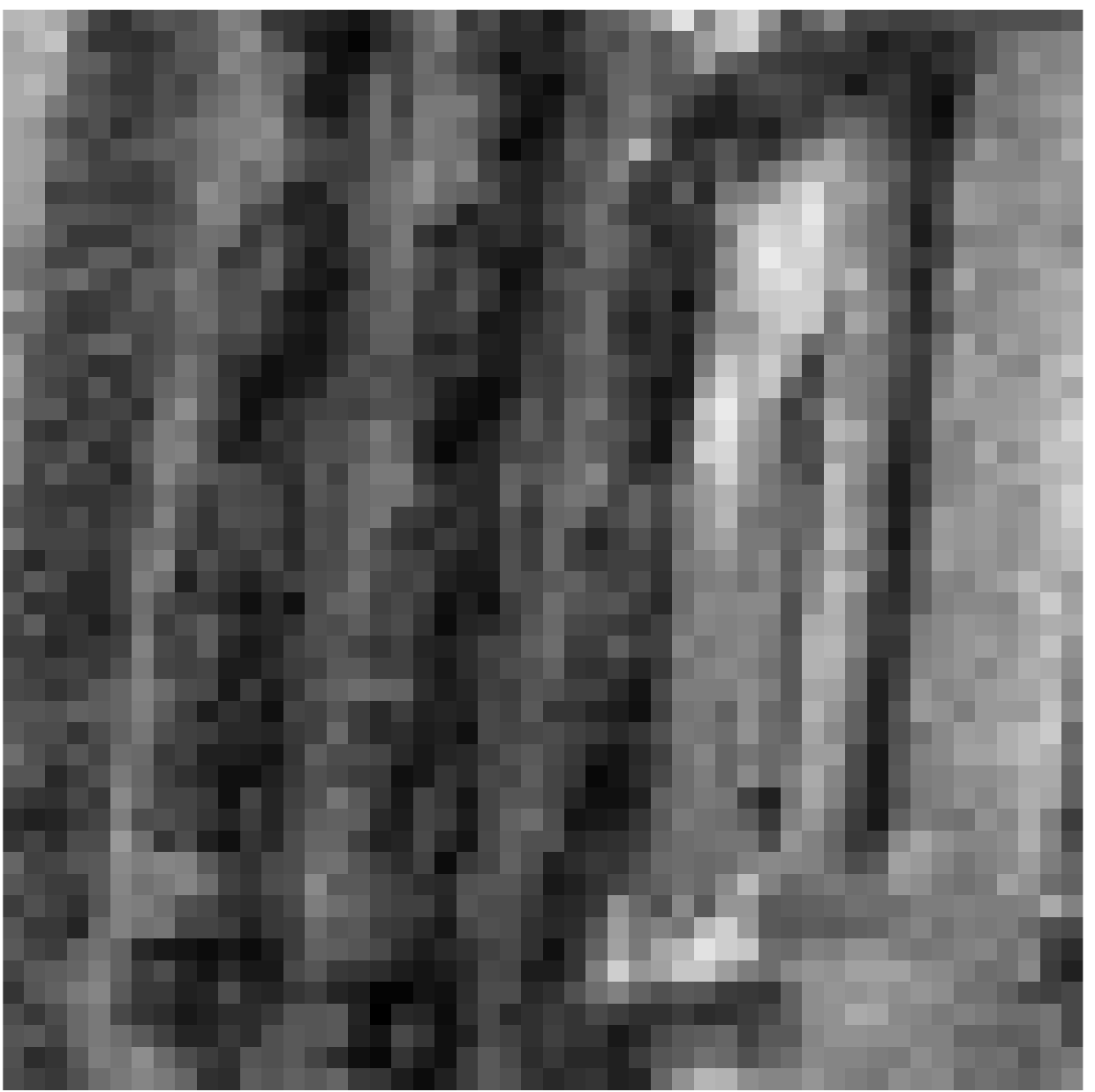}}\includegraphics[width=0.17\linewidth]{{{ndame/intDeno/ext_5/pdf/aguer070}.pdf}}\includegraphics[width=0.17\linewidth]{{{ndame/intDeno/ext_5/pdf/aguer071}.pdf}}\includegraphics[width=0.17\linewidth]{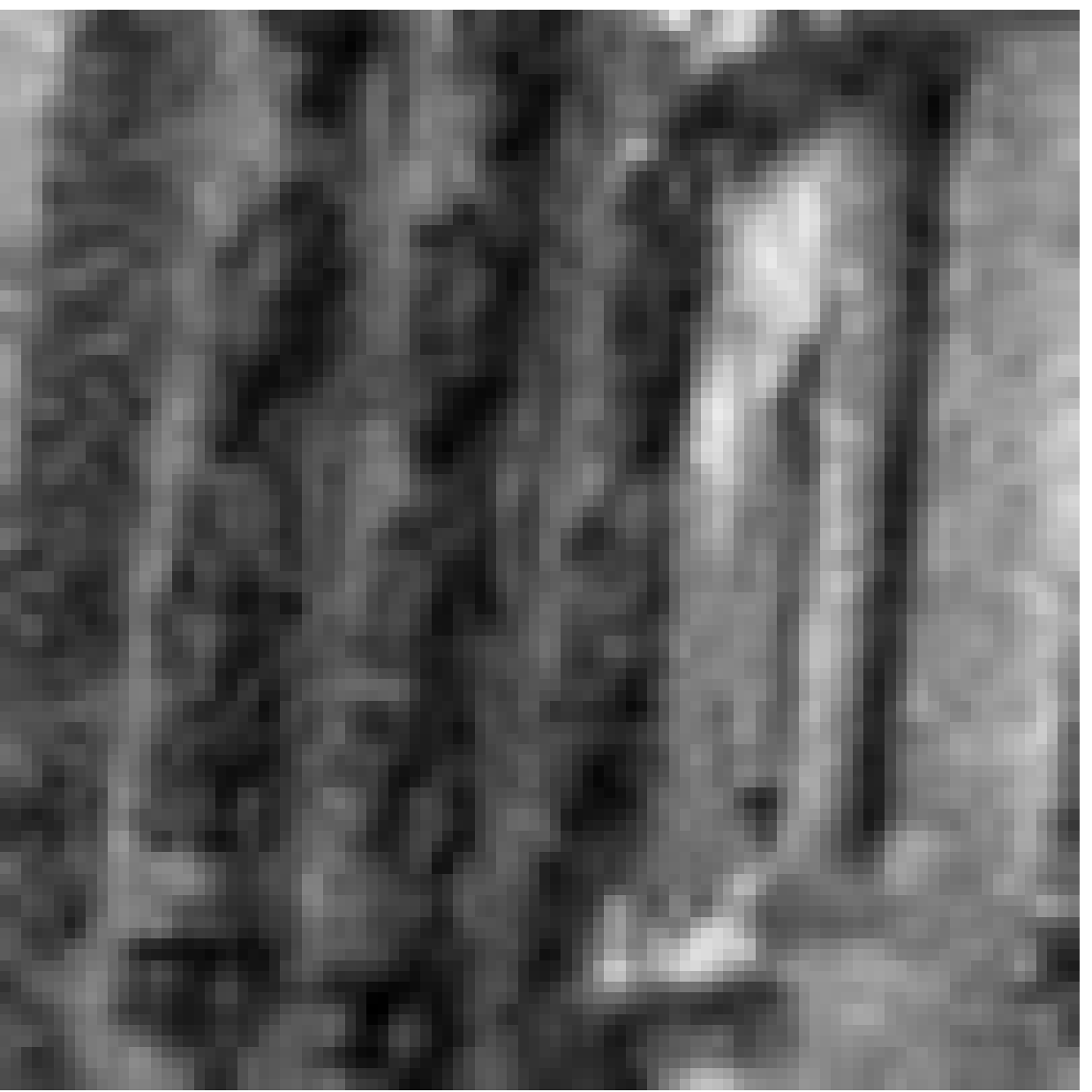}\includegraphics[width=0.17\linewidth]{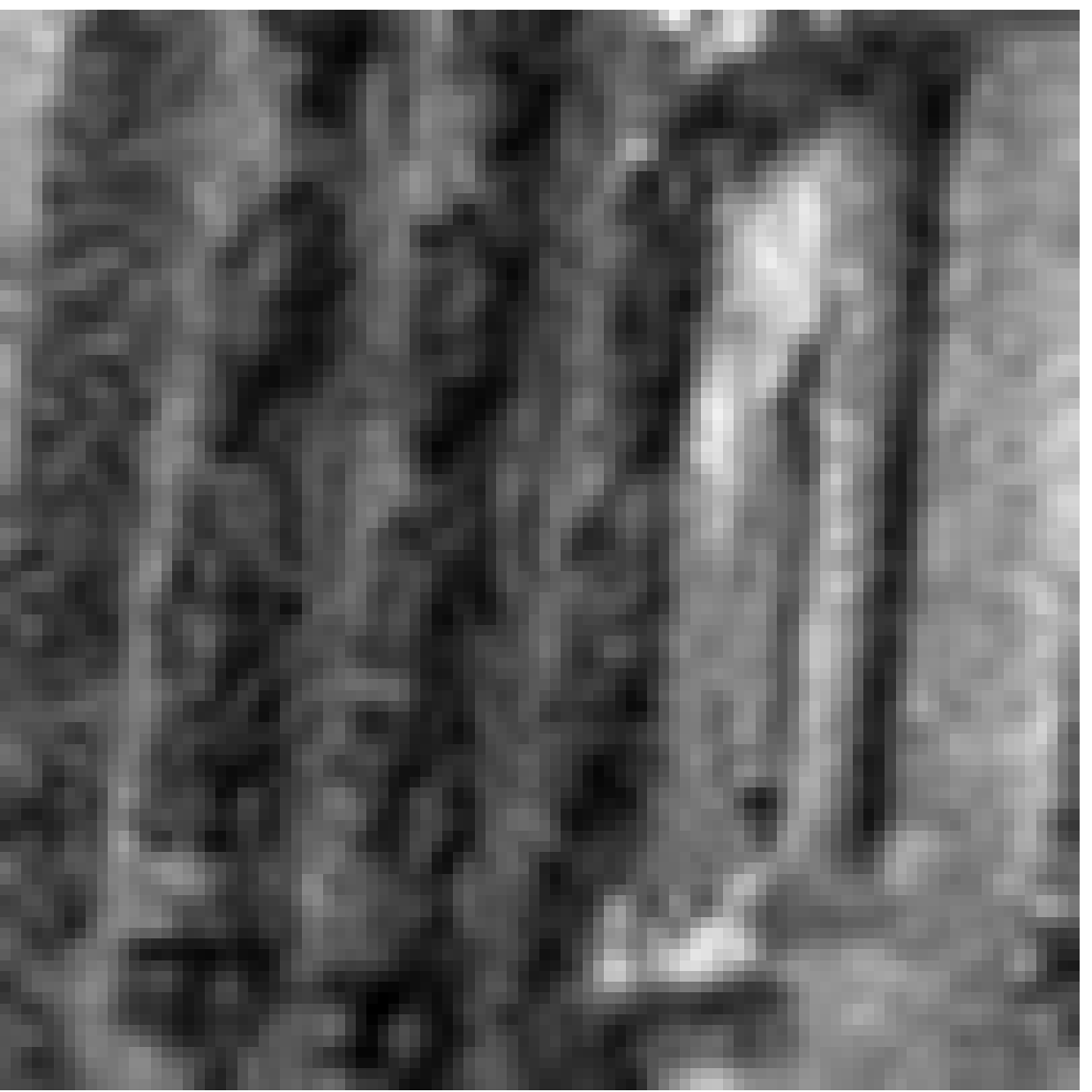}

\scriptsize{Ground-truth\hspace{55pt} \hpnlb{} \hspace{65pt} PLEV \hspace{60pt} Bicubic \hspace{60pt} Lanczos}
\caption{\textbf{Real data. Zooming $\times 2$.} Interpolation of the green channel of a raw image (RGGB). \textbf{Left to right:} Input low-resolution image, result by \hpnlb{}, PLEV~\cite{aguerrebere14ICCP}, bicubic and lanczos interpolation.}
\label{fig:realDataZooming}
\end{figure*}

\subsection{Discussion}

In all the previous experiments, the results obtained with \hpnlb{} outperform or are very close to those obtained by the other evaluated methods. Visually, details are better reconstructed and images are sharper. 

We interpret this as the inability of a fixed set of patch classes (19 for PLE) to accurately represent patches that seldom appear in the image, such as edges or textures (as in Barbara's trouser). The fact that methods such as PLE are actually semi-local (classes are estimated on $128 \times 128$ regions~\cite{yu12}) does not solve this issue. On the contrary, a local model estimation as the one performed by \hpnlb{} correctly handles those cases. 

The performance difference is much more remarkable for the higher masking rates. In such cases, the robustness of the estimation is crucial. Indeed the proposed method performs the estimation from similar patches in a local window. The hypothesis of self-similarity being reinforced by considering local neighbourhoods, such a strategy restricts the possible models, therefore making the estimation more robust. Furthermore, the local model estimation, previously shown to be successful at describing patches~\cite{lebrun13}, gives a better reconstruction even when a very large part of the patch is missing.

EPLL uses more mixture components (200 components are learnt from $2 \times 10^6$ patches of natural images~\cite{zoran11}) in its GMM model than PLE. It was observed in~\cite{wang13} that this strategy is less efficient than PLE for the denoising task. In this work, we also observe that the proposed approach outperforms EPLL, not only in denoising, but also in inpainting and zooming. However, here it is  harder to tell if the improvement is due to the local model estimation performed from similar patches or to the different restoration strategies followed by these methods.

\section{High dynamic range imaging from a single snapshot}
\label{sec:HDR}
In this section, we apply the proposed framework to generate HDR images from a single shot. HDR imaging aims at reproducing an extended dynamic range of luminosity compared to what can be captured using a standard digital camera, which is often not enough to produce an accurate representation of real scenes. In the case of a static scene and a static camera, the combination of multiple images with different exposure levels is a simple and efficient solution~\cite{debevec97,granados10,aguerrebere13}. However, several problems arise when either the camera or the elements in the scene move~\cite{aguerrebere13ICCP,sidibe2009ghost}. 

An alternative to the HDR from multiple frames is to use a single image with spatially varying pixel exposures (SVE)~\cite{nayar00}. An optical mask with spatially varying transmittance is placed adjacent to a conventional image sensor, thus controlling the amount of light that reaches each pixel (see Figure~\ref{fig:HDR_synth})~\cite{nayar00,yasuma10,schoberl13}.  

The greatest advantage of this acquisition method is that it avoids the need for image alignment and motion estimation. Another advantage is that the saturated pixels are not organized in large regions. Indeed, some recent multi-image methods tackle motion problems by taking a reference image and then by estimating motion or reconstructing the image relative to this reference~\cite{sen12,aguerrebere13ICCP}. A problem encountered by these approaches is the need to inpaint very large saturated and underexposed regions in the reference frame. The SVE acquisition strategy avoids this problem since, in general, all scene regions are sampled by at least one of the exposures.

Taking advantage of the ability of the proposed framework to simultaneously estimate missing pixels and denoise well-exposed ones, we propose a novel approach to generate HDR images from a single shot acquired with spatially varying pixel exposures. The proposed approach shows significant improvements over existing methods.

\subsection{Spatially varying exposure acquisition model}
\label{sec:model}
 
An optical mask with spatially varying transmittance is placed adjacent to a conventional image sensor to give different exposure levels to the pixels. This optical mask does not change the acquisition process of the sensor. Hence, the noise model~\eqref{eq:modelZOrig} can be adapted to the SVE acquisition by including the per-pixel SVE gain $\fa_p$\footnote{Some noise sources not modeled here, such as blooming, might have an impact in the SVE acquisition strategy and should be considered in a more accurate image modeling.}:
\begin{equation}
\Zp(p) \sim  \N(\a \fa_p \gp_p \tau \Cp(p)  + \meanR, \a^2 \fa_p \gp_p \tau \Cp(p)  + \sn).
\label{eq:modelZ}
\end{equation}

In the approach proposed by Nayar and Mitsunaga~\cite{nayar00}, the varying exposures follow a regular pattern. Motivated by the aliasing problems of regular sampling patterns, Sch\"oberl  et al.~\cite{schoberl12} propose to use spatially varying exposures on a non-regular pattern. Figure~\ref{fig:HDR_synth} shows examples of both acquisition patterns. This observation led us to choose the non-regular pattern in the proposed approach.
\begin{figure*}
\centering
\begin{minipage}[c]{.16\linewidth}
\begin{center}
\includegraphics[width=.85\linewidth]{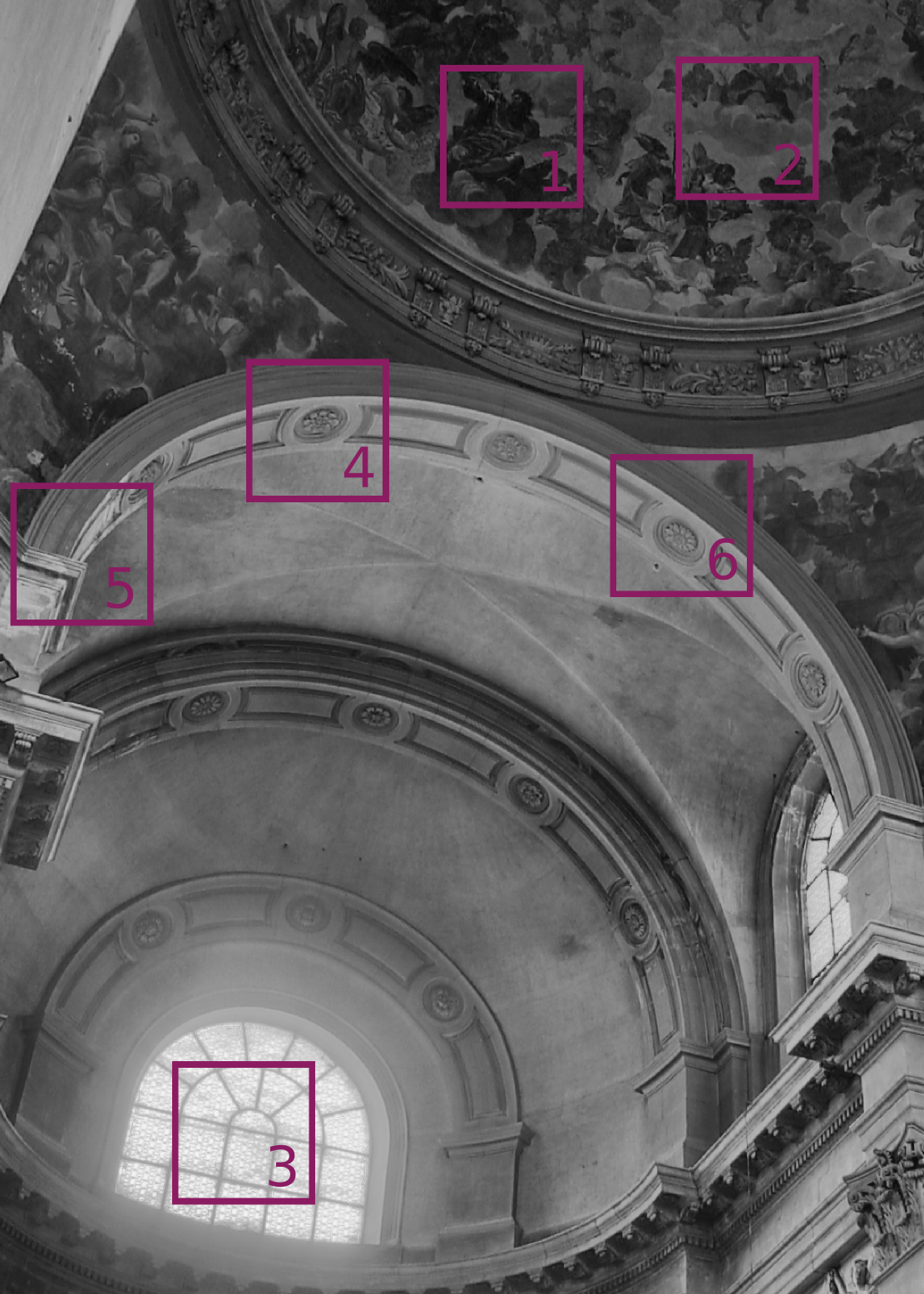}

\vspace{.4em}

\includegraphics[width=0.48\linewidth]{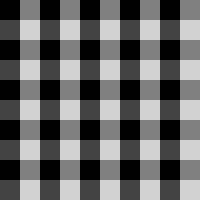} \includegraphics[width=0.48\linewidth]{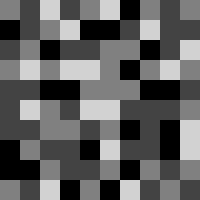}
\end{center}
\end{minipage}
\begin{minipage}[c]{.83\linewidth}
\begin{center}
\fboxsep=0pt\fboxrule=.8pt\fcolorbox{myGreen}{white}{\includegraphics[width=0.16\linewidth]{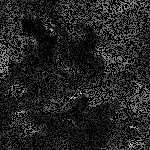}}\hspace{.5pt}\fboxsep=0pt\fboxrule=.8pt\fcolorbox{myRed}{white}{\includegraphics[width=0.16\linewidth]{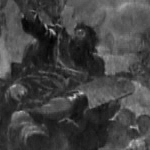}}\hspace{.5pt}\includegraphics[width=0.16\linewidth]{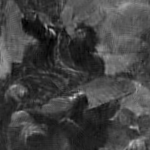}\hspace{.5pt}\includegraphics[width=0.16\linewidth]{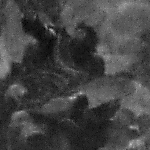}\hspace{.5pt}\includegraphics[width=0.16\linewidth]{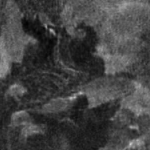}\hspace{.5pt}\includegraphics[width=0.16\linewidth]{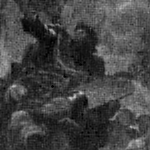}\\
\fboxsep=0pt\fboxrule=.8pt\fcolorbox{myGreen}{white}{\includegraphics[width=0.16\linewidth]{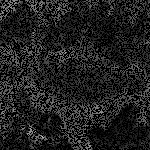}}\hspace{.5pt}\fboxsep=0pt\fboxrule=.8pt\fcolorbox{myRed}{white}{\includegraphics[width=0.16\linewidth]{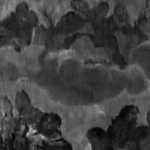}}\hspace{.5pt}\includegraphics[width=0.16\linewidth]{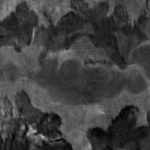}\hspace{.5pt}\includegraphics[width=0.16\linewidth]{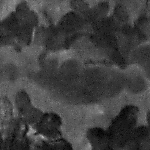}\hspace{.5pt}\includegraphics[width=0.16\linewidth]{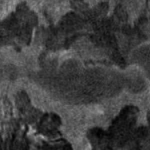}\hspace{.5pt}\includegraphics[width=0.16\linewidth]{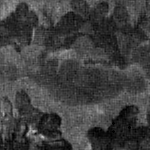}
\end{center}
\scriptsize{\hspace{30pt} Input \hspace{40pt} Ground-truth \hspace{43pt} \hpnlb{} \hspace{50pt} PLEV \hspace{35pt} \Schoberl{} et al. \hspace{10pt}  Nayar \& Mitsunaga}
\end{minipage}
\caption{\textbf{Synthetic data.} \textbf{Left:}  (\textbf{top}) Tone mapped version of the ground-truth image used for the experiments in Section~\ref{ssec:expSynthHDR}.  
(\textbf{bottom}) Regular (left) and non-regular (right) optical masks for an example of 4 different filters. \textbf{Right:} Results for extracts 1 and 6. From left to right: Input image with random pattern, ground-truth, results by \hpnlb{}, PLEV~\cite{aguerrebere14ICCP}, \Schoberl{} et al.~\cite{schoberl12}, Nayar and Mitsunaga~\cite{nayar00}. 50\% missing pixels (for both random and regular pattern). See \psnr{} values for these extracts in Table~\ref{tab:HDRpsnr}. Please see the digital copy for better details reproduction.}
\label{fig:HDR_synth}
\end{figure*}

 \subsection{Hyperprior Bayesian Estimator for Single Shot High Dynamic Range Imaging}
\label{sec:irrEst}
\subsubsection{Problem statement}
In order to reconstruct the dynamic range of the scene we need to solve an inverse problem. We want to estimate the irradiance image $\C$ from the SVE image $\Z$, knowing the exposure levels of the optical mask and the camera parameters. For this purpose we map the raw pixel values to the irradiance domain $\Yp$ with
\begin{equation}
\Yp(p) =  \frac{\Zp(p) - \meanR}{\a \fa_p \gp_p \tau}. 
\label{eq:Y}
\end{equation}
We take into account the effect of saturation and under-exposure by introducing the exposure degradation matrix $\U$, a diagonal matrix given by
\begin{equation} 
(\U)_p = \left\{
\begin{array}{rl}
1 & \text{if } \mean_R < \Zp(p) < z_{sat}, \\
0 & \text{otherwise},
\end{array} \right.
\label{eq:U}
\end{equation}
with $z_{sat}$ equal to the pixel saturation value. From~\eqref{eq:modelZ} and~\eqref{eq:U}, $\Yp(p)$ can be modeled as
\begin{equation}
\Yp(p) |  (\U)_p \sim \N \left( (\U)_p\Cp(p),\frac{\a^2 \fa_p \gp_p \tau (\U)_p\Cp(p)   + \sn}{(\a \fa_p \gp_p \tau)^2} \right).
\label{eq:modelIrr}
\end{equation}
Notice that~\eqref{eq:modelIrr} is the distribution of $\Yp(p)$ for a given exposure degradation factor $(\U)_p$, since $(\U)_p$ is itself a random variable that depends on $\Zp(p)$. The exposure degradation factor must be included in~\eqref{eq:modelIrr} since the variance of the over or under exposed pixels no longer depends on the irradiance $\Cp(p)$ but is only due to the readout noise $\sn$. From~\eqref{eq:modelIrr} we have
\begin{equation}
\Yp = \U \f + \ns,
\label{eq:Ymodel}
\end{equation}
where $\ns$ is zero-mean Gaussian noise with diagonal covariance matrix $\S_{\ns}$ given by
\begin{equation}
(\S_{\ns})_j = \frac{\a^2 \fa_p \gp_p \tau (\U)_p\Cp(p)   + \sn}{(\a \fa_p \gp_p \tau)^2}.
\end{equation}

Then the problem of irradiance estimation can be stated as retrieving $\C$ from $\Yp$, which implies denoising the well-exposed pixel values ($(\U)_p = 1$) and estimating the unknown ones ($(\U)_p = 0$).

\subsubsection{Proposed solution}
From~\eqref{eq:Ymodel}, image $\Yp$ is under the hypothesis of the HBE framework introduced in Section~\ref{sec:newMethod}. The proposed HDR imaging algorithm consists of the following steps: \textbf{1)} generate $\U$ from $\Zp$ according to~\eqref{eq:U}, \textbf{2)} obtain $\Yp$ from $\Zp$ according to~\eqref{eq:Y}, \textbf{3)} apply the HBE approach to $\Yp$ with the given $\U$ and $\S_{\ns}$.

\subsection{Experiments}
\label{sec:exps}

The proposed reconstruction method was thoroughly tested in several synthetic and real data examples. A brief summary of the results is presented in this section.

\subsubsection{Synthetic data}
\label{ssec:expSynthHDR}

Sample images are generated according to Model~\eqref{eq:Ymodel} using the HDR image in Figure~\ref{fig:HDR_synth} as the ground-truth. Both a random and a regular pattern with four equiprobable exposure levels $\fa = \{ 1,8,64,512\}$ are simulated. The exposure time is set to $\tau=1/200$ seconds and the camera parameters are those of a Canon 7D camera set to ISO 200 ($\a=0.87$, $\sn=30$, $\meanR=2048$, $\text{v}_{\text{sat}}=15000$)~\cite{aguerrebere13}.

Figure~\ref{fig:HDR_synth} shows extracts of the results obtained by the proposed method, by PLEV~\cite{aguerrebere14ICCP} (basically an adaptation of PLE to the same single image framework) and by \Schoberl{} et al.~\cite{schoberl12} for the random pattern and by Nayar et Mitsunaga~\cite{nayar00} using the regular pattern. The percentage of unknown pixels in the considered extracts is 50\% (it is nearly the same for both the regular and non-regular pattern). Table~\ref{tab:HDRpsnr} shows the PSNR values obtained in each extract marked in Figure~\ref{fig:HDR_synth}. The proposed method manages to correctly reconstruct the irradiance on the unknown pixels. Moreover, its denoising performance is much better than that of \Schoberl{} et al. and Nayar and Mitsunaga, and still sharper than PLEV.

\begin{table}
\footnotesize
\setlength{\tabcolsep}{3pt}
\centering
\begin{tabular}[h]{c c c c c c c}
\toprule
     & \multicolumn{6}{c}{\textbf{PSNR (dB)}}\\
\cmidrule{2-7}
                &  1 (Fig.~\ref{fig:HDR_synth}) & 2 (Fig.~\ref{fig:HDR_synth}) & 3  & 4  & 5  & 6  \\
\cmidrule{1-7} 
\hpnlb{}            & \textbf{33.08} & \textbf{33.87} & 22.95 & \textbf{35.10} & \textbf{36.80} & \textbf{35.66} \\
PLEV                & 29.65 & 30.82 & 22.77 & 33.99 & 36.42 & 34.73 \\
\Schoberl{} et al.  & 30.38 & 31.16 & 21.39 & 30.04 & 32.84 & 31.02 \\
Nayar and Mitsunaga & 29.39 & 30.10 & \textbf{23.24} & 25.83 & 30.26 & 26.90 \\
\toprule
\end{tabular}
\caption{\psnr{} values for the extracts shown in Figure~\ref{fig:HDR_synth}.}
\label{tab:HDRpsnr}
\end{table}

\begin{figure*}
\centering
\begin{minipage}[c]{.24\linewidth}
\begin{center}
\includegraphics[width=.99\linewidth]{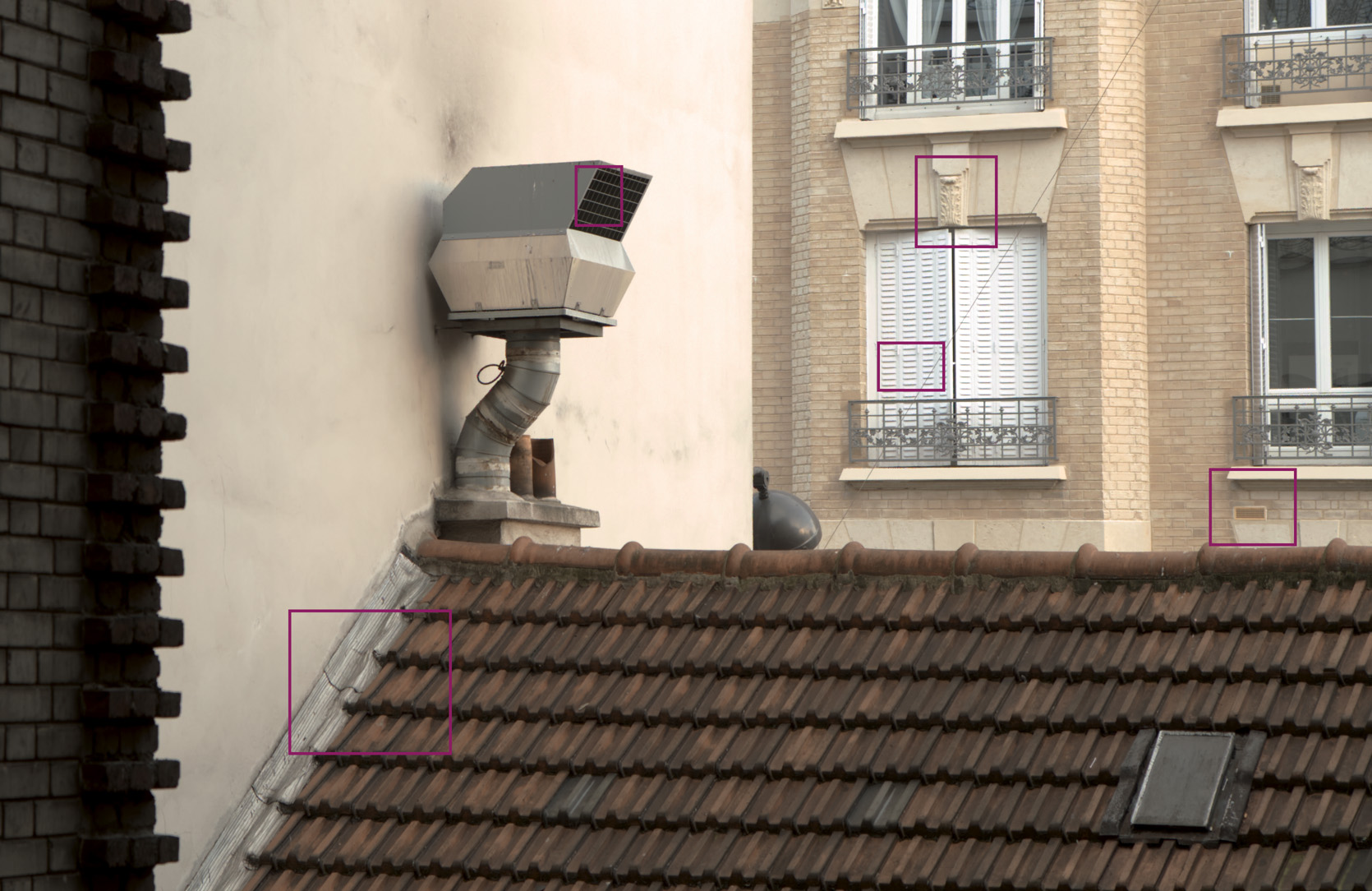}

\vspace{1pt}

\includegraphics[width=.99\linewidth]{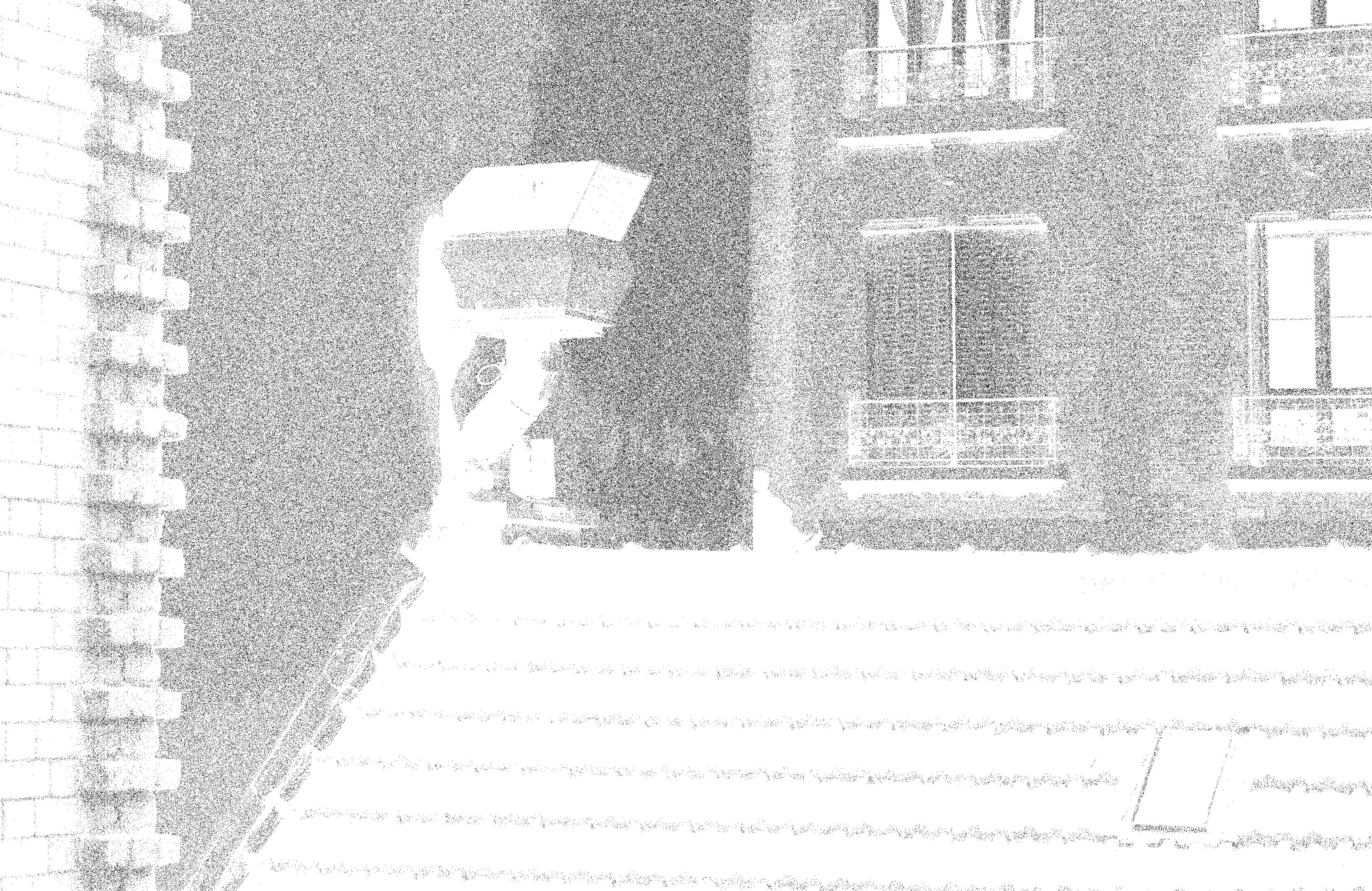}
\end{center}
\end{minipage}
\begin{minipage}[c]{.75\linewidth}
\begin{center}
\includegraphics[height=2.5cm]{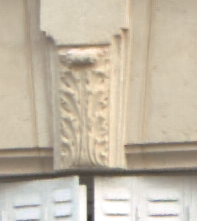}\hspace{.5pt}\includegraphics[height=2.5cm]{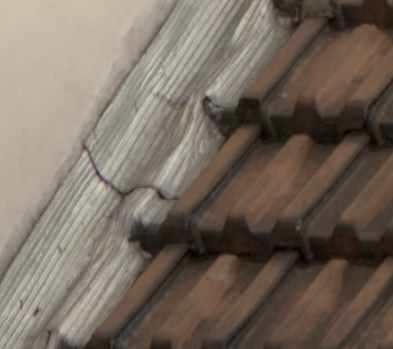}\hspace{.5pt}\includegraphics[height=2.5cm]{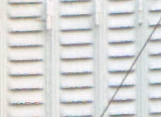}\hspace{.5pt}\includegraphics[height=2.5cm]{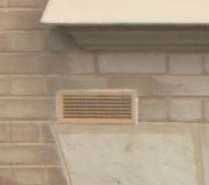}\hspace{.5pt}\includegraphics[height=2.5cm]{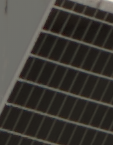}\\
\vspace{1pt}
\includegraphics[height=2.5cm]{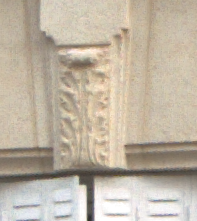}\hspace{.5pt}\includegraphics[height=2.5cm]{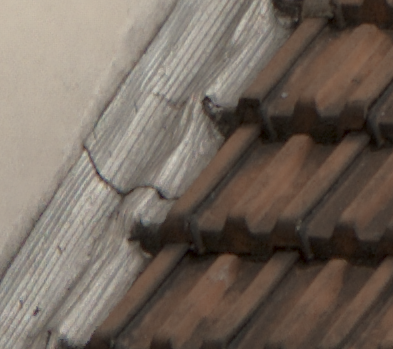}\hspace{.5pt}\includegraphics[height=2.5cm]{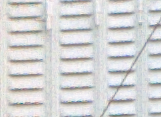}\hspace{.5pt}\includegraphics[height=2.5cm]{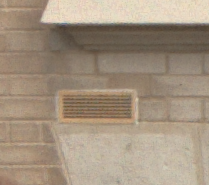}\hspace{.5pt}\includegraphics[height=2.5cm]{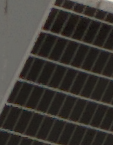}\\
\end{center}
\end{minipage}
\caption{\textbf{Real data.} \textbf{Left:} Tone mapped version of the \hdr{} image obtained by the proposed approach and its corresponding mask of unknown (black) and well-exposed (white) pixels. \textbf{Right:} Comparison of the results obtained by the proposed approach (first row) and PLEV (second row) in the extracts indicated in the top image. Please see the digital copy for better details reproduction.}
\label{fig:realExpsPalomas}
\end{figure*}

\subsubsection{Real data}
The feasibility of the SVE random pattern has been shown in~\cite{schoberl13} and that of the SVE regular pattern in~\cite{yasuma10}. Nevertheless, these acquisition systems are still not available for general usage.\footnote{
While writing the last version of this article the authors got aware of Sony's latest sensor IMX378.
This sensor has a special mode called SME-HDR, which is a variation of the SVE acquisition principle.
Whereas this sensor was adopted by the Google Pixel smartphone in 2016, the special SME-HDR mode is never activated by the Google Pixel phone, according to experts from the company DxO, and we found no way to activate it and access the raw image.
}
However, as stated in Section~\ref{sec:model}, the only variation between the classical and the SVE acquisition is the optical filter. Hence, the noise at a pixel $p$ captured using SVE with an optical gain factor $\fa_p$ and exposure time $\tau/\fa_p$ and a pixel captured with a classical camera using exposure time $\tau$ should be very close. We take advantage of this fact in order to evaluate the reconstruction performance of the proposed approach using real data. 

For this purpose, we generate an SVE image $\z_{sve}$  from four raw images $\{ \z_{raw}^i \}_{i=1,\dots,4}$ acquired with different exposure times. The four different exposure times simulate four different filters of the SVE optical mask. The value at position $(x,y)$ in $\z_{sve}$ is chosen at random among the four available values at that position $\{ \z_{raw}^i(x,y) \}_{i=1,\dots,4}$. Notice that the Bayer pattern is kept on $\z_{sve}$ by construction. The images $\{ \z_{raw}^i \}_{i=1,\dots,4}$ are acquired using a remotely controlled camera and a tripod so as to be perfectly aligned. 
This protocol does not allow us to take scenes with moving objects. Let us emphasize, however, that using a real SVE device, this, as well as the treatment of moving camera, would be a non-issue. 

Figures~\ref{fig:realExpsPalomas} and~\ref{fig:realExpsTelecom} show the results obtained from two real scenes, together with the masks of well-exposed (white) and unknown (black) pixels (the SVE raw images are included in Appendix B in the supplementary material). Recall that among the unknown pixels, some of them are saturated and some of them are under exposed. Square patches of size 6 and 8 were used for the examples in Figure~\ref{fig:realExpsTelecom} and Figure~\ref{fig:realExpsPalomas} respectively. Demosaicing~\cite{hamilton97} and tone mapping~\cite{mantiuk08} are used for displaying purposes.

We compare the results to those obtained by PLEV~\cite{aguerrebere14ICCP}. A comparison against the methods by Nayar and Mitsunaga and Sch\"oberl et al. is not presented since they do not specify  how to treat raw images with a Bayer pattern. The proposed method manages to correctly reconstruct the unknown pixels even in extreme conditions where more than $70\%$ of the pixels are missing, as for example the last extract in Figure~\ref{fig:realExpsTelecom}. 

These examples show the suitability of the proposed approach to reconstruct the irradiance information in both very dark and bright regions simultaneously. See for instance the example in Figure~\ref{fig:realExpsTelecom}, where the dark interior of the building (which can be seen through the windows) and the highly illuminated part of another building are both correctly reconstructed (see the electronic version of the article for better visualization). 
\begin{figure*}
\centering
\begin{minipage}[c]{.24\linewidth}
\begin{center}
\includegraphics[width=.99\linewidth]{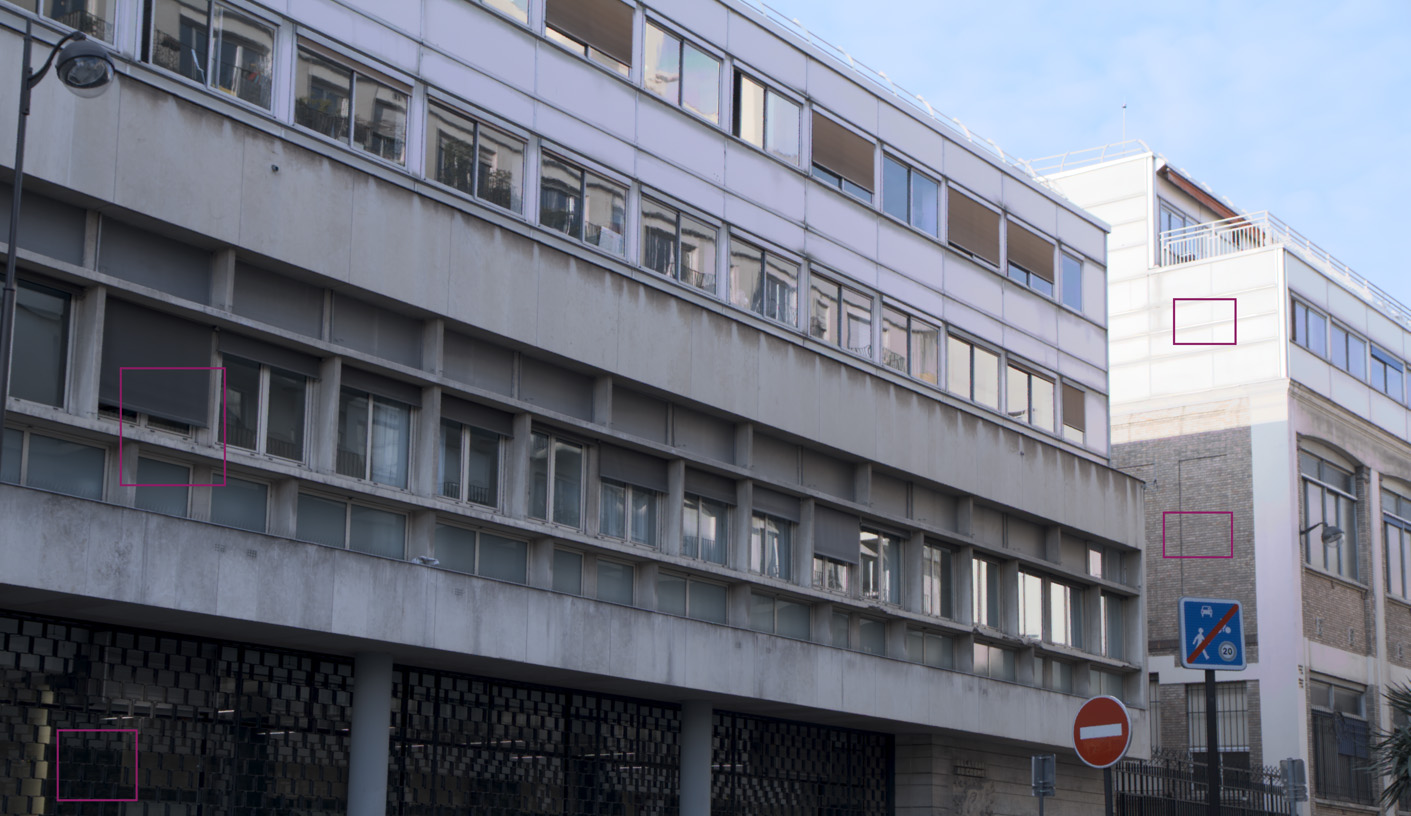}

\vspace{1pt}

\includegraphics[width=0.99\linewidth]{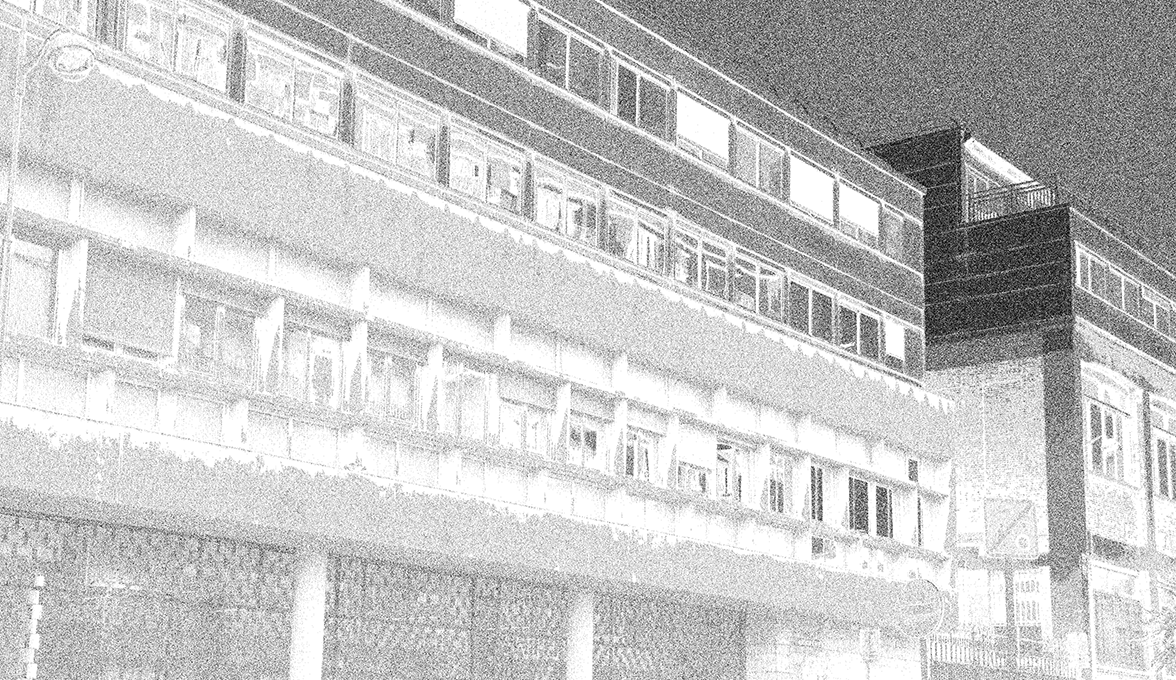}
\end{center}
\end{minipage}
\begin{minipage}[c]{.75\linewidth}
\begin{center}
\includegraphics[height=2.5cm]{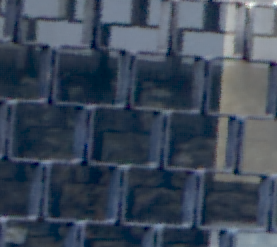} \includegraphics[height=2.5cm]{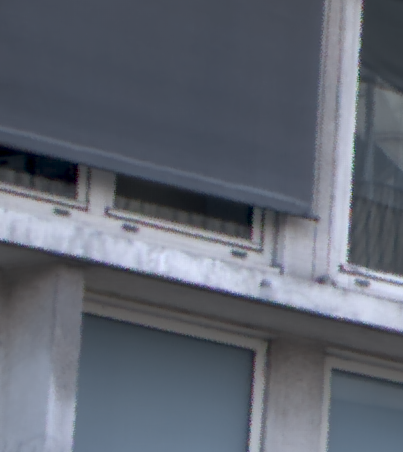} \includegraphics[height=2.5cm]{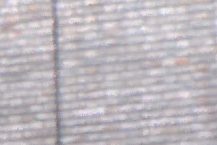} \includegraphics[height=2.5cm]{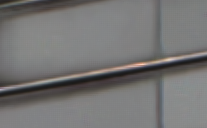}\\
\vspace{2pt}
\includegraphics[height=2.5cm]{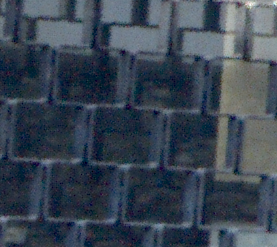} \includegraphics[height=2.5cm]{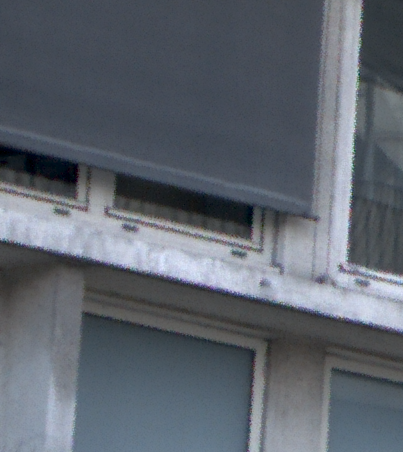} \includegraphics[height=2.5cm]{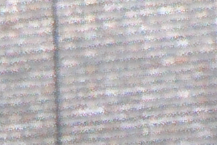} \includegraphics[height=2.5cm]{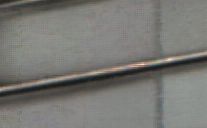}
\end{center}
\end{minipage}
\caption{\textbf{Real data.} \textbf{Left:} Tone mapped version of the \hdr{} image obtained by the proposed approach and its corresponding mask of unknown (black) and well-exposed (white) pixels. \textbf{Right:} Comparison of the results obtained by the proposed approach (first row) and PLEV (second row) in the extracts indicated in the top image. Please see the digital copy for better details reproduction.}
\label{fig:realExpsTelecom}
\end{figure*}

\section{Conclusions}
\label{sec:conclusions}
In this work we have presented a novel image restoration framework. It has the benefits of local patch characterization (that was key to the success of NLB as a state of the art denoising method), but manages to extend its use to more general restoration problems where the linear degradation operator is diagonal, by combining local estimation with Bayesian restoration based on hyperpriors. This includes problems such as zooming, inpainting and interpolation. In this way, all these restoration problems are set under the same framework. It does not include image deblurring or deconvolution, since the degradation operator is no longer diagonal. Correctly addressing deconvolution with large kernels with patch-based approaches and Gaussian prior models is a major challenge that will be the subject of future work.

We have presented a large series of experiments both on synthetic and real data that confirm the robustness of the proposed strategy based on hyperpriors. These experiments show that for a wide range of image restoration problems \hpnlb{} outperforms several state-of-the-art restoration methods.

This work opens several perspectives. The first one concerns the relevance of the Gaussian patch model and its relation to the underlying image patches manifold. If this linear approximation has proven successful for image restoration, its full relevance in other areas remains to be explored, especially in all domains requiring to compare image patches.     
Another important related question is the one of the estimation of the degradation model in images jointly degraded by noise, missing pixels, blur, etc. Restoration approaches generally rely on the precise knowledge of this model and of its parameters. In practice however, we often deal with images for which the acquisition process is unknown, and that have possibly been affected by post-treatments. In such cases, blind restoration remains an unsolved challenge.

Finally, we have presented a novel application of the proposed general framework to the generation of HDR images from a single variable exposure (SVE) snapshot. The SVE acquisition strategy allows the creation of HDR images from a single shot without the drawbacks of multi-image approaches, such as the need for global alignment and motion estimation to avoid ghosting problems. The proposed method manages to simultaneously denoise and reconstruct the missing pixels, even in the presence of (possibly complex) motions, improving the results obtained by existing methods. Examples with real data acquired in very similar conditions to those of the SVE acquisition show the capabilities of the proposed approach. 

\appendices

\section{Alternate minimization scheme convergence}
 \label{ap:mapEst}
We study in the following the convergence of the alternate minimization algorithm~\ref{algo:alternate}.   
To show the main convergence result, we need the following lemma
\begin{lemma}
The function $f$ is coercive on $\R^{n(M+1)}\times S_n^{++}(\R)$. 
\end{lemma}
\begin{proof}
We need to show that
\begin{align*}
\lim_{\|(\{\f_i\},\mean,\L )\|\rightarrow +\infty}f(\{\f_i\},\mean,\L ) &= +\infty. 
\end{align*}
Now, $\|(\{\f_i\},\mean,\L )\|\rightarrow +\infty$ if and only if $\|{\f_i}\| \rightarrow +\infty$ or $\|\mean\|\rightarrow +\infty$ or $\|\L\|\rightarrow +\infty$.

The matrix $\L$ being positive-definite, the terms  $\frac 1 2 \sum_{i=1}^M (\f_i - \mean)^T \L (\f_i - \mean)$ and $\frac {\kappa}{2}(\mean-\mean_0)^T\L(\mean-\mean_0)$ are both positive. Thus
\begin{eqnarray*}
f( \{\f_i\},\mean,\L ) 
&\geq& -{\frac{\nu -n +M}{2}}\log |\L|\\ 
&+&\frac {1}{2}\mathrm{trace}[\nu\S_0 \L ].  
\end{eqnarray*}
Now, this function of $\L$ is convex and coercive on $S_n^{++}(\R)$, which implies that $f( \{\f_i\},\mean,\L )\rightarrow +\infty$ as soon as $\|\L\|\rightarrow +\infty$. It also follows that the previous function of $\L$ has a global minimum that we denote by $m_{\L}$. 
We can now write
\begin{align*}
  f(\{\f_i\},\mean,\L ) &\geq m_{\L} +\frac 1 2 \sum_{i=1}^M (\f_i - \mean)^T \L (\f_i - \mean)  \\
&+\frac {\kappa}{2}(\mean-\mean_0)^T\L(\mean-\mean_0)
\end{align*}
and this function of $(\{\f_i\},\mean)$ clearly tends towards $+\infty$ as soon as $\|{\f_i}\| \rightarrow +\infty$ or $\|\mean\|\rightarrow +\infty$.
\end{proof}

We now show the main convergence result for our alternate minimization algorithm. The proof adapts the arguments in~\cite{Gorski2007} to our case. 
\addtocounter{proposition}{-1} 

\begin{proposition}
 The sequence $f( \{\f_i^l\},\mean^l,\L^l )$ converges monotonically when $l\rightarrow +\infty$.  The sequence $\{\{\f_i^l\},\mean_l,\L^l\}$ generated by the alternate minimization scheme has at least one accumulation point. The set of its accumulation points forms a connected and compact set of partial optima and stationary points of $f$, all having the same function value. 
\end{proposition}

\begin{proof}
  The sequence $f( \{\f_i^l\},\mean^l,\L^l )$ obviously decreases at each step by construction. The coercivity and continuity of $f$ imply that this sequence is also bounded from below, and thus converges. The convergence of $f( \{\f_i^l\},\mean^l,\L^l )$ implies that the sequence $\{\{\f_i^l\},\mean_l,\L^l\}$  is bounded. It follows that it has at least one accumulation point $(\{\f_i^{\star}\},\mean^{\star},\L^{\star})$ and that there exists a strictly increasing sequence $(l_k)_{k\in \sN}$ of integers such that $\{\{\f_i^{l_k}\},\mean^{l_k},\L^{l_k}\}_{k\in \sN}$ converges to $\{\{\f_i^{\star}\},\mean^{\star},\L^{\star}\}$. 

Now, we can show that such an accumulation point is a partial optimum of $f$, \textit{i.e.} that $f(\{\f_i^{\star}\},\mean^{\star},.)$ attains its minimum at $\L^*$ and $f(.,.,\L^*)$ attains its minimum at $(\{\f_i^{\star}\},\mean^{\star})$. By construction, 
\begin{align*}
  f(\{\f^{l_{k}}\},\mean^{l_{k}},\L^{l_{k}}) \leq  f(\{\f^{l_{k}}\},\mean^{l_{k}},\L), \;\;\;\forall \L \in S_n^{++}(\R)
\end{align*}
which implies by continuity of $f$ that
\begin{align}
\label{eq:partial1}
  f(\{\f^{*}\},\mean^{*},\L^{*}) = \argmin_{\L \in S_n^{++}(\R)}  f(\{\f^{*}\},\mean^{*},\L).
\end{align} 

 Let us denote $G(\{\f\},\mean,\L) = (\{\f'\},\mean',\L')$ with
\begin{align*}
(\{\f'\},\mean') &= \argmin_{(\{\f,\}\mean)} f(\{\f\},\mean,\L)\\
\L' &= \argmin_{\L} f(\{\f'\},\mean',\L).
\end{align*}
The alternate minimization scheme consists in updating $G$ at each iteration.  
From Equations~\eqref{eq:muhat},~\eqref{eq:Ci1}, and~\eqref{eq:Lhat}, we see that $G$ is explicit and continuous. Since $\{l_k\}_{k\in \sN}$ is strictly increasing, for each $k\in{\sN}^*$, $l_k \geq l_{k-1}+1$. The sequence $\{f( \{\f_i^l\},\mean^l,\L^l )\}_{l\in\sN}$ decreases, so  
\begin{align*}
f(G(\{\f^{l_{k-1}}\},\mean^{l_{k-1}},\L^{l_{k-1}})) &=
f(\{\f^{l_{k-1}+1}\},\mean^{l_{k-1}+1},\L^{l_{k-1}+1}) \\
&\geq  f(\{\f^{l_{k}}\},\mean^{l_{k}},\L^{l_{k}})\\
&\geq f(G(\{\f^{l_{k}}\},\mean^{l_{k}},\L^{l_{k}})).
\end{align*}
Therefore, as $k\rightarrow +\infty$, since $G$ is continuous, it follows  that 
\begin{align*}
  f(G(\{\f^*\},\mean^*,\L^*)) &= f(\{\f^*\},\mean^*,\L^*).
\end{align*}
Now, writing $(\{\f^{**}\},\mean^{**},\L^{**}) = G(\{\f^*\},\mean^*,\L^*)$, we get
\begin{align*}
  f(\{\f^*\},\mean^*,\L^*) &\geq \argmin_{(\f,\mean)} f(\{\f\},\mean,\L^*)=f(\{\f^{**}\},\mean^{**},\L^{*})\\
&\geq \argmin_{\L} f(\{\f^{**}\},\mean^{**},\L)=  f(\{\f^{**}\},\mean^{**},\L^{**}).
\end{align*}
We can conclude that all these terms are equal and in particular 
\begin{align}
\label{eq:partial2}
  f(\{\f^*\},\mean^*,\L^*) &= f(\{\f^{**}\},\mean^{**},\L^{*}) = \argmin_{(\f,\mean)} f(\{\f\},\mean,\L^{*}).
\end{align}

From~\eqref{eq:partial1} and~\eqref{eq:partial2} we deduce that the accumulation point $(\{\f^{\star}\},\mean^{\star},\L^{\star})$ is a partial optimum of $f$ and since $f$ is differentiable, it is also a stationary point of $f$. 
Moreover, since $f(.,.,\L^*)$ is strictly convex and has a unique minimum, it follows from~\eqref{eq:partial2} that  $(\{\f^{**}\},\mean^{**}) = (\{\f^*\},\mean^*)$. As a consequence,
$\L^{**} = \argmin_{\L} f(\{\f^{**}\},\mean^{**},\L) = \argmin_{\L} f(\{\f^{*}\},\mean^{*},\L) = \L^{*}$. Therefore, the accumulation point $(\{\f^{\star}\},\mean^{\star},\L^{\star})$ (and actually any accumulation point of the sequence) is also a fixed point of function $G$. 
 
We have shown that accumulation points of the sequence $\{\{\f^l\},\mean^l,\L^l \}$ are partial optima of $f$ and fixed points of the function $G$. The set of accumulation points is obviously compact.  
Let us show that it is also a connected set. First, observe that whatever the norm $\|\|$, the sequence  $\|(\{\f^{l+1}\},\mean^{l+1},\L^{l+1} ) - (\{\f^l\},\mean^l,\L^l)\|$ converges to $0$ when $l\rightarrow \infty$. If it was not the case, it would be possible to extract a subsequence $(\{\f^{l_k}\},\mean^{l_k},\L^{l_k} )$ converging to an accumulation point $(\{\f^{*}\},\mean^{*},\L^{*} )$ while $(\{\f^{l_k+1}\},\mean^{l_k+1},\L^{l_k+1} )$ converges to a different accumulation point $(\{\f^{'}\},\mean^{'},\L^{'} )$, but we know that it is impossible since  $(\{\f^{l_k+1}\},\mean^{l_k+1},\L^{l_k+1} ) = G(\{\f^{l_k}\},\mean^{l_k},\L^{l_k} )$ would also tend toward $(\{\f^{*}\},\mean^{*},\L^{*} )$. The sequence $(\{\f^{l}\},\mean^{l},\L^{l} )$ being bounded and such that $\|(\{\f^{l+1}\},\mean^{l+1},\L^{l+1} ) - (\{\f^l\},\mean^l,\L^l)\|$ converges to $0$, the set of its accumulation points is connected (see~\cite{ostrowski1960solution}).
The fact that all accumulation points have the same function value is obvious since the sequence $\{f( \{\f^l\},\mean^l,\L^l )\}_{l\in\sN}$ decreases.
\end{proof}

\section*{Acknowledgement}
We would like to thank the reviewers for their thorough and fruitful contributions, as well as Mila Nikolova and Alasdair Newson for their insightful comments and the authors of~\cite{lebrun13IPOL,wang13b,zoran11_web} for kindly providing their code. 
This work has been partially funded by the French Research Agency (ANR) under grant nro ANR-14-CE27-001 (MIRIAM), and by the ``Investissement d'avenir'' project, reference ANR-11-LABX-0056-LMH.

\ifCLASSOPTIONcaptionsoff
  \newpage
\fi

\bibliographystyle{IEEEtran}
\bibliography{references}

\begin{IEEEbiography}[{\includegraphics[width=1in,height=1.25in,clip,keepaspectratio]{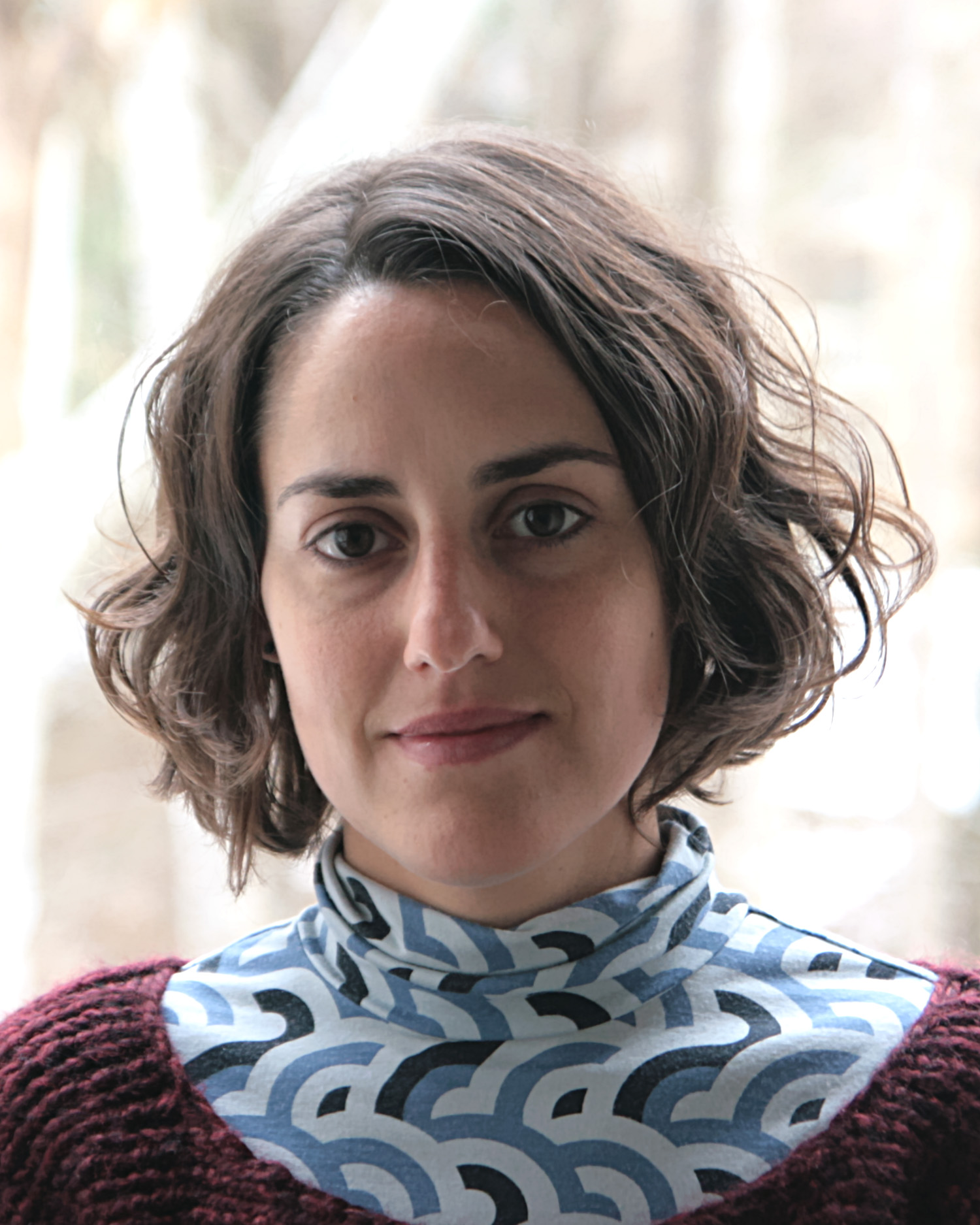}}]{Cecilia Aguerrebere}
received the B.Sc., M.Sc. and Ph.D. degrees in electrical engineering from Universidad de la Rep\'ublica, Uruguay, in 2006, 2011 and 2014 respectively, the M.Sc. degree in applied mathematics from ENS Cachan, France, in 2011, and the Ph.D. degree in Signal and Image Processing from T\'el\'ecom ParisTech, France, in 2014 (joint Ph.D program with Universidad de la Rep\'ublica, Uruguay). From August 2015 she is with the Electrical and Computer Engineering Department, Duke University, where she holds a Postdoctoral Research Associate position.
\end{IEEEbiography}
\begin{IEEEbiography}[{\includegraphics[width=1in,height=1.25in,clip,keepaspectratio]{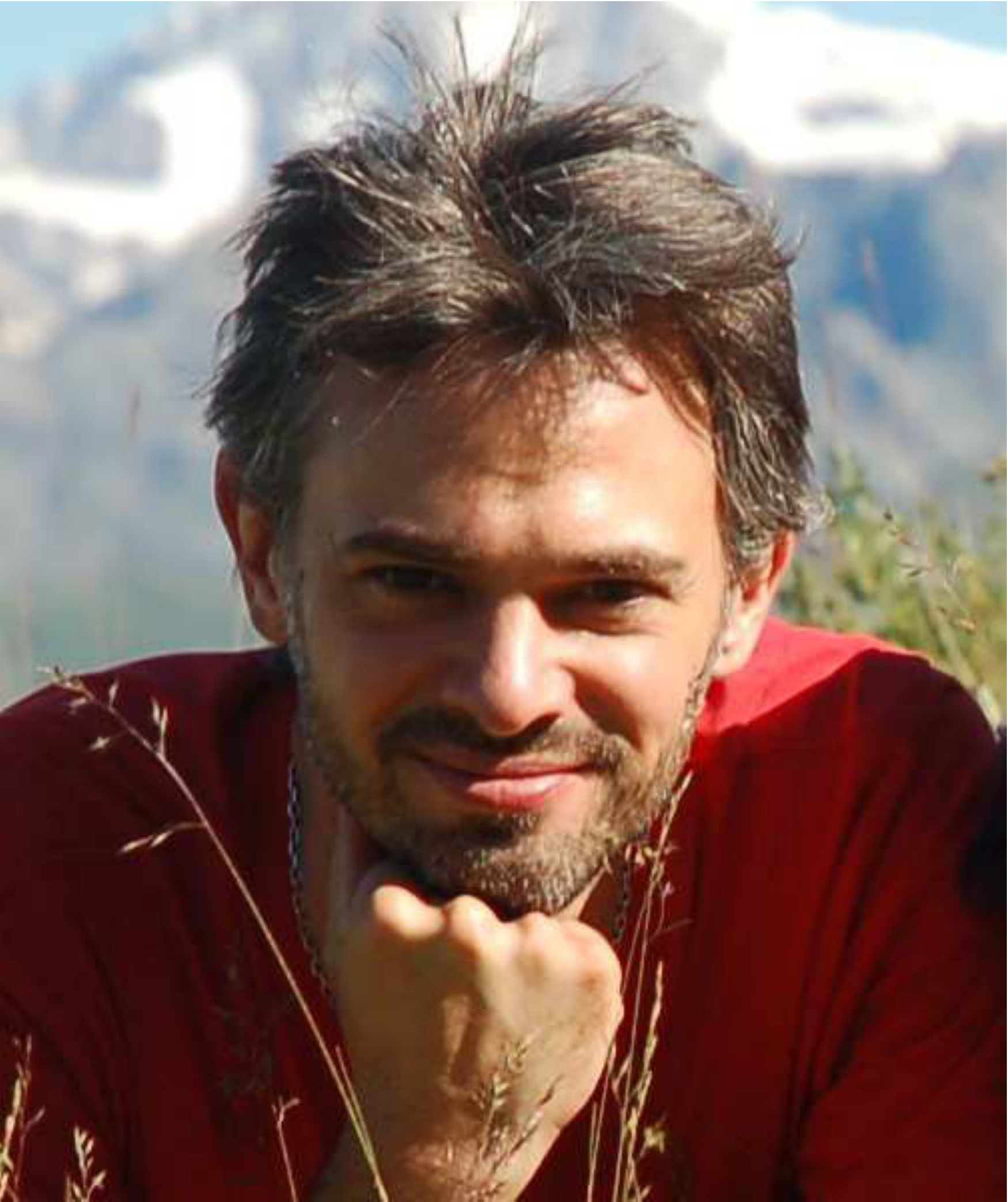}}]{Andr\'es Almansa} received his HDR, Ph.D. and M.Sc./Engineering degrees in Applied Mathematics and Computer Science from Universit\'e Paris-Descartes, ENS Cachan (France) and Universidad de la Republica (Uruguay), respectively, where he is an Associate Professor since 2004. His current interests as a CNRS Research Scientist at Telecom ParisTech include image restoration and analysis, subpixel stereovision and applications to earth observation, high quality digital photography and film restoration.
\end{IEEEbiography}
\begin{IEEEbiography}[{\includegraphics[width=1in,height=1.25in,clip,keepaspectratio]{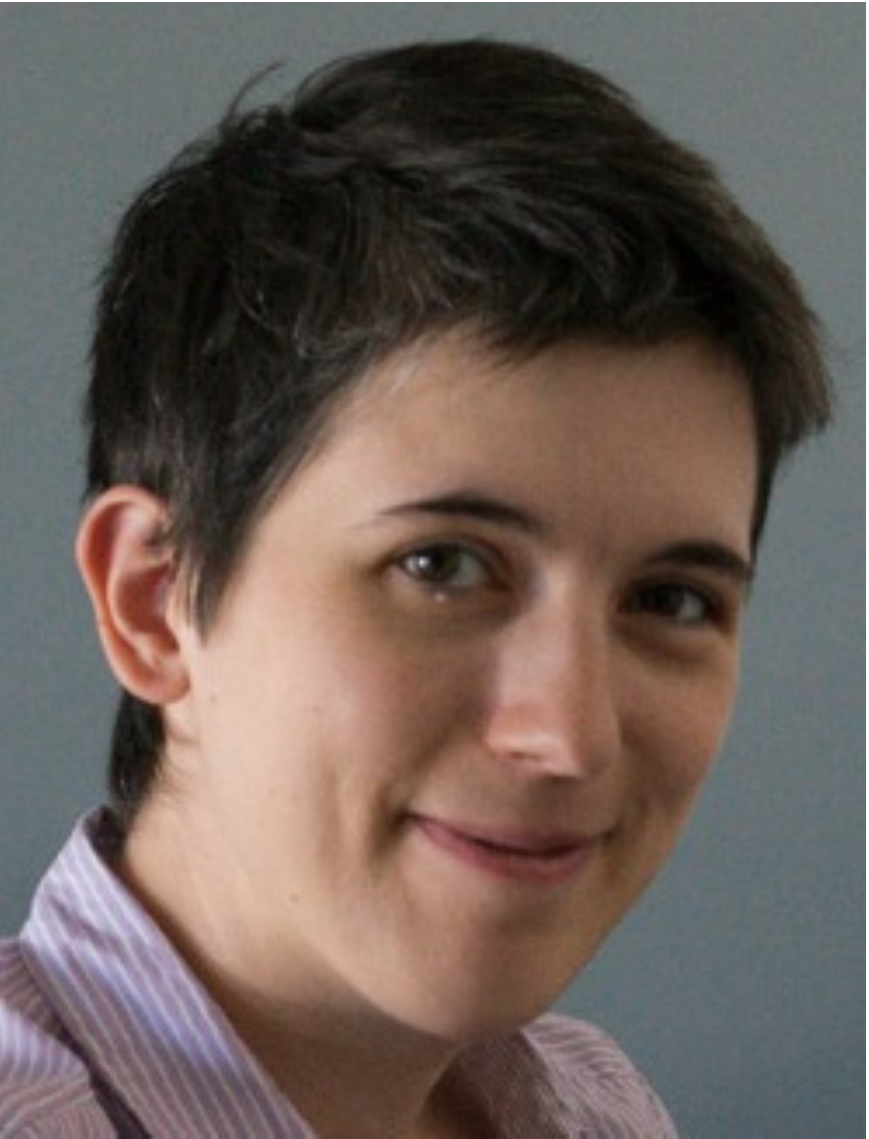}}]{Julie Delon} studied mathematics from the \'Ecole Normale Sup\'erieure de Cachan and received the Ph. D. degree from ENS Cachan, France, in 2004 and the HDR degree from ENS Cachan in 2011. Between 2005 and 2013, she was a CNRS Researcher with T\'el\'ecom ParisTech, Paris, France. Since 2013, she is a Professor of applied mathematics with Paris-Descartes University. Her current research interests include mono and multi-image restoration, optimal transport, and stochastic approaches in computer vision. She was coordinator of the young researcher ANR project FREEDOM between 2007 and 2011.
\end{IEEEbiography}
\begin{IEEEbiography}[{\includegraphics[width=1in,height=1.25in,clip,keepaspectratio]{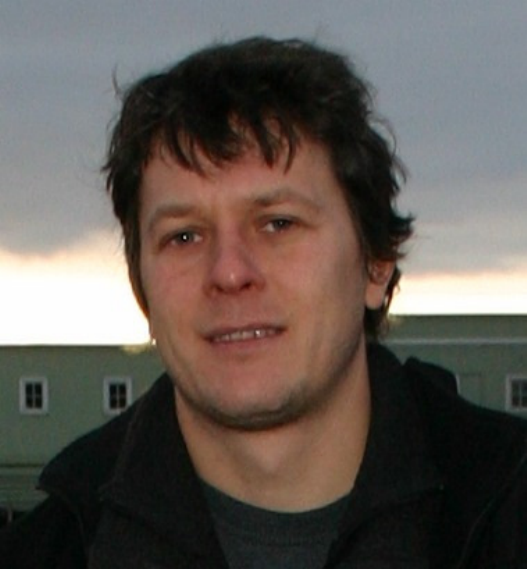}}]{Yann Gousseau} received the engineering degree from the \'Ecole Centrale de Paris, France, in 1995, and the Ph.D. degree in applied mathematics from the University of Paris-Dauphine in 2000. He is currently a professor at T\'el\'ecom ParisTech. His research interests include the mathematical modeling
of natural images and textures, stochastic geometry, image analysis, computer vision and image processing.
\end{IEEEbiography}
\begin{IEEEbiography}[{\includegraphics[width=1in,height=1.25in,clip,keepaspectratio]{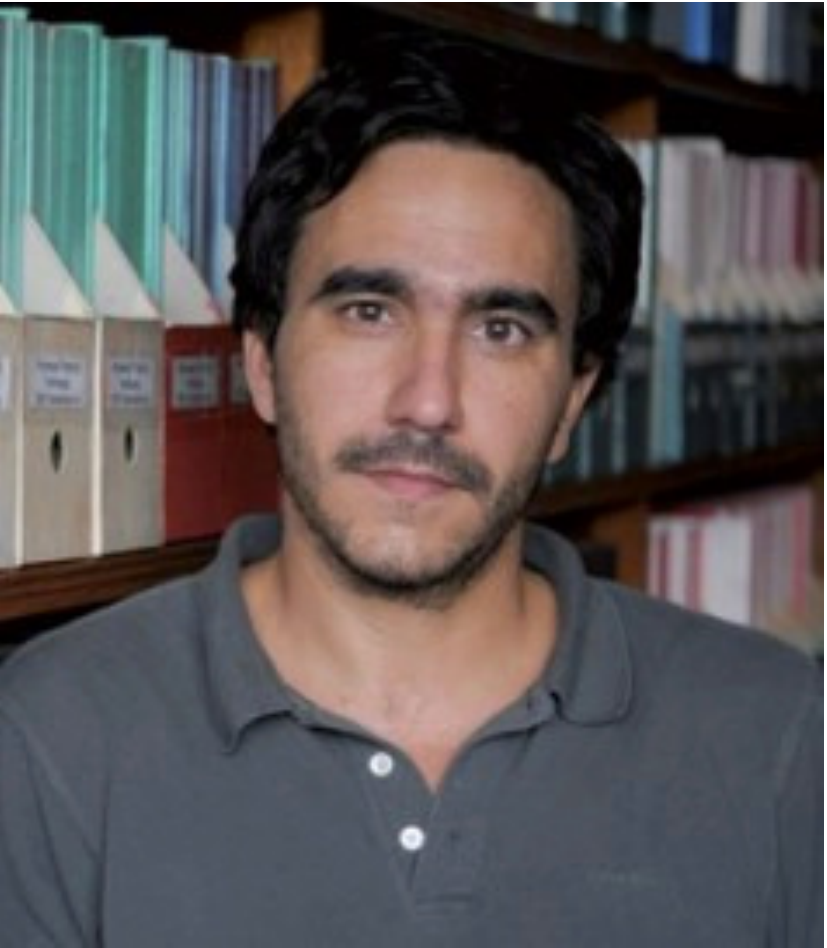}}]{Pablo Mus\'e}
received his Electrical Engineering degree from Universidad de la Rep\'ublica, Uruguay, in 1999, his M.Sc. degree in Mathematics, Vision and Learning and his Ph.D. in Applied Mathematics from ENS de Cachan, France, in 2001 and 2004 respectively. From 2005 to 2006 he was with Cognitech, Inc., Pasadena, CA, USA, where he worked on computer vision and image processing applications. In 2006 and 2007, he was a Postdoctoral Scholar with the Seismological Laboratory, California Institute of Technology, Pasadena, working on remote sensing using optical imaging, radar and GPS networks. Since 2008, he has been with the Division of Electrical Engineering, School of Engineering, Universidad de la Rep\'ublica, where he is currently a Full Professor of signal processing.
\end{IEEEbiography}

\end{document}